\documentclass[acmsmall]{acmart}

\AtBeginDocument{%
  }

\setcopyright{acmcopyright}
\copyrightyear{2023}
\acmYear{2023}
\acmDOI{XXXXXXX.XXXXXXX}

\acmJournal{JACM}
\acmVolume{0}
\acmNumber{0}
\acmArticle{0}
\acmMonth{0}

\usepackage[T1]{fontenc}
\usepackage{microtype}
\usepackage{amsfonts, amsmath, amsthm}
\usepackage{bm, bbm}
\usepackage{nicefrac}
\usepackage{booktabs}
\usepackage{graphicx}
\usepackage[font=small, belowskip=0ex]{caption}
\usepackage[aboveskip=-1ex]{subcaption}
\usepackage{csvsimple}

\newtheorem{lemma}{Lemma}											
\newtheorem{theorem}{Theorem}										
	
\theoremstyle{definition}		
\numberwithin{equation}{section}									

\DeclareMathOperator{\clip}{clip}
\DeclareMathOperator{\softmax}{softmax}
\newcommand{\ub}[2]{\underbrace{#1}_{\scriptscriptstyle{#2}}}

\begin{document}

\title{Examining Policy Entropy of Reinforcement Learning Agents for Personalization Tasks}

\author{Anton~Dereventsov}
\email{adereventsov@lirio.com}
\author{Andrew~Starnes}
\email{astarnes@lirio.com}
\affiliation{%
  \institution{Lirio AI Research, Lirio LLC}
  \city{Knoxville}
  \state{Tennessee}
  \country{USA}
  \postcode{37923}
}

\author{Clayton~Webster}
\email{cwebster@lirio.com}
\affiliation{%
  \institution{Behavioral Reinforcement Learning Lab, Lirio LLC}
  \city{Knoxville}
  \state{Tennessee}
  \country{USA}
  \postcode{37923}
}


\begin{abstract}
This effort is focused on examining the behavior of reinforcement learning systems in personalization environments and detailing the differences in policy entropy associated with the type of learning algorithm utilized.
We demonstrate that Policy Optimization agents often possess low-entropy policies during training, which in practice results in agents prioritizing certain actions and avoiding others.
Conversely, we also show that Q-Learning agents are far less susceptible to such behavior and generally maintain high-entropy policies throughout training, which is often preferable in real-world applications.
We provide a wide range of numerical experiments as well as theoretical justification to show that these differences in entropy are due to the type of learning being employed.
\end{abstract}




\maketitle

\section{Introduction}\label{sec:intro}
Recommendation and personalization are often mistakenly considered interchangeable concepts in the context of providing content suggestions tailored to the interests of a particular individual.
However, though both are complex methods of reaching and retaining customers and patients, there are subtle but important differences between the two.
While executing a recommendation generally involves filtering a collection of data based on the historical preferences, interests, or behaviors of a user, personalization always focuses on one thing: each user's own personal attributes.
In other words, recommendation requires a user's previous interactions with content in order to provide bespoke recommendations, whereas personalization instead uses a user's provided characteristics.

In industries such as retail, e-commerce, media apps, or even healthcare, recommendation system models (see, e.g.,~\cite{10.1007/978-1-4899-7637-6_1,10.1609/aimag.v32i3.2361}) are critical to customer retention.
Corporations like Netflix, Spotify, Amazon, etc., use sophisticated collaborative filtering and content-based recommendation systems for video, song, and/or product recommendations~\cite{10.1145/2843948,10.1007/978-1-4899-7637-6_11,10.1145/2959100.2959120,10.1109/MIC.2017.72,10.1145/2623372}.
For a recent overview of recommender systems in the healthcare domain see, e.g.,~\cite{10.1007/s10844-020-00633-6} and the references therein.

Conventional personalization focuses on personal, transactional, demographic, and possibly health-related information, such as an individual's age, residential location, employment, purchases, medical history, etc.
The most basic example of this task involves the inclusion of the customer's name in the subject line or content of an email.
This technique relies on generalization and profiling to make specific assumptions about the individual based on their characteristics.
Additional applications of personalization include: web content personalization and layout customization~\cite{10.1016/j.jss.2016.02.008, Ricci2011IntroductionTR}; customer-centric interaction with healthcare providers~\cite{lasalvia2020personalization, 10.1145/3318236.3318249, lei2017actor, zhu2018robust, hassouni2018personalization, tan2020adaptive}; personalized medical treatments~\cite{2007_aspinall, gs2001personalized, harrison2023zero}.

One of the major challenges associated with personalization techniques is the time required to adapt and update such approaches to changes in individual behaviors, reactions, and choices.
Recently, reinforcement learning (RL) has been increasingly exploited in personalized recommendation systems that continually interact with users (see, e.g., \cite{10.3233/DS-200028} and the references therein).
As opposed to traditional recommendation techniques, RL is a more complex and transformative approach that considers behavioral and  real-time data produced as the result of user action.
Examples of this technique include online browsing behavior, communication history, in-app choices, and other engagement data.
This allows for more individualized experiences like adding individualized engaging sections to the body of an email or sending push notifications at a time when the customer is typically active, which results in more customized communication and thus, ultimately, greater conversion.

For the purpose of this work we consider the two most commonly employed-in-practice types of RL techniques: namely, Policy Optimization (PO) and Q-Learning (QL).
While both approaches strive to learn the optimal policy, the trajectories they follow are different, which results in different behaviors during the training process.
We numerically and theoretically explore this phenomenon throughout this effort and explain that its source stems from the learning objectives that each approach utilizes.
Our main contributions are:
\begin{itemize}
	\item formalization of the phenomenon of differences in behavior of the PO and QL agents;
	\item empirical demonstration of the phenomenon on a wide variety of personalization tasks;
	\item theoretical explanation of the driving factors behind this phenomenon on contextual bandit environments.
\end{itemize}

\subsection{Related Work}
The motivation of our research is to establish knowledge on the fundamental behavior of reinforcement learning agents, which alignes with the Alberta Plan for AI Research~\cite{sutton2022alberta}.
Specifically, the goal of this paper is to understand the effect that the type of learning has on an agent's policy entropy, with a particular focus on the distinction between PO and QL algorithm families.
In~\cite{nachum2017bridging} the authors explore similar differences and propose an approach that unifies both families, though we are interested in understanding the differences.

The undesirable behavior of the agent's policy is conventionally addressed with entropy regularization~\cite{ahmed2019understanding, yang2019regularized}.
However, we are not considering regularization techniques in the current work as our goal is to understand the underlying behavior of the algorithm.

In~\cite{garg2021alternate} the authors address issues with the action distribution placing too much weight on sub-optimal actions by proposing a new gradient estimator.
In~\cite{mei2020escaping} the authors show that using a softmax normalization for a policy gradient agent causes sensitivity to parameter initialization and propose an alternative policy normalization to remedy these issues.
Unlike the stated works, we are less interested in fixing the issue and more focused on understanding what causes it in the first place.

In our numerical experiments we rewrite image classification tasks as contextual bandit environments.
Such an approach was proposed in~\cite{dudik2011doubly} and was utilized in~\cite{swaminathan2015counterfactual, chen2019surrogate, garg2021alternate}.
In our theoretical results we investigate the behavior of the policy dynamics under PO and QL update rules.
This question has been explored in various settings, see e.g.~\cite{agarwal2021theory, wang2019neural}.

As publicly available personalization data is rather scarce, we additionally employ simulated environments that are designed to replicate real-world personalization applications.
Examples of such simulators are presented in~\cite{rohde2018recogym, ie2019recsim, dereventsov2021unreasonable}, where the authors consider a similar problem setting and address related issues.

\section{Preliminaries}\label{sec:preliminaries}

In this paper we consider a contextual bandit setting, defined in~\cite{langford2007epoch}, with discrete action space, which is the standard setting for recommendation and personalization tasks, see e.g.\cite{li2010contextual, tang2015personalized}.
Throughout this paper we use the notational standard MDPNv1~\cite{thomas2015notation}.
Namely, $\mathcal{S} \subset \mathcal{R}^d$ denotes the state (context) space, $\mathcal{A} = \{1, 2, \ldots, K\}$ denotes the action space consisting of $K$ available actions, and $r : \mathcal{S} \times \mathcal{A} \to \mathcal{R}$ denotes the reward function.

The main metric of performance of a reinforcement learning agent is the value $V$ of its policy $\pi$, which is defined as the expected return under the policy, i.e.
\[
	V(\pi) = \mathbb{E} \Big[ r(s,a) \ \big|\ s \sim \mathcal{S}, a \sim \pi(s) \Big].
\]

In addition to the agent's performance, the distribution of the agent's policy $\pi$ is often critical in practical applications as it directly translates to the actions the agent is taking throughout the training process.
A conventional way to quantify the policy distribution is by computing its entropy $\mathcal{H}(\pi)$ given by
\[
	\mathcal{H}(\pi) = \mathbb{E} \Big[ -\sum_{a \in \mathcal{A}} \pi(a|s) \log\pi(a|s) \ \big|\ s \sim \mathcal{S} \Big].
\]
Entropy indicates how distributed the policy is, with more localized policies having lower entropy values, which is known to lead to undesirable results, discussed in e.g.~\cite{dou2008evaluating}.

Note that often the optimal policies are deterministic and therefore have a low entropy, thus $\mathcal{H}(\pi)$ might not actually represent the true distribution of the policy $\pi$, especially in the later stages of the training process.
To address this issue, we introduce the \textit{batch-entropy} $\mathcal{H}_b(\pi;\mathcal{S}_0)$, which is computed on the evaluation set $\mathcal{S}_0 \subset \mathcal{S}$ as
\[
	\mathcal{H}_b(\pi;\mathcal{S}_0)
	= -\sum_{a \in \mathcal{A}} \mathbb{E}\big[ \pi(a|s) \,\big|\, s \sim \mathcal{S}_0 \big]
		\log \mathbb{E}\big[ \pi(a|s) \,\big|\, s \sim \mathcal{S}_0 \big].
\]
Unlike the regular entropy $\mathcal{H}(\pi)$, the batch-entropy $\mathcal{H}_b(\pi;\mathcal{S}_0)$ is computed on the average distribution over the evaluation set $\mathcal{S}_0$ and thus serves as a more reliable representation of the distribution of the policy $\pi$.
In particular, low values of $\mathcal{H}_b(\pi;\mathcal{S}_0)$ indicate that the distribution of $\pi$ is localized on the whole evaluation set $\mathcal{S}_0$ rather than its individual elements.

We also note that in practice it is typically unfeasible to calculate the expectation over the state space, see e.g.~\cite{dereventsov2021offline}, and instead it is conventional to employ some form of Monte Carlo estimate over the available data.
In our numerical examples, the value and the entropy are computed on the fixed evaluation sets to guarantee that the policy's value and entropy are estimated consistently throughout the training process.

In this work we focus on the distinction between Policy Optimization (PO) and Q-Learning (QL) approaches.
PO methods, see, e.g.,~\cite{grondman2012survey}, attempt to directly maximize the expected return of the policy, QL methods, see, e.g.~\cite{jang2019q}, attempt to learn the reward signal, which implicitly leads to higher returns.

Despite being commonly used in practical RL environments, see, e.g., \cite{srivihok2005commerce, wang2017interactive, pan2019policy, li2020sample, jose2021supervised}, PO and QL methods behave very differently when learning an environment involving personalized communication.
While the goal of both approaches is to learn the optimal policy, they typically employ fundamentally different solution strategies which result in drastically different behaviors in terms of the policy entropy.
In this work we provide a detailed explanation of this phenomenon, showing that the primary cause comes from the learning objective that each approach utilizes.
In the next section we showcase this point on a simple example.

\section{Sample Experiment}\label{sec:sample}
We demonstrate a  difference in PO and QL learning styles in the following toy problem on MNIST dataset.
Consider the contextual bandit $(\mathcal{S,A},r)$, where $\mathcal{S} \subset \mathbb{R}^{784}$ is the set of flattened input images, $\mathcal{A} = \{0, 1, \ldots, 9\}$ is the set of available labels, and $r : \mathcal{S \times A} \to \mathbb{R}$ is the binary reward function indicating the correctness of prediction, i.e.
\[
	r(s,a) = \left\{\begin{array}{cl}
		1 & \text{if image } s \text{ has label } a,
		\\
		0 & \text{otherwise}.
	\end{array}\right.
\]

We deploy two basic agents on this environment~--- policy gradient agent (PG) and q-learning agent (QL).
Both agents are parameterized by linear network $\mathcal{Z} : \mathcal{S} \to \mathbb{R}^{10}$, which is represented by the matrix of trainable weights $W \in \mathbb{R}^{10 \times 784}$, i.e. for any state $s \in \mathcal{S}$ we have
\[
	\mathcal{Z}(s) = \ub{W}{10 \times 784} \times \ub{s}{784} \in \mathbb{R}^{10}.
\]
Then the policies $\pi_{pg}$ and $\pi_{ql}$ are derived as
\begin{align*}
	\pi_{pg}(s) &= \softmax(\mathcal{Z}^{pg}(s)),
	\\
	\pi_{ql}(s) &= [0, \ldots, 1, \ldots, 0],
\end{align*}
where the value $1$ is at the position of the maximal coordinate of the output vector $\mathcal{Z}^{ql}(s) \in \mathbb{R}^{10}$.

Both agents are being trained via gradient descent optimization with a batch size of $1$.
The learning curves of the agents are presented in Figure~\ref{fig:sample_reward}.
We observe that both agents solve the environment within a comparable time.

\begin{figure}[t]
	\centering
	\begin{subfigure}{.49\linewidth}
		\includegraphics[width=\linewidth]{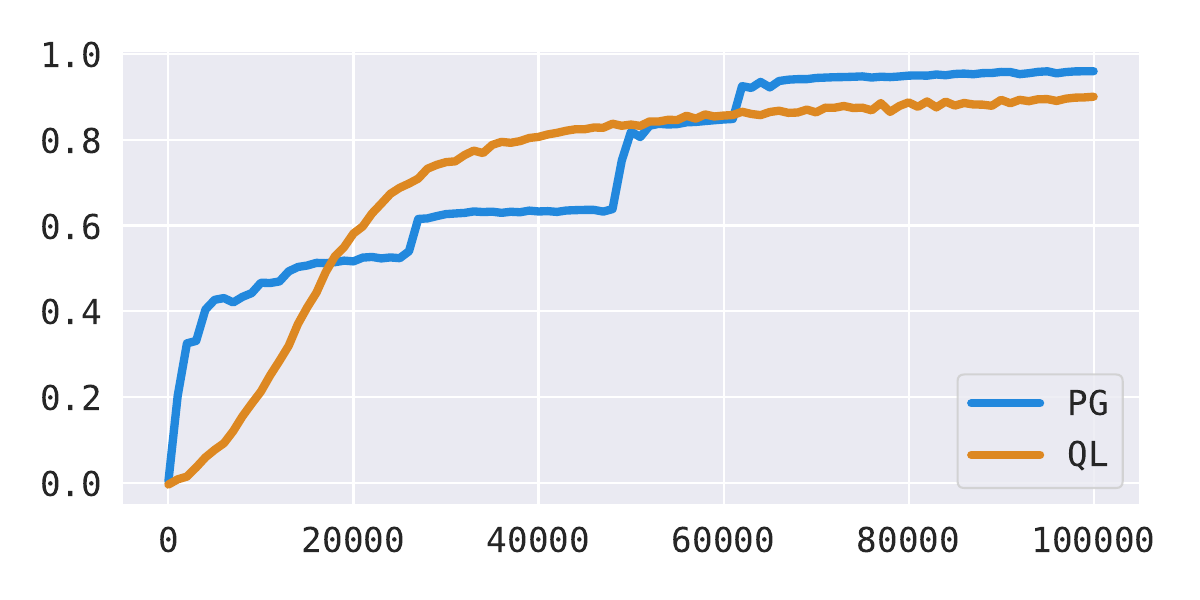}
		\caption{Policy value}
		\label{fig:sample_reward}
	\end{subfigure}
	\begin{subfigure}{.49\linewidth}
		\includegraphics[width=\linewidth]{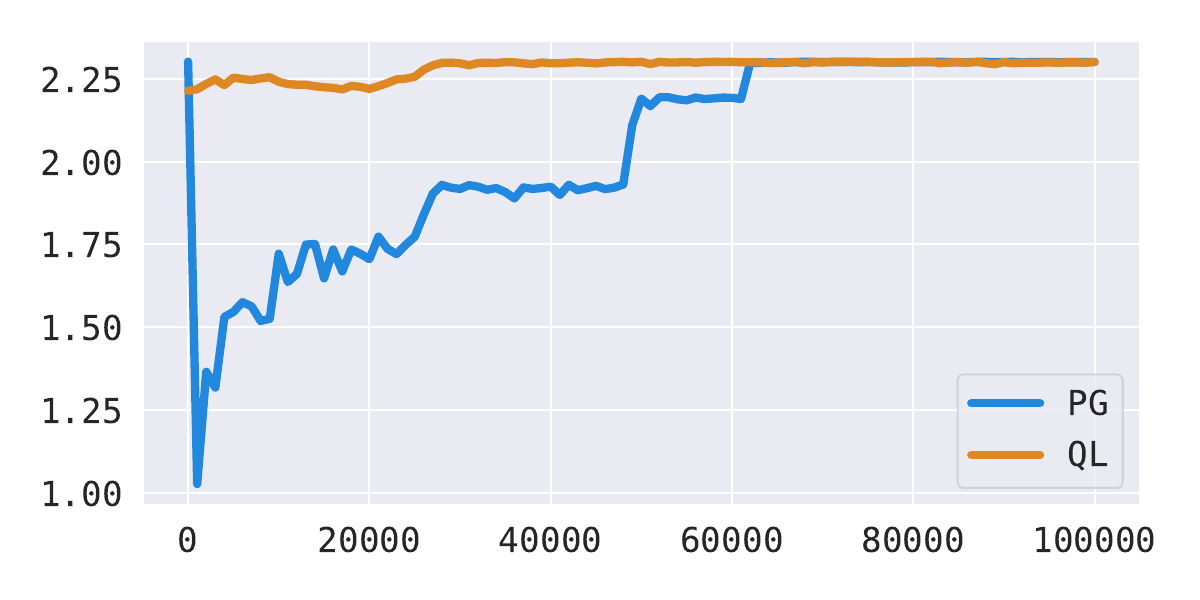}
		\caption{Policy batch-entropy}
		\label{fig:sample_entropy}
	\end{subfigure}
	\caption{Results of the Sample Experiment.}
	\label{fig:sample_reward_entropy}
\end{figure}

However, the learning curves alone do not show the whole process of policy learning and thus are not enough to highlight the difference between the agents.
To illustrate this point, we measure and plot the batch-entropy of the agents' policies throughout the training, presented in Figure~\ref{fig:sample_entropy}.
We observe that, unlike the reward values, the entropies of the policies $\pi_{pg}$ and $\pi_{ql}$ behave very differently.
Namely, the entropy of PG-agent drops early in the training process and only recovers closer to the end of the training, which is an undesirable behavior in practice.

To explore the issue further, we measure and plot action selection histograms, presented in Figure~\ref{fig:sample_dist}, that are obtained as follows: at any evaluation step we compute the label predictions on the test set for each agent.

\begin{figure}[h]
	\centering
	\includegraphics[width=.75\linewidth]{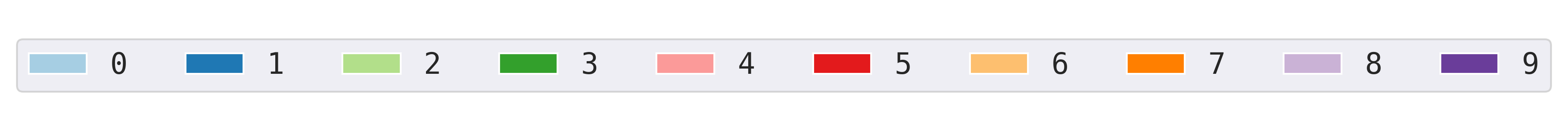}
	\\[-2ex]
	\begin{subfigure}{.49\linewidth}
		\includegraphics[width=\linewidth]{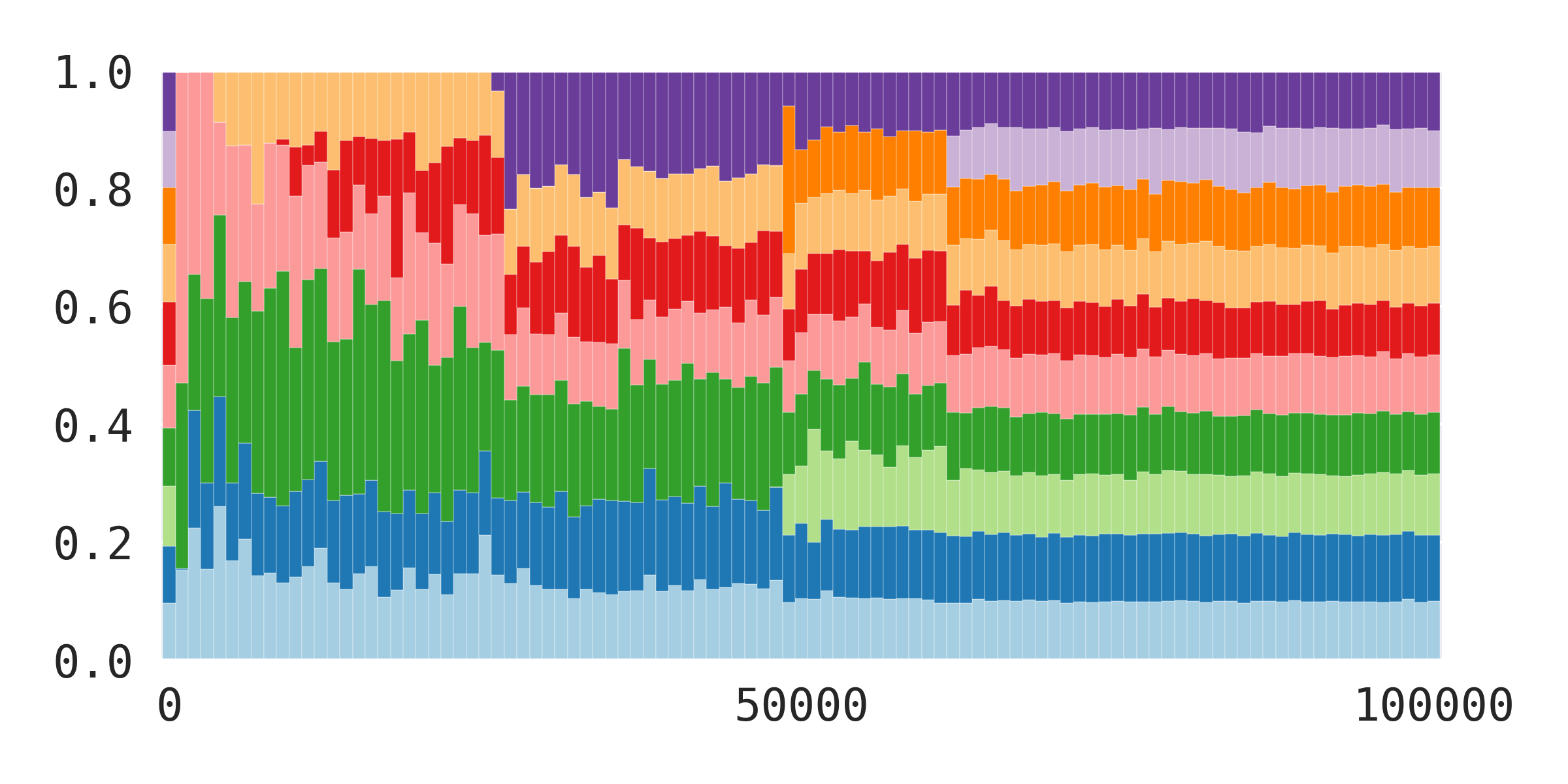}
		\caption{PG-agent}
		\label{fig:sample_dist_pg}
	\end{subfigure}
	\begin{subfigure}{.49\linewidth}
		\includegraphics[width=\linewidth]{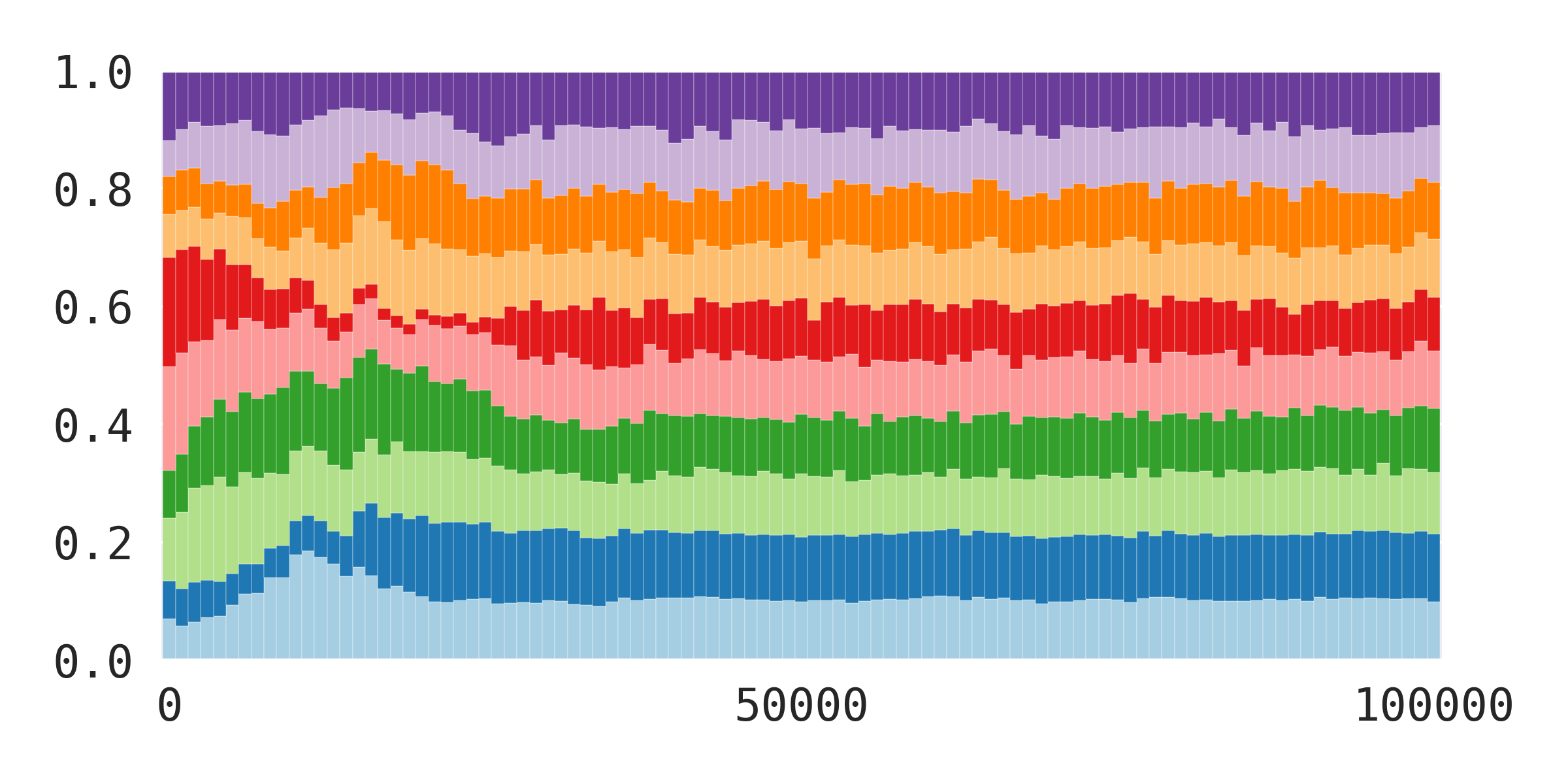}
		\caption{QL-agent}
		\label{fig:sample_dist_ql}
	\end{subfigure}
	\caption{Action selection histograms in the Sample Experiment.}
	\label{fig:sample_dist}
\end{figure}

The histograms indicate that PG-agent does not predict labels $2,7,8$ throughout most of the learning process.
While such an issue is forgivable in simulated environments or offline-trained systems, it is unacceptable due to ethical considerations in online personalization applications where the agent interacts with users in real time, see e.g.~\cite{abel2016reinforcement, whittlestone2021societal}.

Even though this issue can potentially be addressed with a practical solution such as entropy regularization, the goal of this paper is to demonstrate this phenomenon and understand why it happens in the first place.
To that end we state the update rules for the network outputs $\mathcal{Z}$, given in a general form in Theorem~\ref{thm:outputs_linear} in Section~\ref{sec:theory}.

The change $W \leftarrow W^\prime$ of the network weights $W \in \mathbb{R}^{10 \times 784}$ on the interaction $(s,a,r)$ caused by the step of gradient descent with the learning rate $\lambda > 0$ is given by the following update rule:
\[
	W^\prime = W - \lambda \nabla\mathcal{L}(s,a,r),
\]
where $\nabla\mathcal{L}(s,a,r) \in \mathbb{R}^{10 \times 784}$ is the value of the loss gradient on the interaction $(s,a,r)$.
The gradient loss for each agent is given as (see Lemma~\ref{thm:grad_pi} in Appendix~\ref{sec:theory_appendix})
\begin{align*}
	\nabla\mathcal{L}_{pg}(s,a,r) &= -r \nabla\log\pi_{pg}(a|s)
	= -r \big[\mathbbm{1}(a=k) - \pi(k|s)\big]_{k=1}^K \times \nabla\mathcal{Z}^{pg}(s),
	\\
	\nabla\mathcal{L}_{ql}(s,a,r) &= -2 (r - \mathcal{Z}^{ql}_a(s)) \nabla\mathcal{Z}^{ql}_a(s),
\end{align*}
where $\mathcal{Z}_a(s)$ is the network output coordinate corresponding to the action $a \in \mathcal{A}$.

In this setting the outputs $z_k : \mathcal{S} \to \mathbb{R}$ for each label $k \in \mathcal{A} = \{0, 1, \ldots, 9\}$ change as follows:
\begin{align*}
	z^{{pg}^\prime}_k(x) &= z^{pg}_k(x) + \lambda r \big( \mathbbm{1}(a=k) - \pi_{pg}(k|s) \big) \langle s, x \rangle,
	\\
	z^{{ql}^\prime}_k(x) &= z^{ql}_k(x) + 2 \lambda \mathbbm{1}(a=k) \big( r - \langle W_a, s \rangle \big) \langle s, x \rangle,
\end{align*}
where $x \in \mathcal{S}$ denotes the network input.

Taking into account that the reward values are binary, i.e. $r(s,a) \in \{0,1\}$, we arrive at
\begin{equation}\label{eq:pg_z_update}
	z^{{pg}^\prime}_k(x) = \left\{\begin{array}{ll}
		z_k^{pg}(x) + \lambda (1 - \pi_{pg}(k|s)) \langle s, x \rangle
		&\text{if label } a \text{ is correct and } a = k,
		\\
		z_k^{pg}(x) - \lambda \pi_{pg}(k|s) \langle s, x \rangle
		&\text{if label } a \text{ is correct and } a \ne k,
		\\
		z_k^{pg}(x)
		&\text{if label } a \text{ is incorrect}.
		\end{array}\right.
\end{equation}
Taking into account that pixels from the input images take values from $[0,1]$, we deduce that $\langle s, x \rangle \ge 0$ for all $s,x \in \mathcal{S}$.
Therefore the update~\eqref{eq:pg_z_update} implies that when the prediction $a \in \mathcal{A}$ is correct, the selected output $z_a^{pg}$ is reinforced for all the inputs, while all other outputs $z_k^{pg}$ with $k \ne a$ are penalized.
If the prediction $a \in \mathcal{A}$ is incorrect, no changes are made to the network.
It is evident how such an update scheme leads to a situation where the agent has a set of ``good'' actions that the agent will select consistently and ``bad'' actions that the agent will avoid if possible~--- the behavior seen in Figure~\ref{fig:sample_dist_pg}.

Conversely, we do not observe a similar entropy drop for QL-agent since that corresponding update rule for the outputs $\mathcal{Z}^{ql}$ is given as
\begin{equation}\label{eq:ql_z_update}
	z^{{ql}^\prime}_k(x) = \left\{\begin{array}{ll}
		z_k^{ql}(x) + 2\lambda (1 - z_k^{ql}(s)) \langle s, x \rangle
		&\text{if } a = k \text{ and label } a \text{ is correct},
		\\
		z_k^{ql}(x) - 2\lambda z_k^{ql}(s) \langle s, x \rangle
		&\text{if } a = k \text{ and label } a \text{ is incorrect},
		\\
		z_k^{ql}(x)
		&\text{if } a \ne k.
		\end{array}\right.
\end{equation}
Note that, unlike the PG-agent updates~\eqref{eq:pg_z_update}, after each interaction $(s,a,r)$ the QL-agent changes only the output $z_a^{ql}$ that was selected.
Moreover, the direction of the output change (i.e. reinforcing/penalizing) is determined not only by the correctness of the prediction $a$, but by the accuracy of the output value $z_a^{ql}(s)$ in relation to the reward value $r = r(s,a)$.

In this toy example we demonstrate that the agents trained with policy optimization and q-learning methods behave fundamentally differently, despite having the same policy values.
The underlying reason for this phenomenon is encompassed in the network output update rules~\eqref{eq:pg_z_update} and~\eqref{eq:ql_z_update}, that are stated for a more general setting in Theorem~\ref{thm:outputs_linear}.
We note that even though the toy example considered in this section is not practical, it helps build an intuition for the source of the distinct behavior of the different learning styles.
Moreover, despite the fact that in practice the environments are more intricate and the agents are more complicated and use additional regularization, the underlying issue still persists and is caused by the same factors, as we explore in the next sections.

\section{Numerical Experiments}\label{sec:numerics}
In this section we deploy reinforcement learning algorithms on a variety of personalization tasks.
The presented experiments are performed in Python3 with the use of publicly available libraries.
The datasets are either publicly available or included in the repository so that the results can be recreated.
Our experiments do not require GPU usage and can be executed on personal laptops.
Both the source code and the necessary datasets are available at~\url{https://github.com/sukiboo/policy_entropy}.

The reinforcement learning agents utilized in our examples are implemented via Stable-Baselines3\footnote{\url{https://github.com/DLR-RM/stable-baselines3}} library~\cite{raffin2021stable}.
Specifically, we deploy the following algorithms:
\begin{itemize}\setlength{\itemsep}{0ex}
	\item Advantage Actor Critic~\cite{mnih2016asynchronous}
	\item Deep Q Network~\cite{mnih2013playing}
	\item Proximal Policy Optimization~\cite{schulman2017proximal}
\end{itemize}
These algorithms were selected because they are widely used in practice and support discrete action spaces and both continuous and discrete state spaces, as such environments best model our setting.

The performance of the agents is measured in terms of policy value and policy batch-entropy, which are computed over the evaluation set that is fixed throughout the training.
Additionally, for each agent we provide the stochastic action selection histogram that represents the agent's policy distribution over the evaluation set.
For the simplicity of presentation the histograms are sorted, though the unmodified histograms for each experiment are provided in Appendix~\ref{sec:numerics_appendix}.
In all the plots presented in Figures~\ref{fig:mnist_exp}--\ref{fig:personalization_exp} the x-axis shows the number of unique agent-environment interactions.

The main point of our numerical experiments is to substantiate the claim that, with all else equal, policy optimization agents tend to have policies of low entropy, while q-learning agents maintain a higher entropy on a wide range of tasks.

\subsection{Image Classification Experiment}
In this experiment we train reinforcement learning agents to predict labels of the given images from the supervised learning datasets.
Specifically, we rewrite the image classification task as a contextual problem, the same way it is proposed in~\cite{dudik2011doubly, swaminathan2015counterfactual, chen2019surrogate}.
For the purpose of this experiment we use MNIST\footnote{\url{http://yann.lecun.com/exdb/mnist/}} and CIFAR10\footnote{\url{https://www.cs.toronto.edu/~kriz/cifar.html}} datasets.

The state space $\mathcal{S}$ consists of the train/test images of corresponding sizes: $28 \times 28$ for MNIST and $32 \times 32 \times 3$ for CIFAR10.
The action space $\mathcal{A} = \{0, 1, \ldots, 9\}$ consists of the available labels: $\{0, 1, 2, 3, 4, 5, 6, 7, 8, 9\}$ for MNIST and $\{$airplane, automobile, bird, cat, deer, dog, frog, horse, ship, truck$\}$ for CIFAR10.
On each interaction an agent observes a state (a training image) $s \in \mathcal{S}$ and takes an action (makes a prediction) $a \in \mathcal{A}$, and receives a reward $r(s,a)$.
The reward function $r : \mathcal{S \times A} \to \mathbb{R}$ is given as
\[
	r(s,a) = \left\{\begin{array}{cl}
				1 &\text{ if } a \text{ is the correct label for } s,
				\\
				\nicefrac{-1}{9} &\text{ if } a \text{ is an incorrect label for } s.
			\end{array}\right.
\]
In this setting a random policy has the value of $0$ and the optimal policy has the value of $1$.

The policy values, policy batch-entropy, and the action selection histograms for MNIST and CIFAR10 datasets are presented in Figures~\ref{fig:mnist_exp} and~\ref{fig:cifar10_exp} respectively.
The hyperparameter choices for each agent are listed in Appendix~\ref{sec:numerics_appendix} in Tables~\ref{tab:mnist_params} and~\ref{tab:cifar10_params} for MNIST and CIFAR10 datasets respectively, and any unspecified parameters are kept at their default values.
The unsorted action selection histograms for MNIST and CIFAR10 datasets are presented in Figures~\ref{fig:mnist_dist_raw} and~\ref{fig:cifar10_dist_raw} respectively in Appendix~\ref{sec:numerics_appendix}.

\begin{figure}[h]
	\centering
	\begin{subfigure}{.49\linewidth}
		\includegraphics[width=\linewidth]{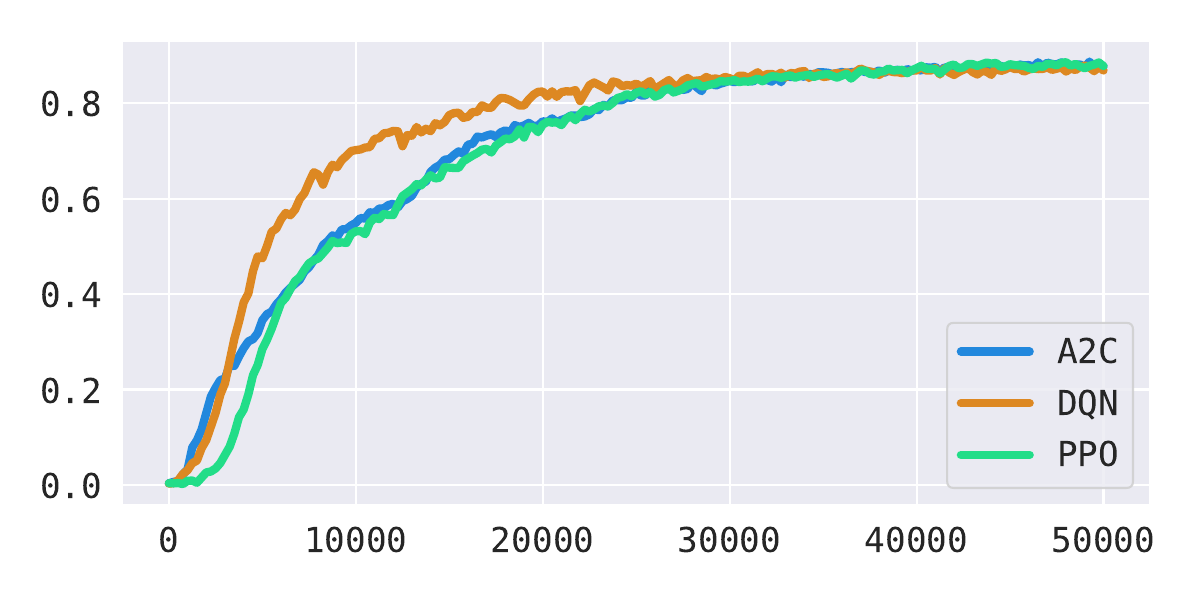}
		\caption{Policy value}
	\end{subfigure}
	\begin{subfigure}{.49\linewidth}
		\includegraphics[width=\linewidth]{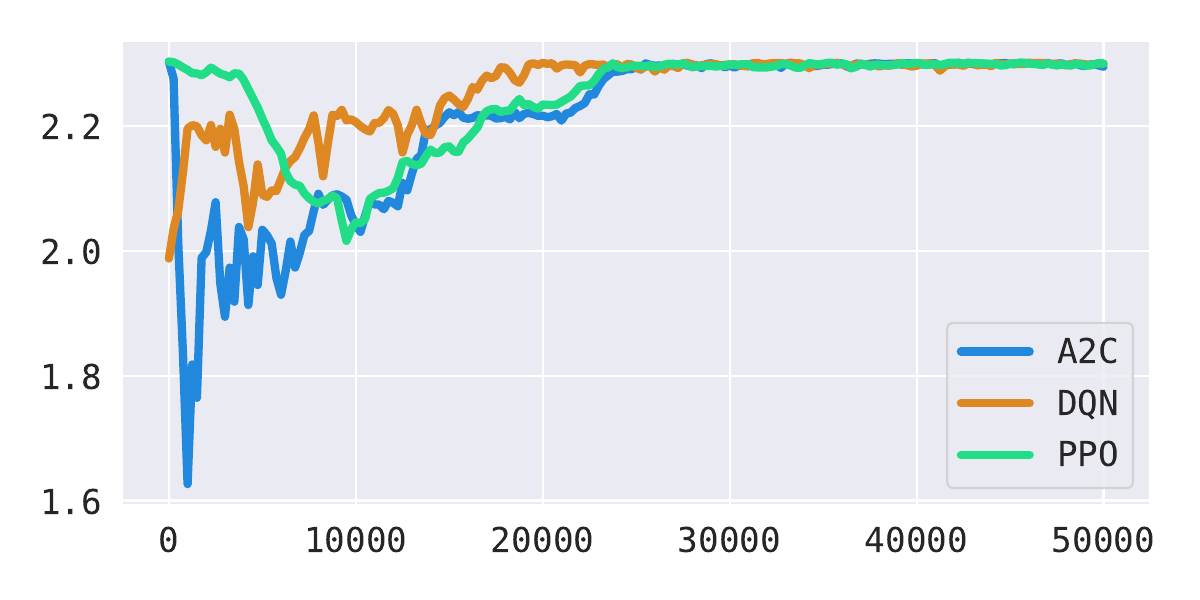}
		\caption{Policy batch-entropy}
	\end{subfigure}
	\\
	\begin{subfigure}{.32\linewidth}
		\includegraphics[width=\linewidth]{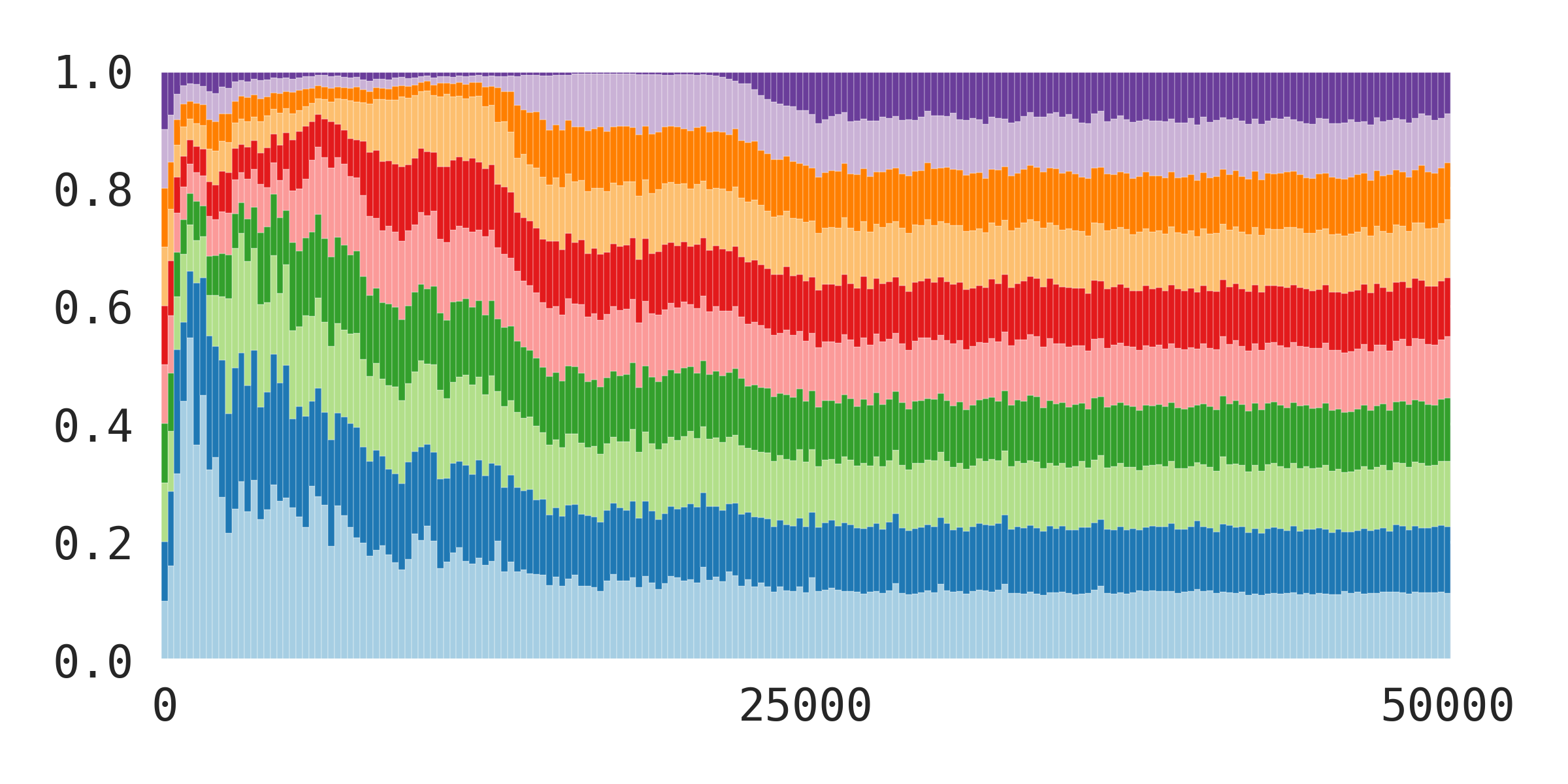}
		\caption{A2C histogram}
	\end{subfigure}
	\begin{subfigure}{.32\linewidth}
		\includegraphics[width=\linewidth]{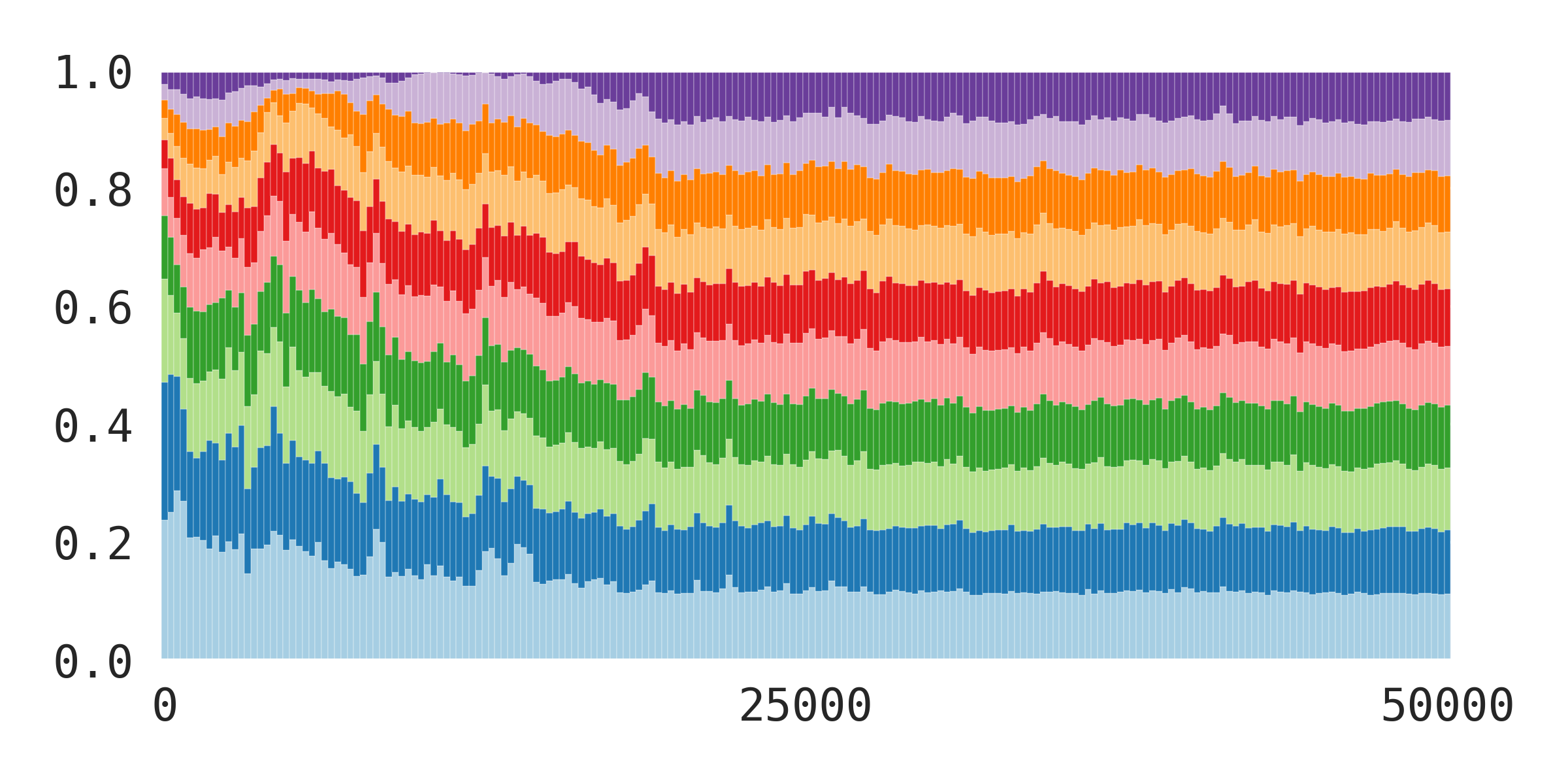}
		\caption{DQN histogram}
	\end{subfigure}
	\begin{subfigure}{.32\linewidth}
		\includegraphics[width=\linewidth]{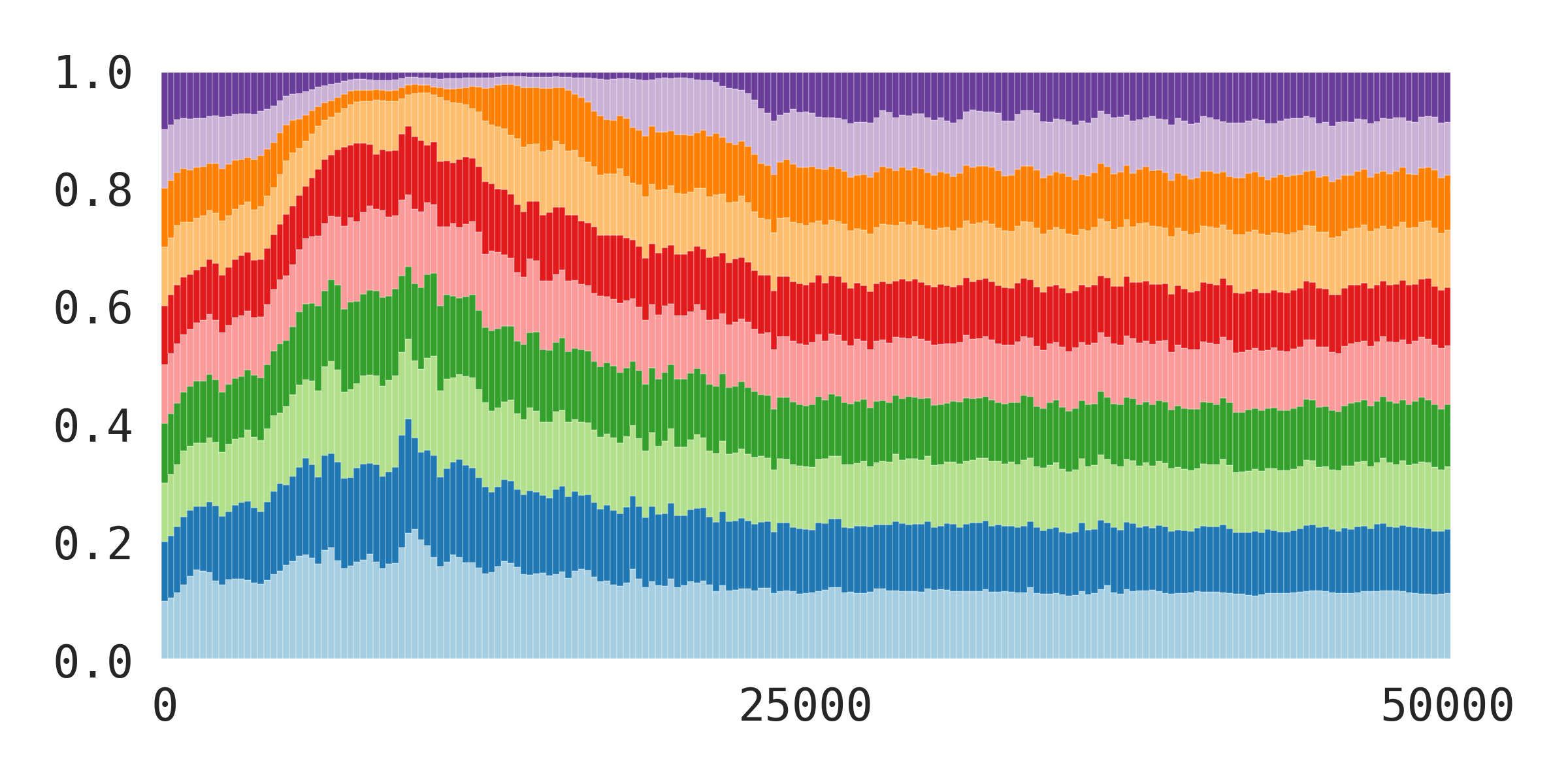}
		\caption{PPO histogram}
	\end{subfigure}
	\caption{Results of the Image Classification Experiment on MNIST dataset.}
	\label{fig:mnist_exp}
\end{figure}

\begin{figure}[h]
	\centering
	\begin{subfigure}{.49\linewidth}
		\includegraphics[width=\linewidth]{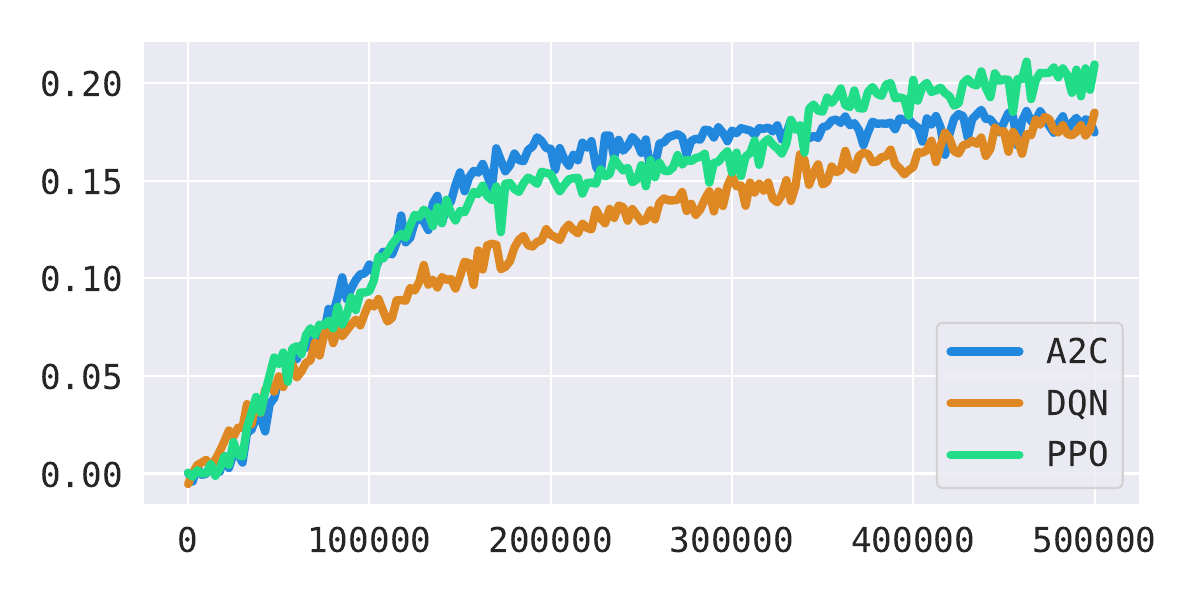}
		\caption{Policy value}
	\end{subfigure}
	\begin{subfigure}{.49\linewidth}
		\includegraphics[width=\linewidth]{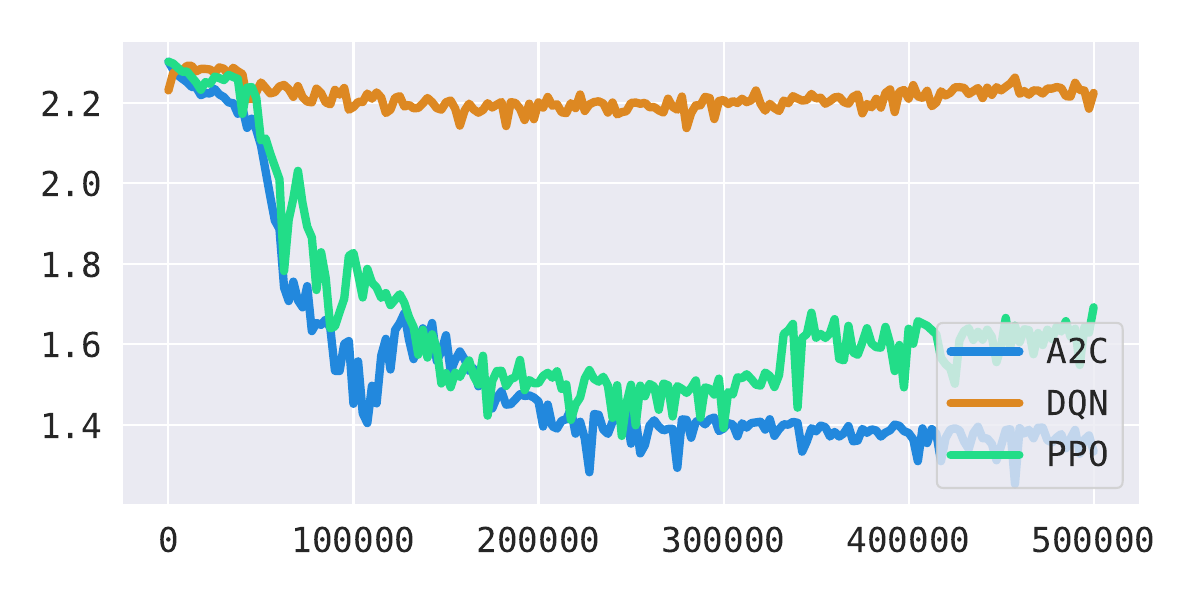}
		\caption{Policy batch-entropy}
	\end{subfigure}
	\\
	\begin{subfigure}{.32\linewidth}
		\includegraphics[width=\linewidth]{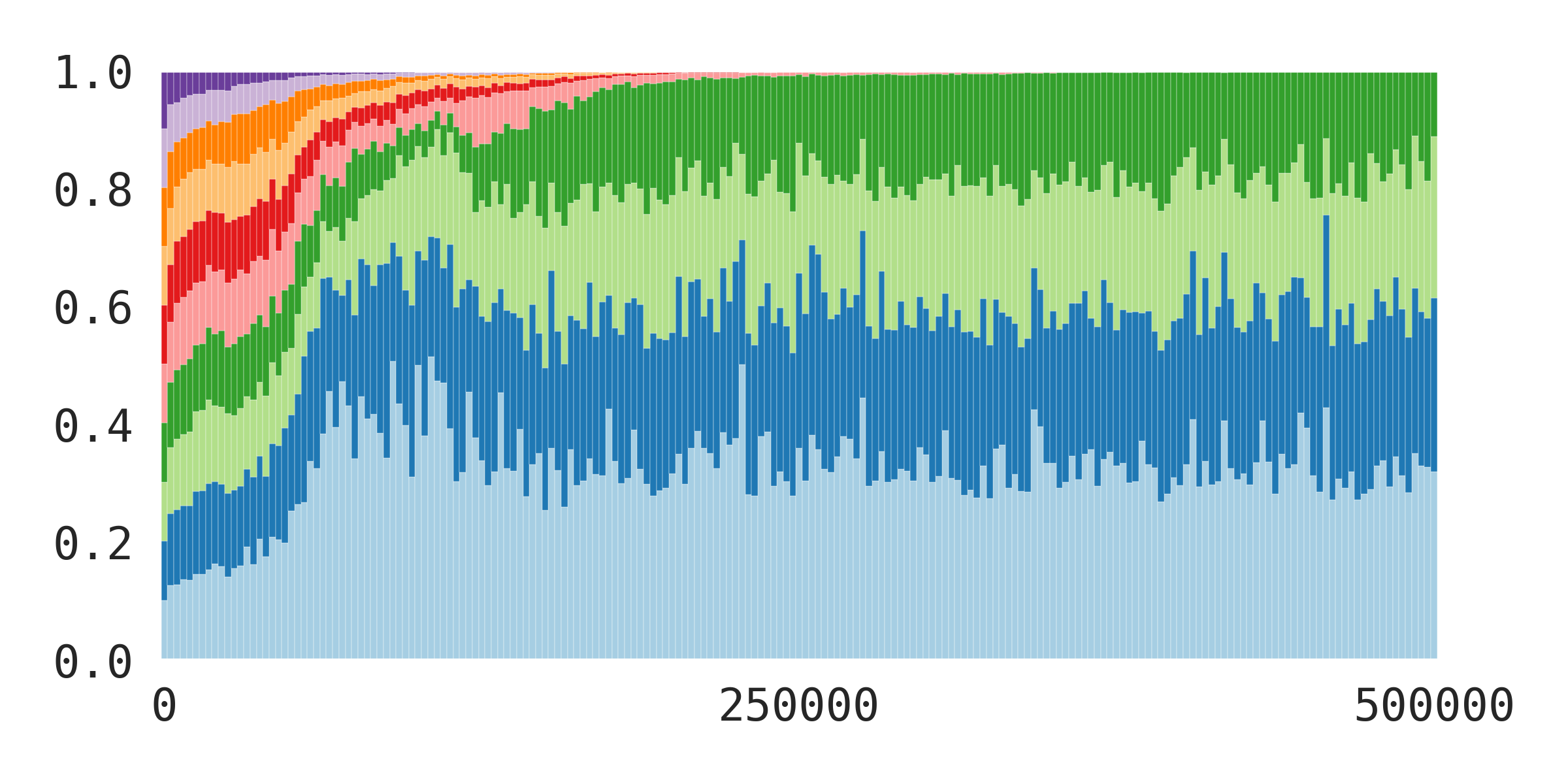}
		\caption{A2C histogram}
	\end{subfigure}
	\begin{subfigure}{.32\linewidth}
		\includegraphics[width=\linewidth]{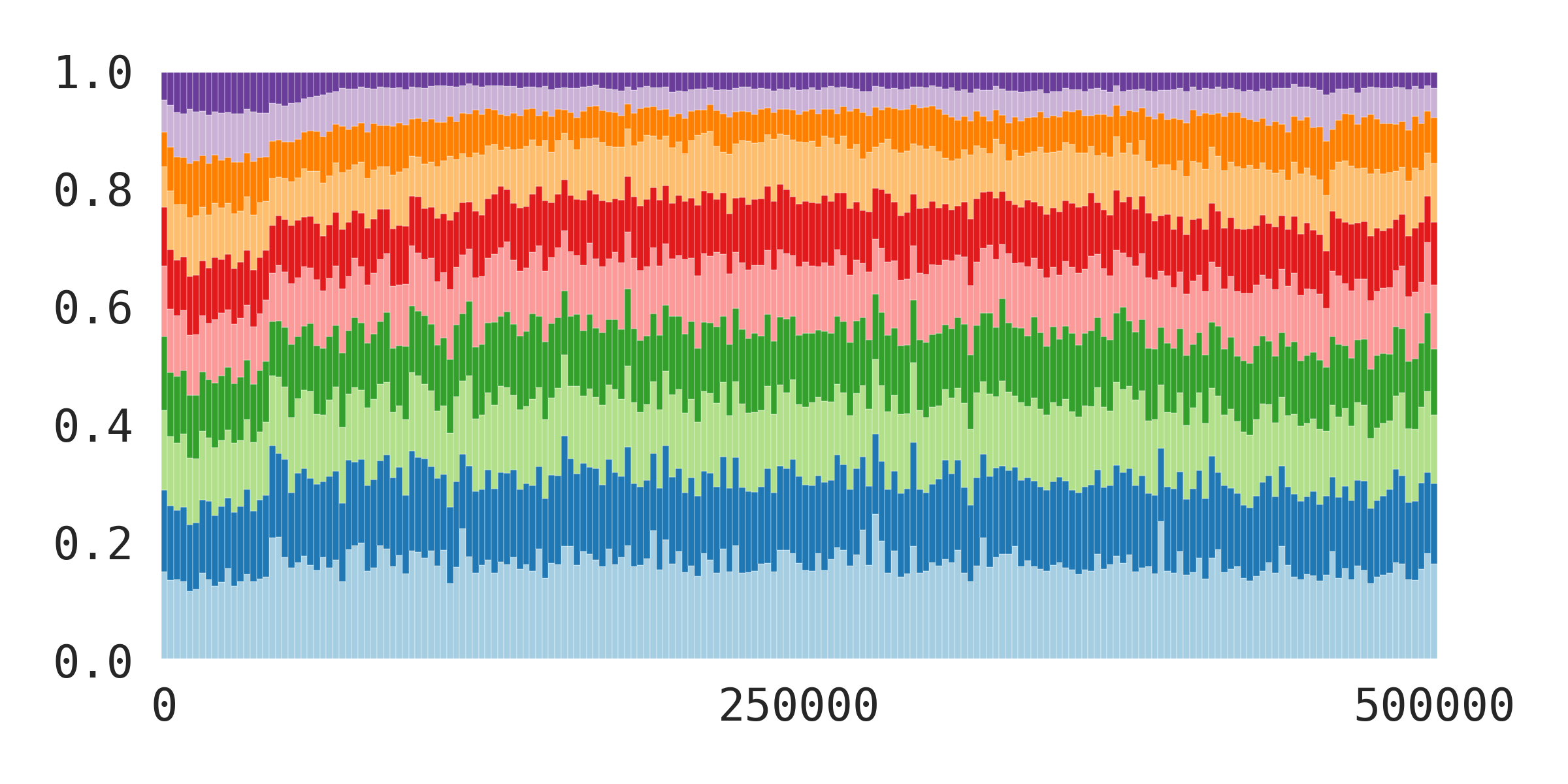}
		\caption{DQN histogram}
	\end{subfigure}
	\begin{subfigure}{.32\linewidth}
		\includegraphics[width=\linewidth]{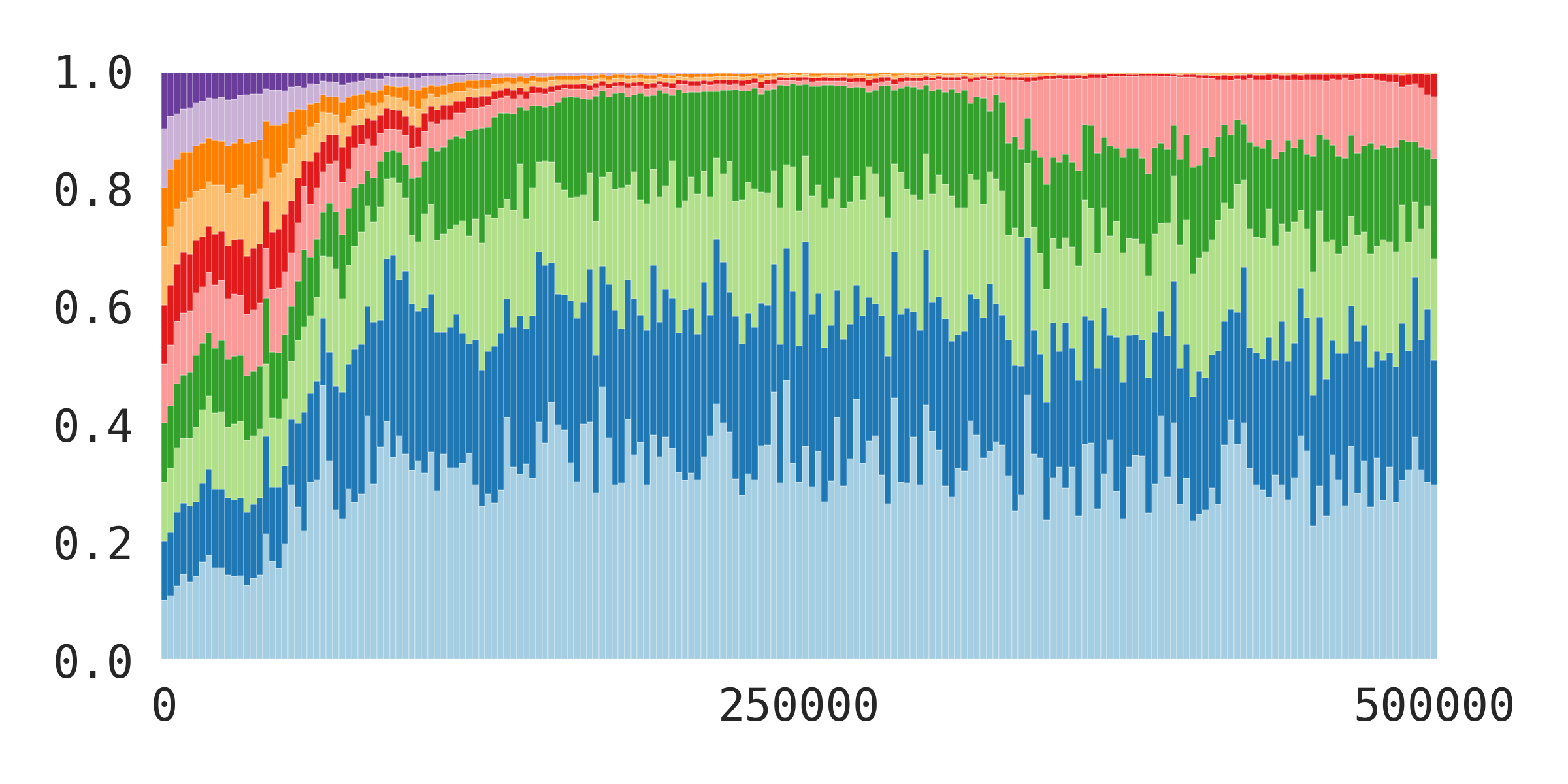}
		\caption{PPO histogram}
	\end{subfigure}
	\caption{Results of the Image Classification Experiment on CIFAR10 dataset.}
	\label{fig:cifar10_exp}
\end{figure}

\subsection{Music Recommendation Experiment}
In this experiment we train a music recommendation agent on the Spotify\footnote{\url{https://www.spotify.com/}} platform.
We obtain data using Spotify Web API\footnote{\url{https://developer.spotify.com/documentation/web-api/}} that we use to train agents to provide personalized recommendations from `Top 50 - Global' playlist based on the user's preferences for musical genres.

The state space $\mathcal{S} = \{0,1\}^{20} \subset \mathbb{R}^{20}$ consists of vectors of user preferences and each user is represented by a vector $s \in \mathcal{S}$ of preferences for musical genres from Table~\ref{tab:spotify_genres} in Appendix~\ref{sec:numerics_appendix}.
Numerically, each coordinate of $s$ is either $0$ or $1$, indicating whether the user likes a particular musical genre or not.
The sparsity of user's preference vectors is between $5\%$ and $25\%$, i.e. each user has a preference to $1$--$5$ out of $20$ available musical genres.

The action space $\mathcal{A} = \{0, 1, \ldots, 49\}$ is a discrete set of tracks to recommend to users.
The available tracks are taken from `Top 50 - Global' playlist presented in Table~\ref{tab:spotify_actions} in Appendix~\ref{sec:numerics_appendix}.
Each track is represented by the vector of $10$ audio features obtained with Spotify Web API, see Figure~\ref{fig:spotify_genres}.
The available features are the following: acousticness, danceability, energy, instrumentalness, liveness, loudness, mode, speechiness, tempo, valence.
Each audio feature is represented by a number from the interval $[0,1]$, indicating how prominent it is in a given track.

For a user $s \in \mathcal{S}$ the suitability of a track $a \in \mathcal{A}$ is determined by its audio features and their relevance to the audio features associated with the user's preferred musical genres displayed in Table~\ref{tab:spotify_features}.
Namely, we compute a user's $s$ preference $p$ for a given track $a$ as
\[
	p(s,a) = \ub{s}{1 \times 10} \times \ub{\mathcal{F_S}}{20 \times 10} \times \ub{\mathcal{F}_a}{10 \times 1} \in \mathbb{R},
\]
where $\mathcal{F_S} \in \mathbb{R}^{20 \times 10}$ and $\mathcal{F}_a \in \mathbb{R}^{10}$ are the mean-normalized audio features of the musical genres and the track $a$ respectively.

The reward for each interaction is determined by how relevant the audio features of the recommended track are to the audio features corresponding to the user's preferred musical genres.
Specifically, in the reward function $r: \mathcal{S \times A} \to \mathbb{R}$ is given as
\[
	r(s,a) = \left\{\begin{array}{rl}
		-1 & \text{if}\quad p(s,a) < -\varepsilon,
		\\
		0 & \text{if}\quad -\varepsilon \le p(s,a) \le \varepsilon,
		\\
		1 & \text{if}\quad p(s,a) > \varepsilon,
		\end{array}\right.
\]
where parameter $\varepsilon > 0$ indicates a user's threshold for providing feedback.
The values $r(s,a) \in \{-1,1\}$ indicate that the user $s$ liked/disliked the track $a$, whereas the reward value of $0$ indicates the absence of user feedback.
In our experiment we use the value $\varepsilon = 0.1$, which results in about $50\%$ average feedback sparsity.

Each agent's policy value, policy batch-entropy, and action selection histogram for this experiment are presented in Figure~\ref{fig:spotify_exp}.
The evaluations are performed on the set of $10,000$ users sampled before the training.
The unsorted action selection histograms for this experiment are presented in Figure~\ref{fig:spotify_dist_raw} in Appendix~\ref{sec:numerics_appendix}.
We note that the data used in this experiment, presented in Tables~\ref{tab:spotify_genres}--\ref{tab:spotify_actions}, was obtained on 16 August 2023 and, since Spotify playlists are not stationary, will likely differ if obtained at a later date.

\begin{figure}
	\centering
	\begin{subfigure}{.54\linewidth}
		\includegraphics[width=\linewidth]{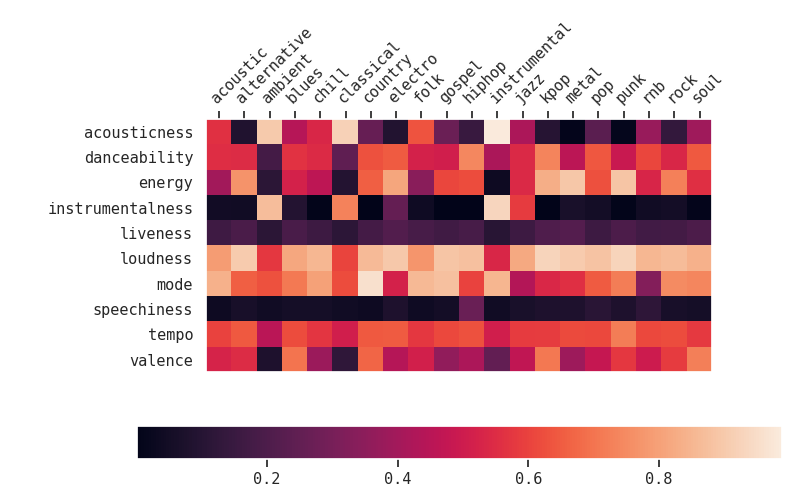}
		\caption{Parameterization of audio features}
	\end{subfigure}
	\begin{subfigure}{.45\linewidth}
		\includegraphics[width=\linewidth]{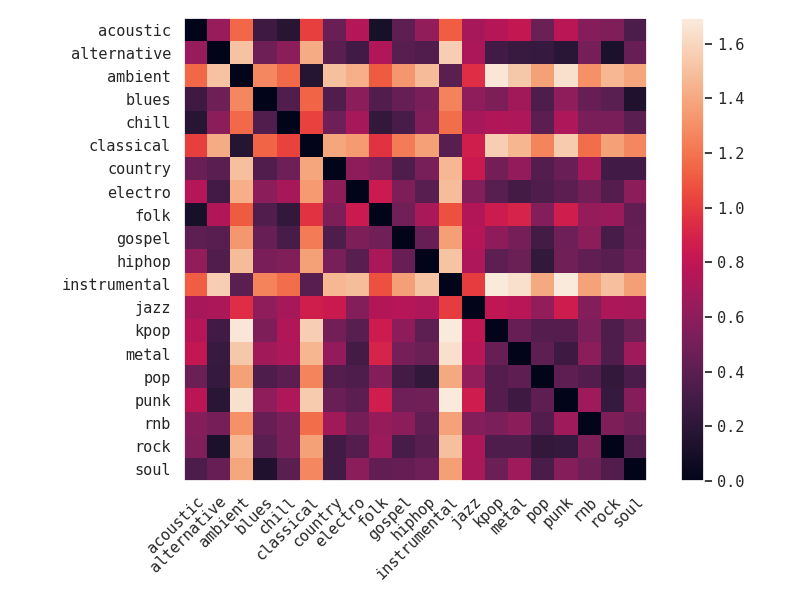}
		\caption{Pair-wise differences of audio features}
	\end{subfigure}
	\caption{Audio features of the musical genres used in the Music Recommendation Experiment.}
	\label{fig:spotify_genres}
\end{figure}

\begin{figure}[h]
	\centering
	\begin{subfigure}{.49\linewidth}
		\includegraphics[width=\linewidth]{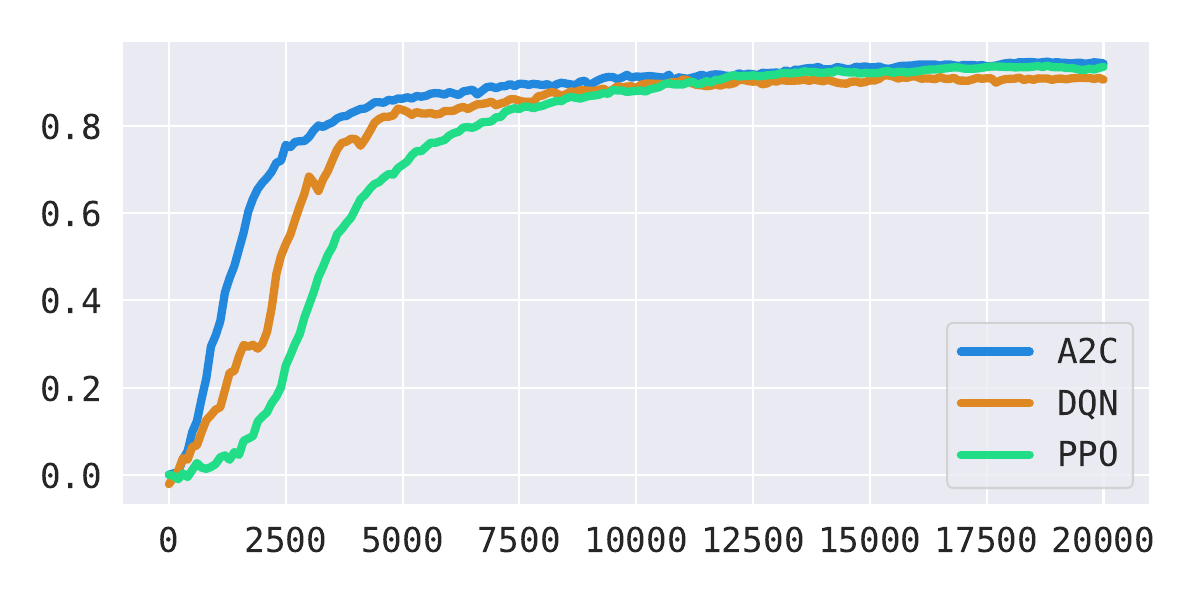}
		\caption{Policy value}
	\end{subfigure}
	\begin{subfigure}{.49\linewidth}
		\includegraphics[width=\linewidth]{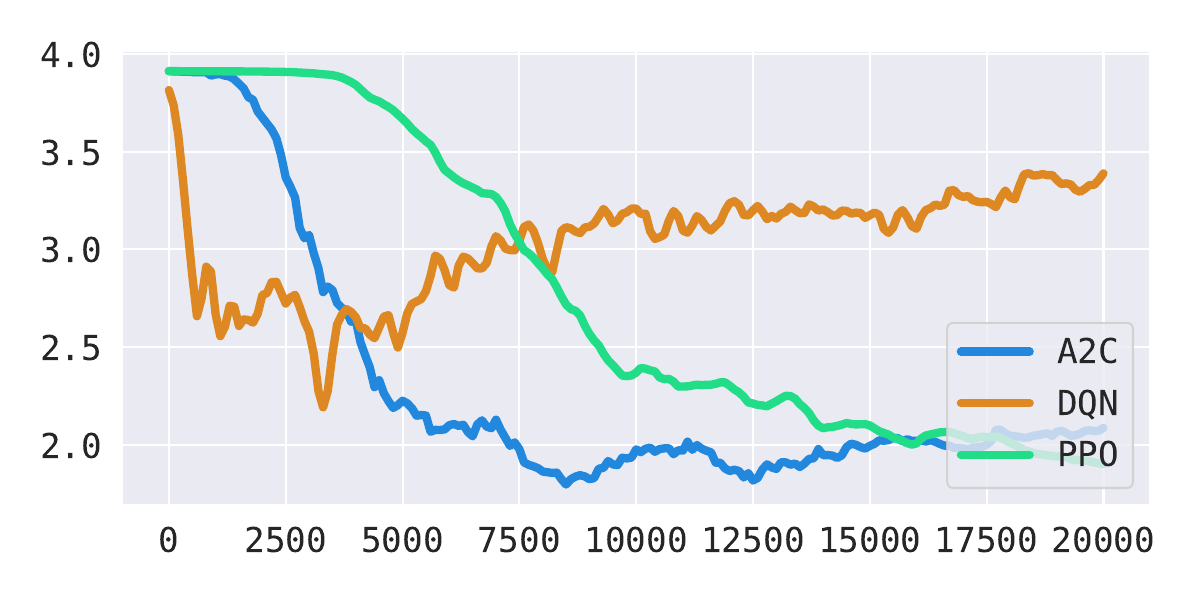}
		\caption{Policy batch-entropy}
	\end{subfigure}
	\\
	\begin{subfigure}{.32\linewidth}
		\includegraphics[width=\linewidth]{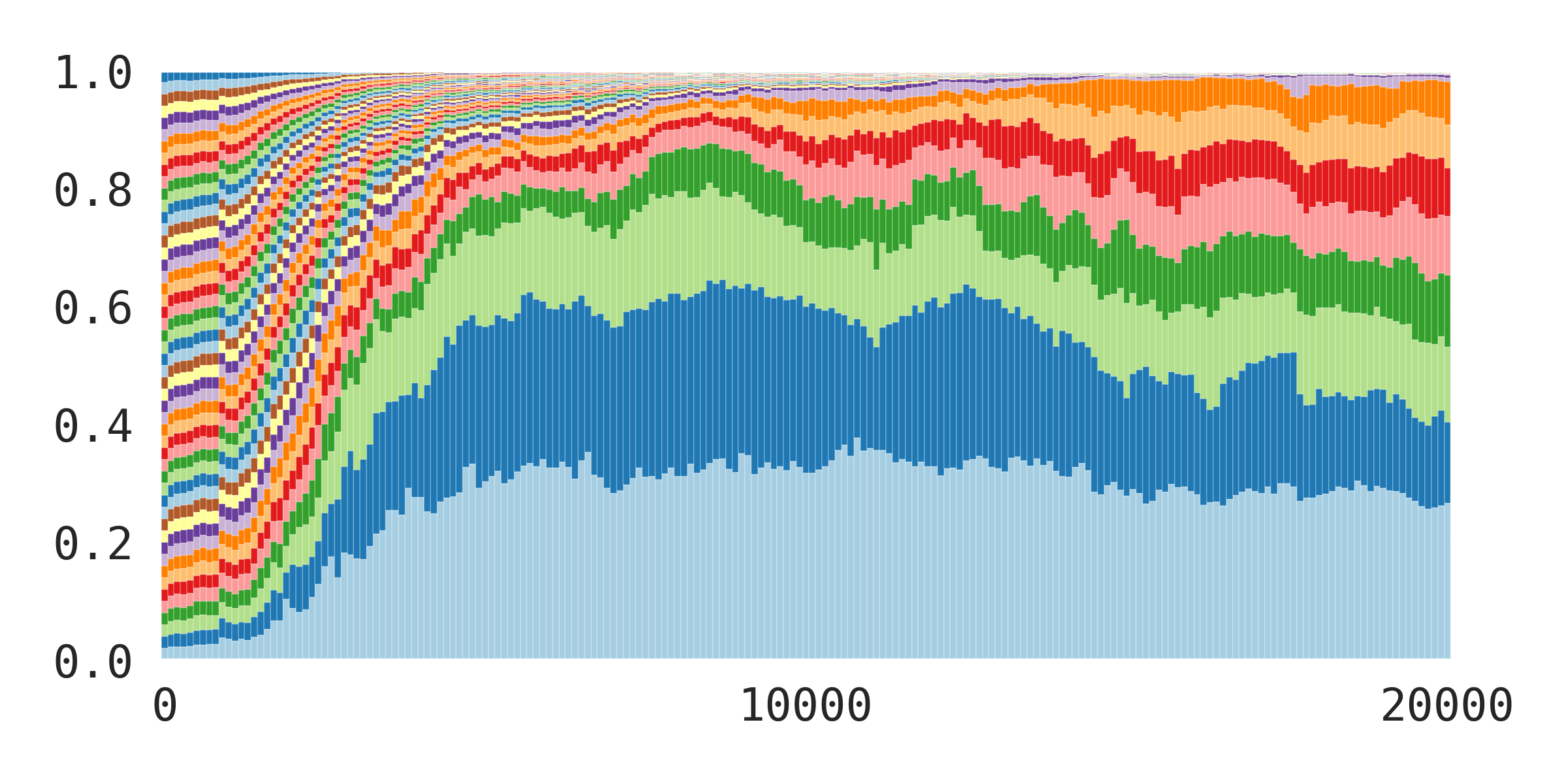}
		\caption{A2C histogram}
	\end{subfigure}
	\begin{subfigure}{.32\linewidth}
		\includegraphics[width=\linewidth]{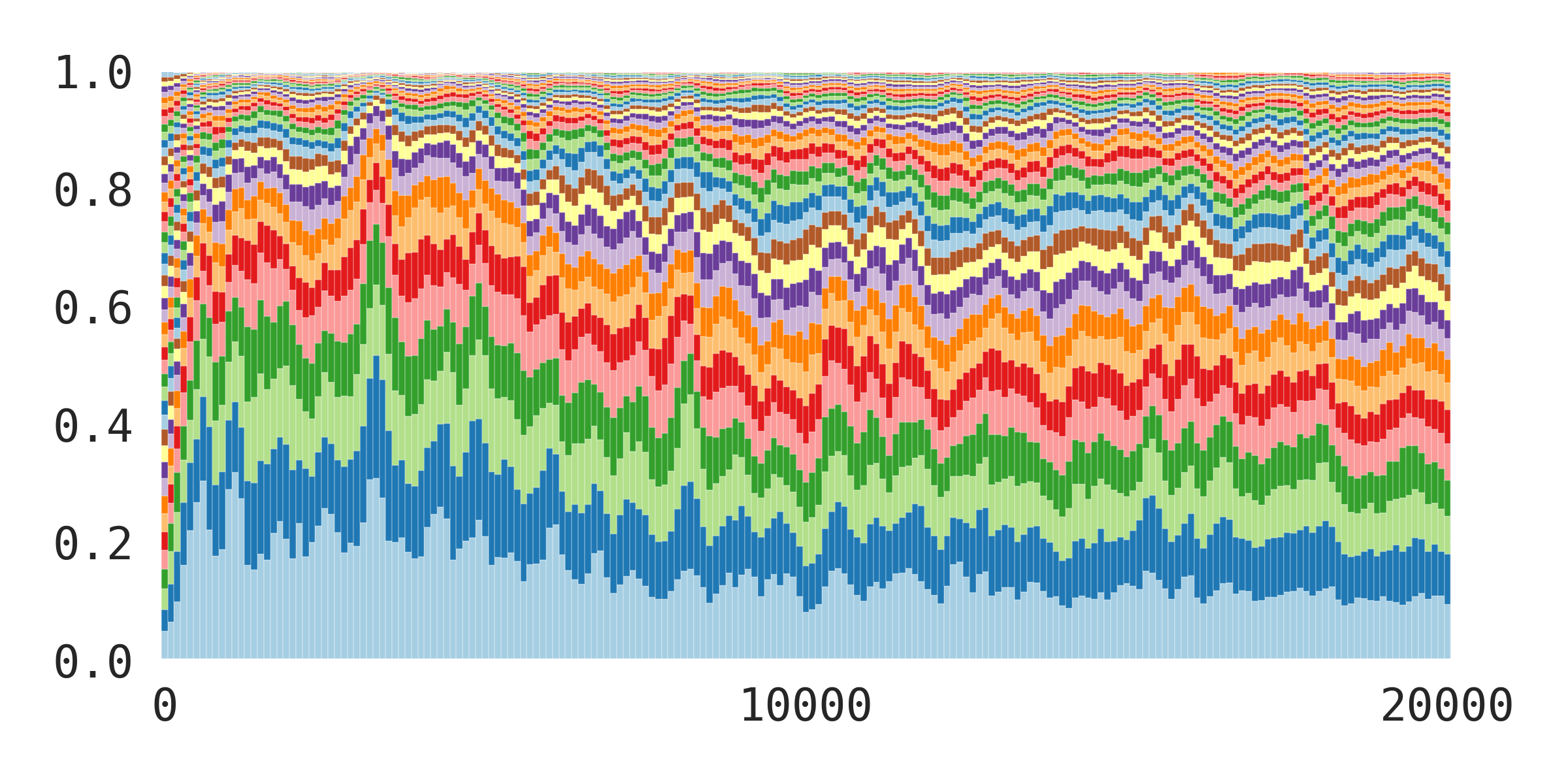}
		\caption{DQN histogram}
	\end{subfigure}
	\begin{subfigure}{.32\linewidth}
		\includegraphics[width=\linewidth]{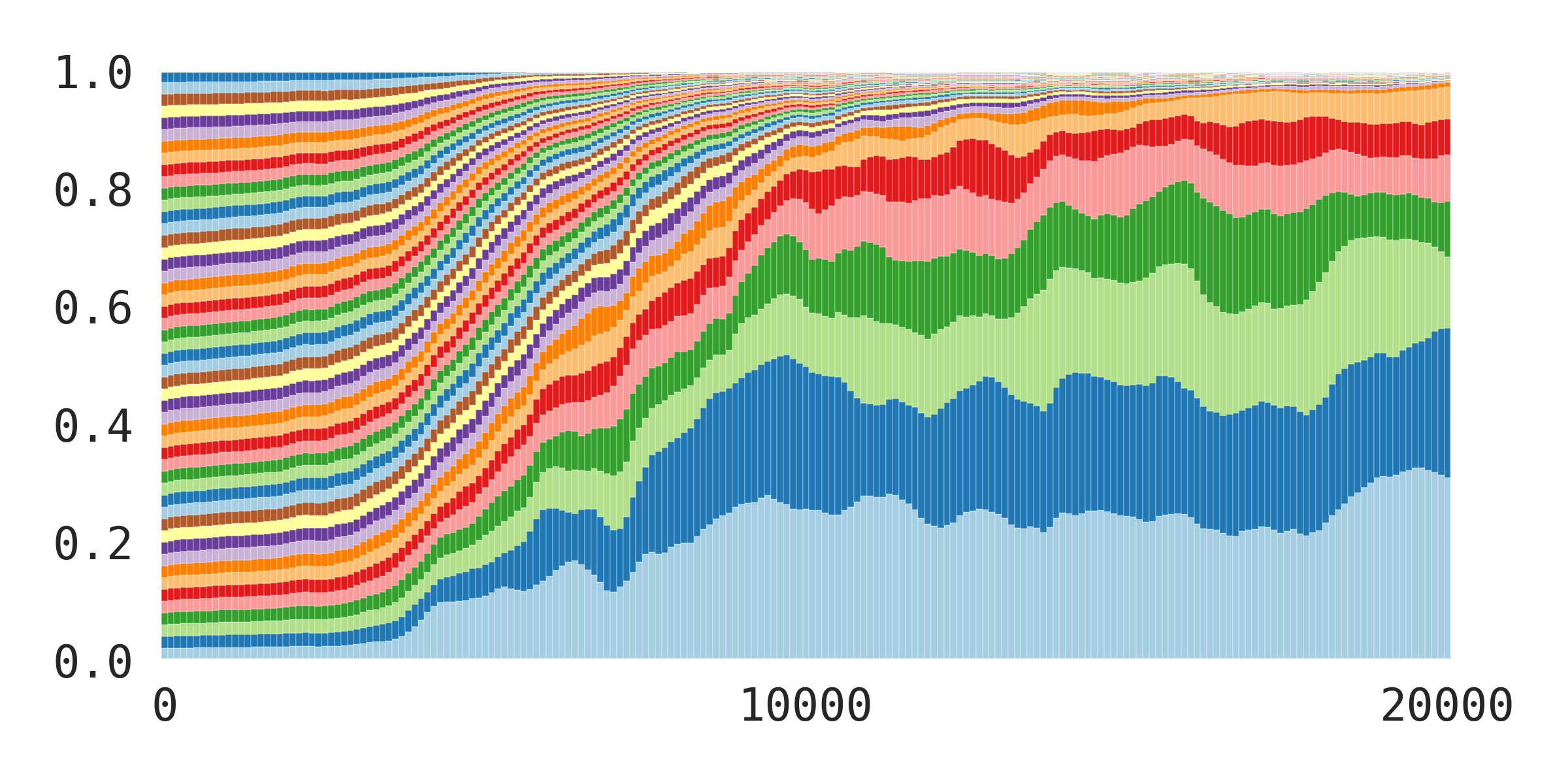}
		\caption{PPO histogram}
	\end{subfigure}
	\caption{Results of the Music Recommendation Experiment.}
	\label{fig:spotify_exp}
\end{figure}

\subsection{Online Advertisement Experiment}
In this experiment we train reinforcement learning agents to provide product recommendations based on users' preferences.
We use the RecoGym\footnote{\url{https://github.com/criteo-research/reco-gym}} environment~\cite{rohde2018recogym} to simulate the interactions between users and the recommended products.

We set the numbers of products to $50$ and add a user preference simulation to form a proper contextual bandit environment, and otherwise use the default parameters set by the authors.
Specifically, we set the state space $\mathcal{S} = \mathbb{R}^{50}$ to represent users' shopping preferences which determines the likelihood of clicking on an advertised product.
The simulated preferences are sampled from the normal distribution, i.e.
\[
	s = [s_1, \ldots, s_{50}] \in \mathcal{S}
	\text{ is such that }
	s_i \sim \mathcal{N}(0,1)
	\text{ for any }
	1 \le i \le 50.
\]

The action space $\mathcal{A} = \{0, 1, \ldots, 49\}$ consists of $50$ products that are being shown to users.
The order of the available products $\mathcal{A}$ is then permuted in accordance with the user's preferences to ensure variety in users' feedback to different products.

The reward function $r : \mathcal{S \times A} \to \{0, 1\}$ represents the user's $s \in \mathcal{S}$ response to observing an advertisement for the product $a \in \mathcal{A}$, where the reward value of $1$ indicates that the user clicked on the advertisement, and the reward value of $0$ indicates the absence of the user's reaction.
The reward function $r : \mathcal{S \times A} \to \mathbb{R}$ is given as
\[
	r(s,a) = \left\{\begin{array}{cl}
				1 &\text{ if the user } s \text{ clicked the product } a,
				\\
				0 &\text{ if the user } s \text{ didn't click the product } a.
			\end{array}\right.
\]
We note that the exact mechanics by which the reward is assigned is established by the RecoGym's developers and is outside the scope of this work.

Each agent's policy value, policy batch-entropy, and action selection histogram for this experiment are presented in Figure~\ref{fig:recogym_exp}.
The evaluations are performed on the set of $10,000$ users sampled before the training.
The hyperparameter choices for each agent are listed in Appendix~\ref{sec:numerics_appendix} in Table~\ref{tab:recogym_params} and any unspecified parameters are kept at their default values.
The unsorted action selection histograms for each agent are presented in Figure~\ref{fig:recogym_dist_raw} in Appendix~\ref{sec:numerics_appendix}.

\begin{figure}[h]
	\centering
	\begin{subfigure}{.49\linewidth}
		\includegraphics[width=\linewidth]{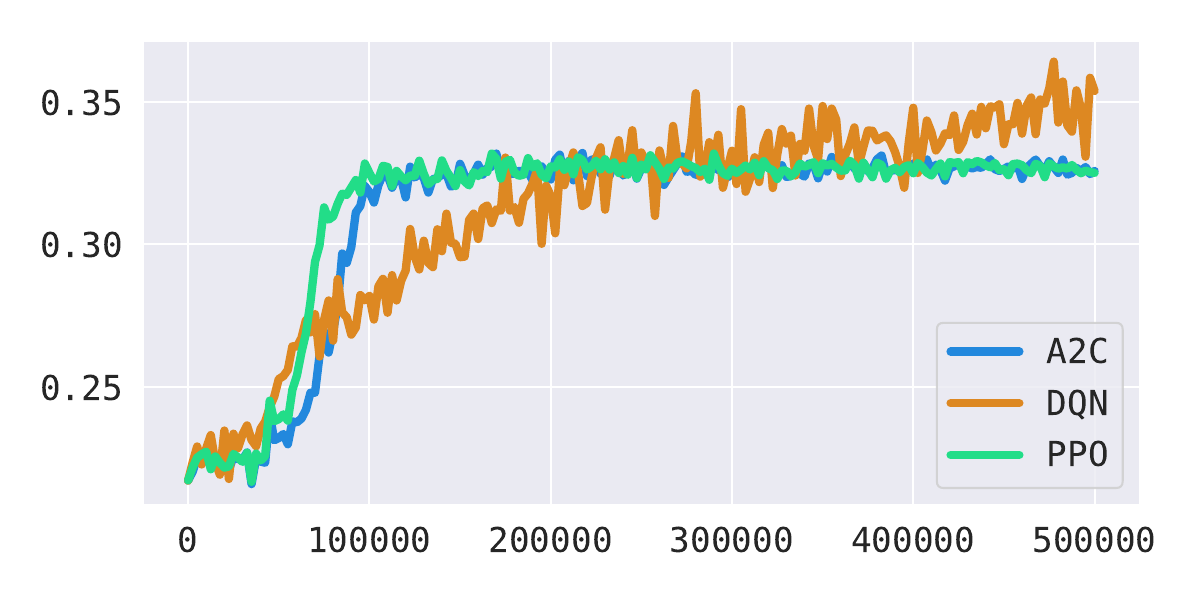}
		\caption{Policy value}
	\end{subfigure}
	\begin{subfigure}{.49\linewidth}
		\includegraphics[width=\linewidth]{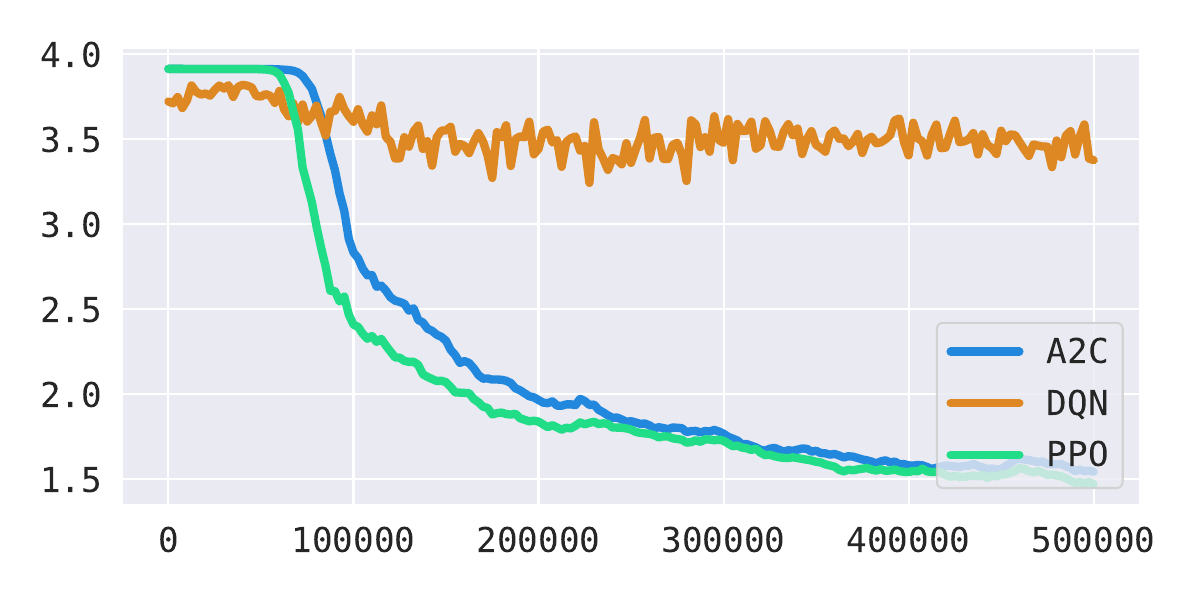}
		\caption{Policy batch-entropy}
	\end{subfigure}
	\\
	\begin{subfigure}{.32\linewidth}
		\includegraphics[width=\linewidth]{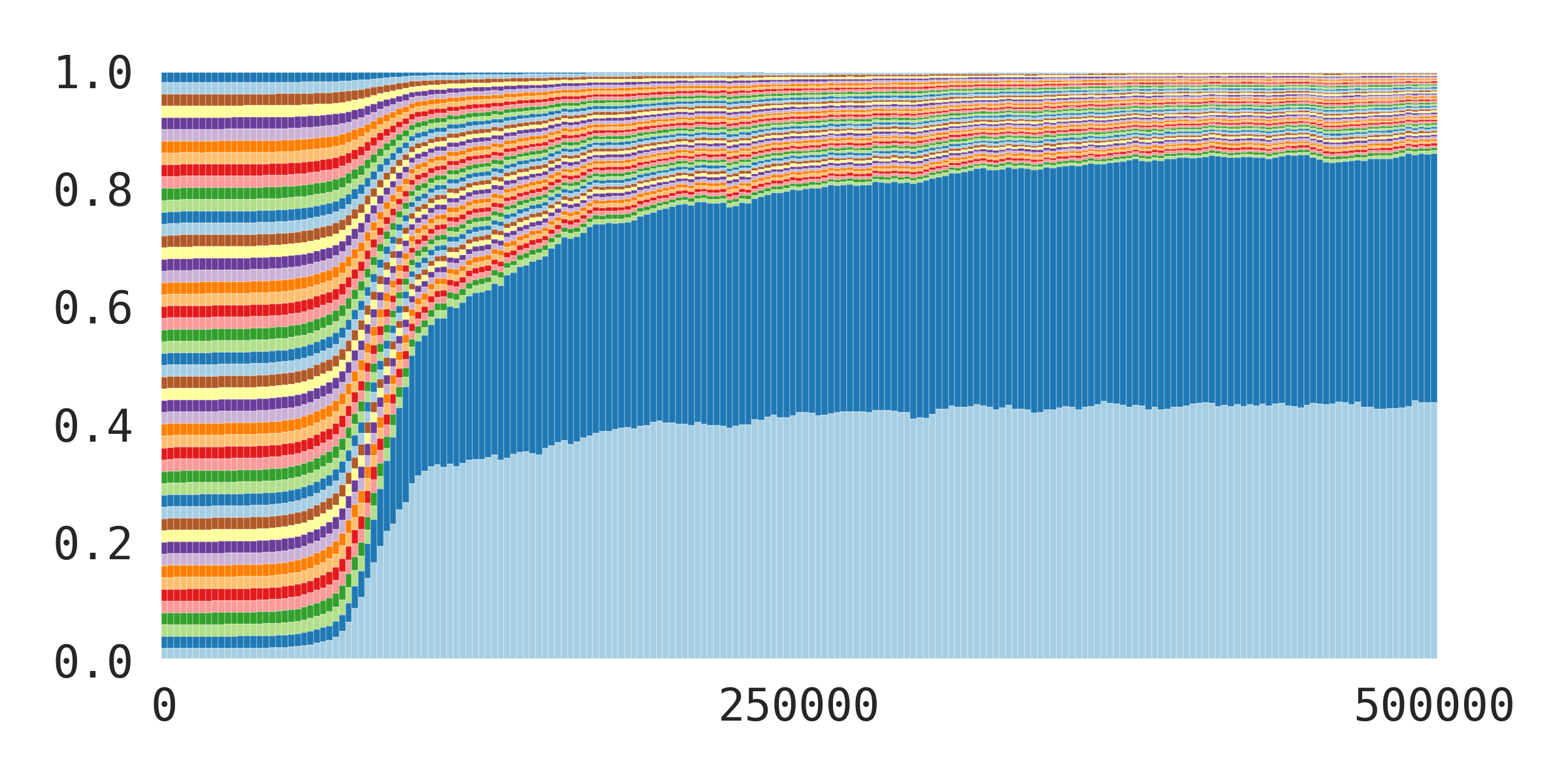}
		\caption{A2C histogram}
	\end{subfigure}
	\begin{subfigure}{.32\linewidth}
		\includegraphics[width=\linewidth]{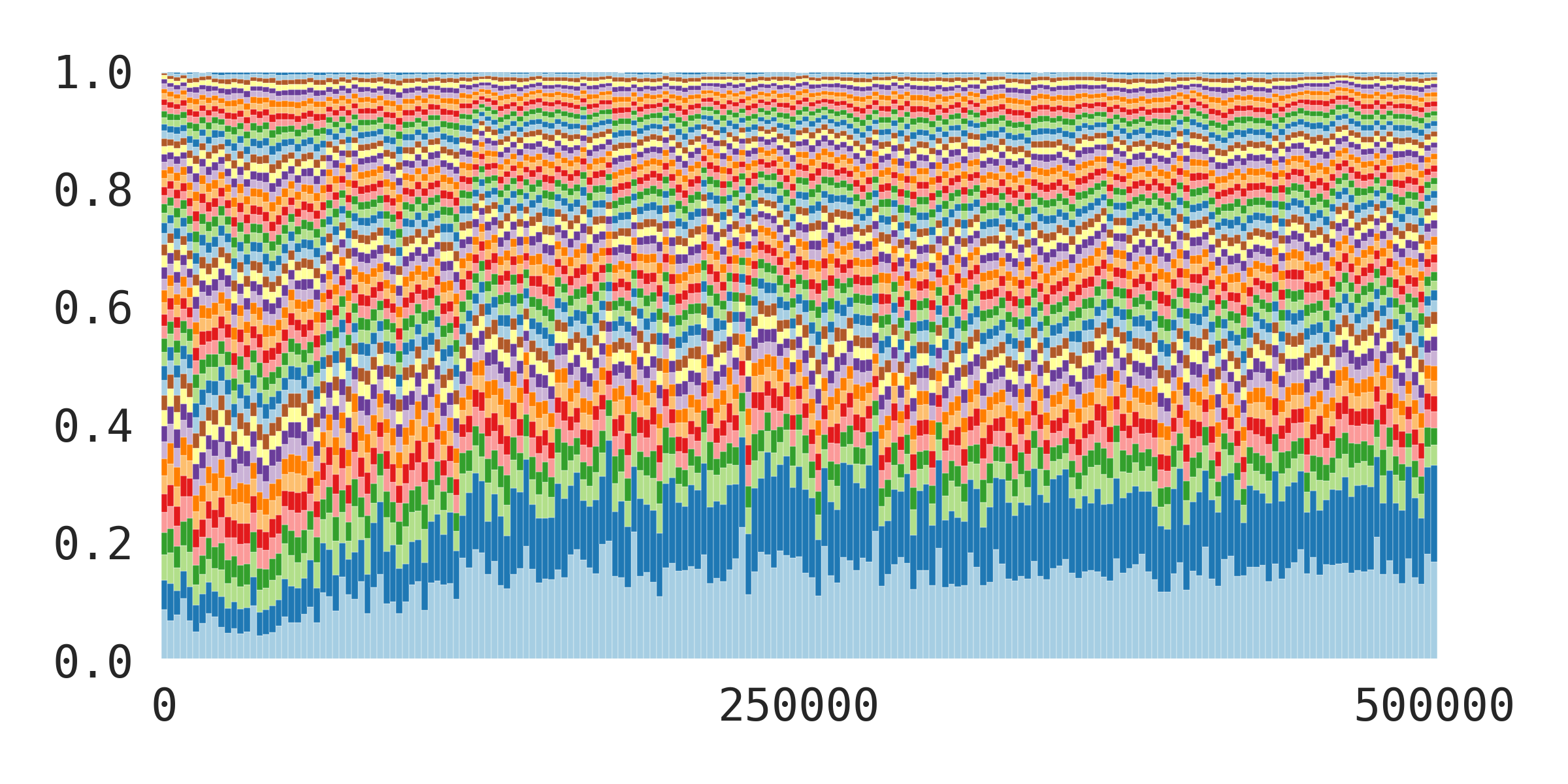}
		\caption{DQN histogram}
	\end{subfigure}
	\begin{subfigure}{.32\linewidth}
		\includegraphics[width=\linewidth]{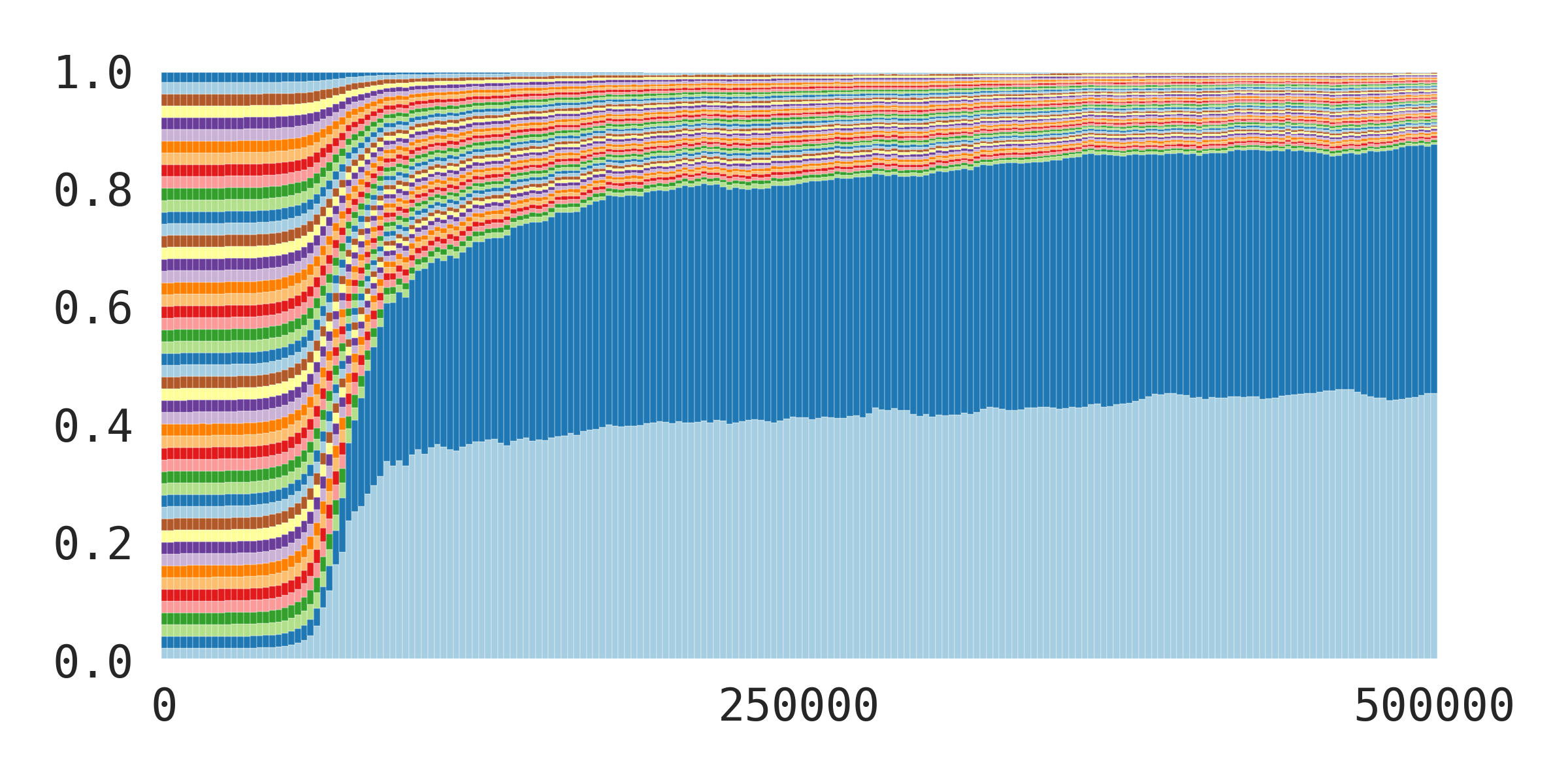}
		\caption{PPO histogram}
	\end{subfigure}
	\caption{Results of the Online Advertisement Experiment.}
	\label{fig:recogym_exp}
\end{figure}

\subsection{Behavioral Preference Experiment}
In this experiment we deploy reinforcement learning agents on Synthetic Personalization Environment\footnote{\url{https://github.com/sukiboo/personalization_wain21}} to learn users' behavioral preferences from the available data.
Synthetic Personalization Environment is introduced in~\cite{dereventsov2021unreasonable} as a human behavioral preference simulator.
We use the default environment parameters set by the authors and train agents on an unmodified reward signal.

The state space $\mathcal{S} = \mathbb{R}^{100}$ describes the user's behavior, which determines how the user will respond to the provided recommendation.
The action space $\mathcal{A} = \{0, 1, \ldots, 99\}$ consists of $100$ behavioral recommendations that can be given to users.
The reward function $r : \mathcal{S \times A} \to \mathbb{R}$ represents the user's $s \in \mathcal{S}$ response to the behavioral recommendation $a \in \mathcal{A}$, where the higher reward indicates the higher chance of the behavioral adoption.
The calculation of the reward signal is established by the Synthetic Personalization Environment's developers and schematically demonstrated in Figure~\ref{fig:personalization_reward}, though the exact formulation is outside of the scope of this work.

Each agent's policy value, policy batch-entropy, and action selection histogram for this experiment are presented in Figure~\ref{fig:personalization_exp}.
The evaluations are performed on the set of $10,000$ users sampled before the training.
The hyperparameter choices for each agent are listed  in Appendix~\ref{sec:numerics_appendix} in Table~\ref{tab:personalization_params} and any unspecified parameters are kept at their default values.
The unsorted action selection histograms for each agent are presented in Figure~\ref{fig:personalization_dist_raw} in Appendix~\ref{sec:numerics_appendix}.

\begin{figure}
	\centering
	\includegraphics[width=.8\linewidth]{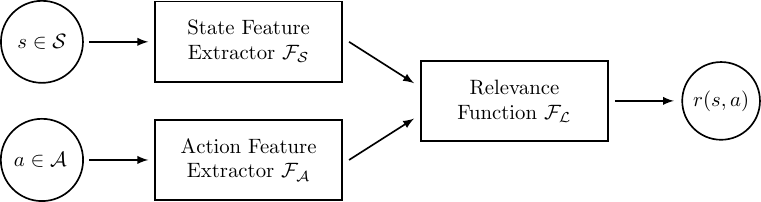}
	\caption{Schematic structure of the reward function used in the Behavioral Preference Experiment.}
	\label{fig:personalization_reward}
\end{figure}

\begin{figure}[h]
	\centering
	\begin{subfigure}{.49\linewidth}
		\includegraphics[width=\linewidth]{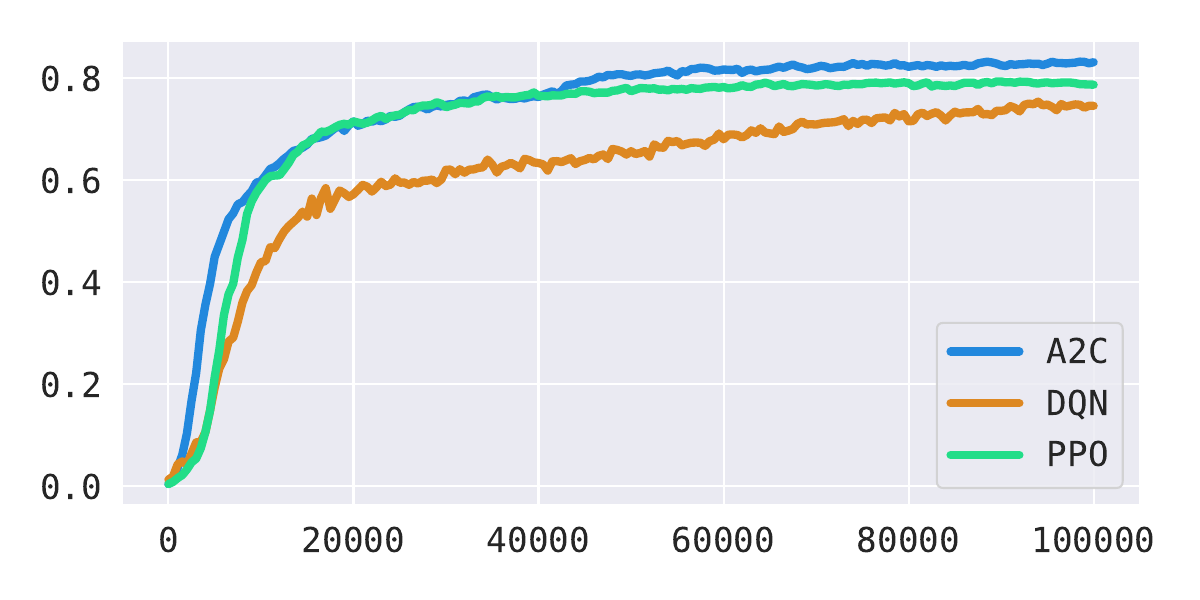}
		\caption{Policy value}
	\end{subfigure}
	\begin{subfigure}{.49\linewidth}
		\includegraphics[width=\linewidth]{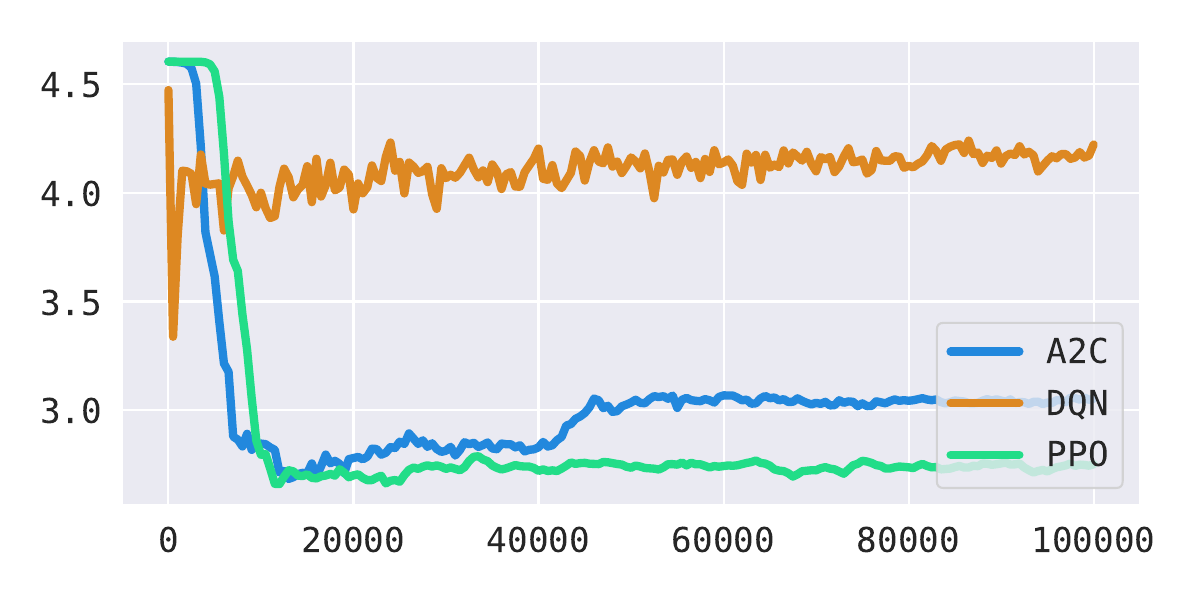}
		\caption{Policy batch-entropy}
	\end{subfigure}
	\\
	\begin{subfigure}{.32\linewidth}
		\includegraphics[width=\linewidth]{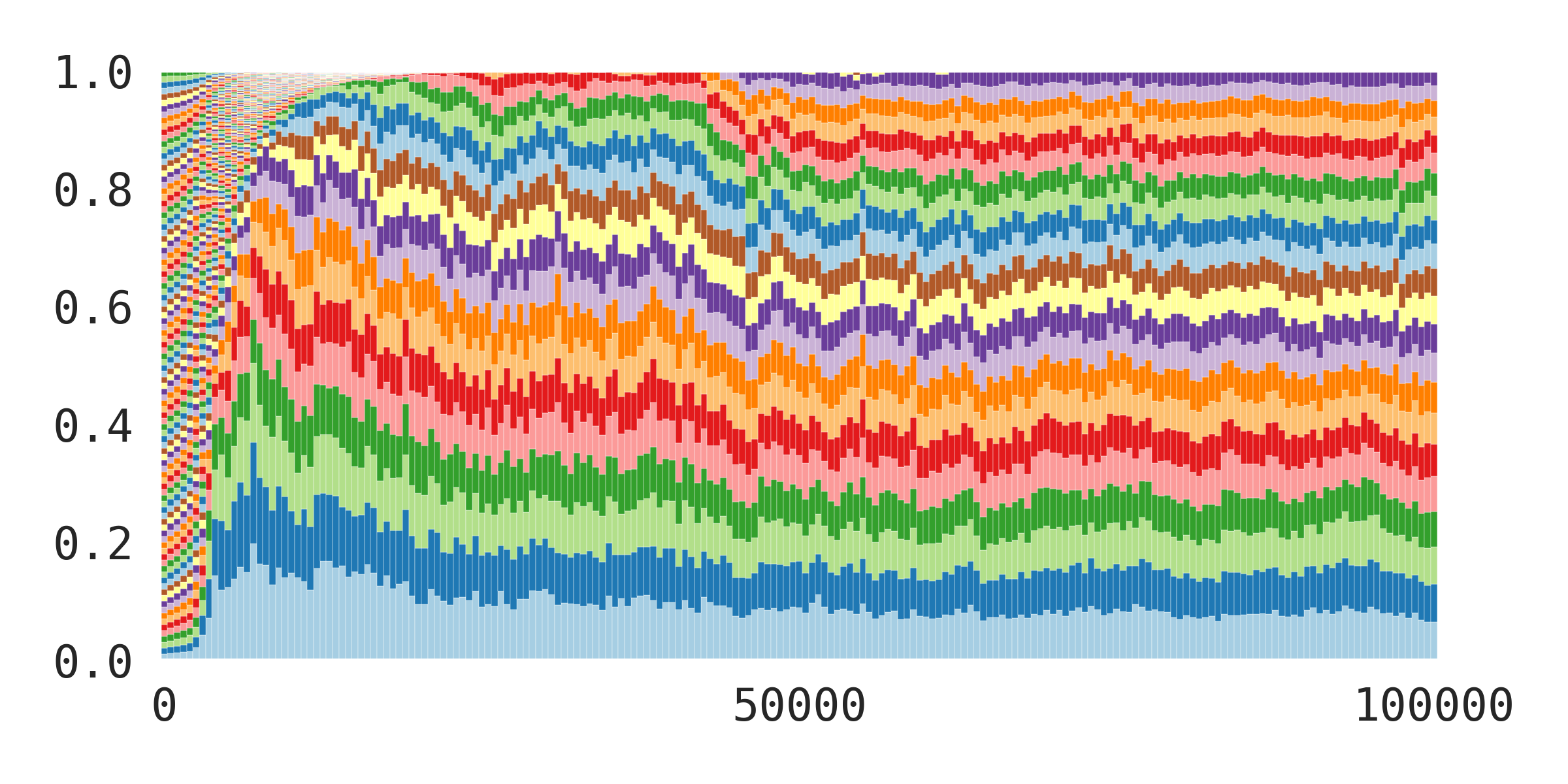}
		\caption{A2C histogram}
	\end{subfigure}
	\begin{subfigure}{.32\linewidth}
		\includegraphics[width=\linewidth]{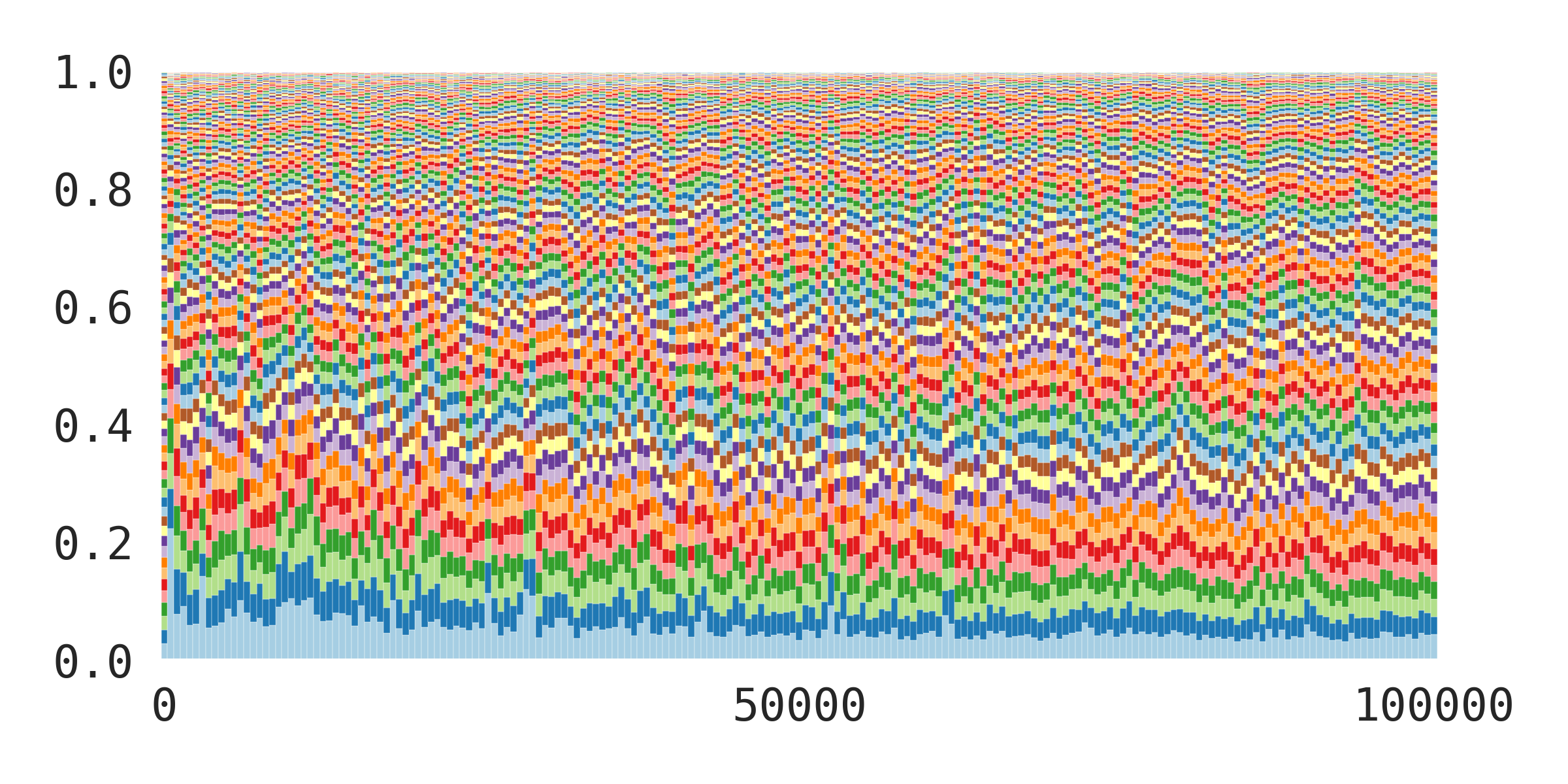}
		\caption{DQN histogram}
	\end{subfigure}
	\begin{subfigure}{.32\linewidth}
		\includegraphics[width=\linewidth]{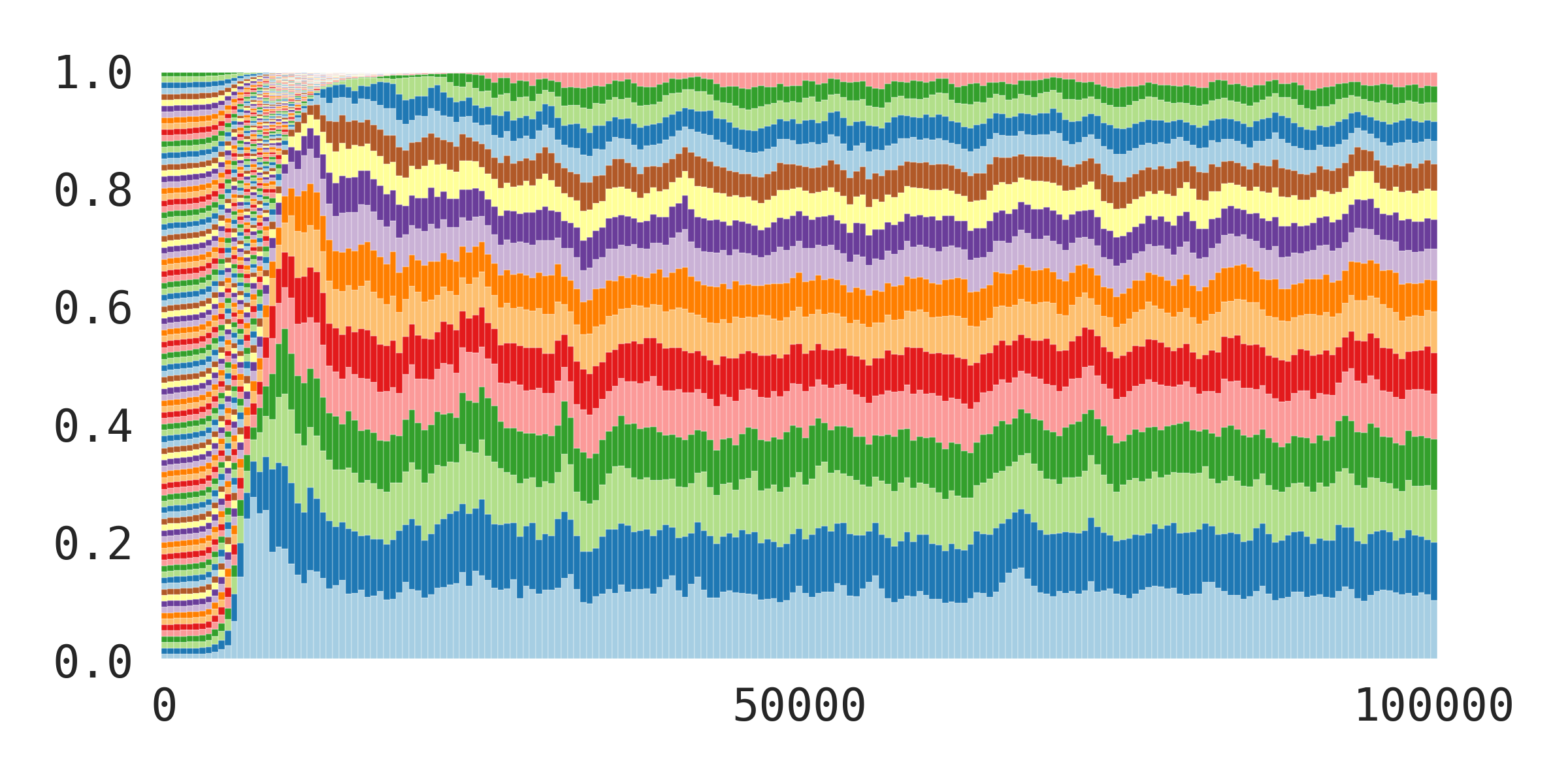}
		\caption{PPO histogram}
	\end{subfigure}
	\caption{Results of the Behavioral Preference Experiment.}
	\label{fig:personalization_exp}
\end{figure}

\section{Theoretical Analysis}\label{sec:theory}
In this section we state our main theoretical results regarding the behavior of the Policy Optimization and Q-learning agents on the contextual bandit environments.
The proofs for the stated results are provided in Appendix~\ref{sec:theory_appendix}.

We begin by establishing the network output changes for the linear policy agents, which compliments the sample experiment presented in Section~\ref{sec:sample}.
Note that this result showcases that the updates of the PG-agent's network depend on the action selection probability $\pi(k|s)$, whereas the QL-agent's network does not.
Even though this is a simplified setting, this explicit dependence on the action selection probability ultimately leads to the drop in policy entropy.
thus
\begin{theorem}\label{thm:outputs_linear}
Let a PG or QL agent with a linear network be trained on a contextual bandit task $\{\mathcal{S,A},r\}$ with the gradient descend algorithm with a batch size $N$ and a learning rate $\lambda$.
Then with each update of the trainable weights $W \leftarrow W^\prime$ the policy output mappings $\mathcal{Z} : \mathcal{S} \to \mathbb{R}^K$ change as follows:
\[
	\mathcal{Z}(x;W^\prime) = \mathcal{Z}(x;W) 
	+ \frac{\lambda}{N} \sum_{n=1}^N \ub{\Omega(s_n,a_n,r_n)}{K \times d} \times x,
\]
where
\begin{align*}
	\Omega_{pg}(s,a,r)
	&= r \Big[ \big(\mathbbm{1}(a=k) - \pi(k|s)\big) \, s \Big]_{k=1}^K,
	\\
	\Omega_{ql}(s,a,r)
	&= 2\big( r - z_a(s) \big) \Big[ \mathbbm{1}(a=k) \, \langle s, x \rangle \Big]_{k=1}^K,
\end{align*}
and $\{(s_n,a_n,r_n)\}_{n=1}^N$ is the batch of agent-environment interactions on which the the update $\Theta \leftarrow \Theta^\prime$ is performed.
\end{theorem}

Next, we establish analogous update rules in general contextual bandit environments for the agents utilized in Section~\ref{sec:numerics}.
Evidently the complexity of the considered algorithms makes the stated result far less interpretable, though the main idea remains: the explicit dependency on the action selection probability in the policy optimization agents motivates the optimizer to ``optimize out'' some of the available actions, which leads to lower entropy policies.

\begin{theorem}\label{thm:outputs}
Let an A2C, DQN, or PPO agent be trained on a contextual bandit task $\{\mathcal{S,A},r\}$ with the gradient descent algorithm with a batch size $N$ and a learning rate $\lambda$.
Then with each update of the trainable weights $\Theta \leftarrow \Theta^\prime$ the policy output mappings $\mathcal{Z} : \mathcal{S} \to \mathbb{R}^K$ change as follows:
\[
	\mathcal{Z}(\,\cdot\,;\Theta^\prime)
	= \mathcal{Z}(\,\cdot\,;\Theta)
	+ \nabla\mathcal{Z}(\,\cdot\,;\Theta) \times \frac{\lambda}{N} \sum_{n=1}^N \Omega(s_n,a_n,r_n)^T
	+ \mathcal{O}(\|\Theta - \Theta^\prime\|^2),
\]
where
\begin{align*}
	\Omega_{a2c}(s,a,r)
	&= \big(r - v(s;\Theta) \big)
	\Big(\nabla v(s;\Theta)
	+ \big[\mathbbm{1}(a=k) - \pi(k|s)\big]_{k=1}^K \times \nabla\mathcal{Z}(s;\Theta) \Big),
	\\
	\Omega_{dqn}(s,a,r)
	&= 2\big( r - z_a(s;\Theta) \big) \nabla z_a(s;\Theta),
	\\
	\Omega_{ppo}(s,a,r)
	&= \big( r - v(s) \big)
	\cdot \Bigg( \nabla v(s;\Theta)
	+ \big[\mathbbm{1}(a=k) - \pi(k|s)\big]_{k=1}^K
	\times \nabla\mathcal{Z}(s;\Theta)
	\cdot \rho(a,s) \Bigg),
\end{align*}
where
\[
	\rho(a,s) = \left\{\begin{array}{ll}
		\frac{\pi(a|s)}{\widetilde{\pi}(a|s)}
		& \text{ if } \left| 1 - \frac{\pi(a|s)}{\widetilde{\pi}(a|s)} \right| < \epsilon
		\\
		0 & \text{ otherwise}
	\end{array}\right.
\]
and $\{(s_n,a_n,r_n)\}_{n=1}^N$ is the batch of agent-environment interactions on which the the update $\Theta \leftarrow \Theta^\prime$ is performed.
\end{theorem}

\section{Conclusion}\label{sec:conclusion}
In this paper, we explore the behavior of reinforcement learning agents employing different learning paradigms in personalization contextual bandit tasks.
Our findings reveal that the trajectory of the learned policy is significantly influenced by the type of learning algorithm used.
Specifically, we observe that Policy Optimization agents tend to develop low-entropy policies during training, leading them to prioritize certain actions over others.
In contrast, Q-Learning agents maintain higher-entropy policies throughout training, which may be more beneficial in diverse real-world applications. 
These observations are supported by numerical experiments and are further analyzed theoretically in Appendix~\ref{sec:theory}, providing new insights into the dynamics of policy entropy in reinforcement learning systems.

\subsection{Restrictions and Limitations}
While our study provides both practical and analytical insights into the behavior of reinforcement learning agents, it does have certain limitations. 
One significant limitation is the absence of entropy regularization or other optimization relaxation techniques in our experimental setup.
We selected this approach to focus on understanding the fundamental behaviors of RL agents without the influence of these techniques; however, this choice may limit the direct applicability of our findings to real-world scenarios where such techniques are commonly employed to enhance performance and ensure robustness.

\subsection{Future Work}
The findings from this study open several avenues for future research.
One immediate area of interest is the integration of entropy regularization techniques into the training of RL agents.
Investigating how these techniques affect policy entropy and overall agent performance could provide deeper insights and potentially lead to more effective RL applications.
Additionally, expanding this research to include other types of reinforcement learning algorithms and comparing their behavior in similar tasks could further enrich our understanding of the dynamics at play.
We intend to pursue these research directions to not only address the limitations of the current study but also to expand the scope of our understanding of reinforcement learning in complex environments.

\bibliographystyle{abbrv}
\bibliography{references}

\clearpage
\appendix
\section*{Appendix}

\section{Numerical Experiments}\label{sec:numerics_appendix}
In this section we provide additional details for the numerical experiments presented in Section~\ref{sec:numerics}.
The source code and the required datasets are available at~\url{https://github.com/sukiboo/policy_entropy}.

\begin{table}[h]
	\centering
	\begin{tabular}{cccc}
		\toprule
		 & \multicolumn{3}{c}{Parameters}
		\\\cmidrule{2-4}
		Agent & Architecture & Batch size & Learning rate
		\\\midrule
		A2C & [64,64] & 32 & 1e-3
		\\
		DQN & [64,64] & 32 & 3e-4
		\\
		PPO & [64,64] & 32 & 1e-3
		\\\bottomrule
	\end{tabular}
	\caption{Hyperparameter selection for the Image Classification Experiment on MNIST dataset.}
	\label{tab:mnist_params}
\end{table}

\begin{table}[h]
	\centering
	\begin{tabular}{cccc}
		\toprule
		 & \multicolumn{3}{c}{Parameters}
		\\\cmidrule{2-4}
		Agent & Architecture & Batch size & Learning rate
		\\\midrule
		A2C & [128,128] & 32 & 1e-4
		\\
		DQN & [128,128] & 32 & 5e-6
		\\
		PPO & [128,128] & 32 & 1e-4
		\\\bottomrule
	\end{tabular}
	\caption{Hyperparameter selection for the Image Classification Experiment on CIFAR10 dataset.}
	\label{tab:cifar10_params}
\end{table}

\begin{figure}[h]
	\centering
	\begin{subfigure}{.32\linewidth}
		\includegraphics[width=\linewidth]{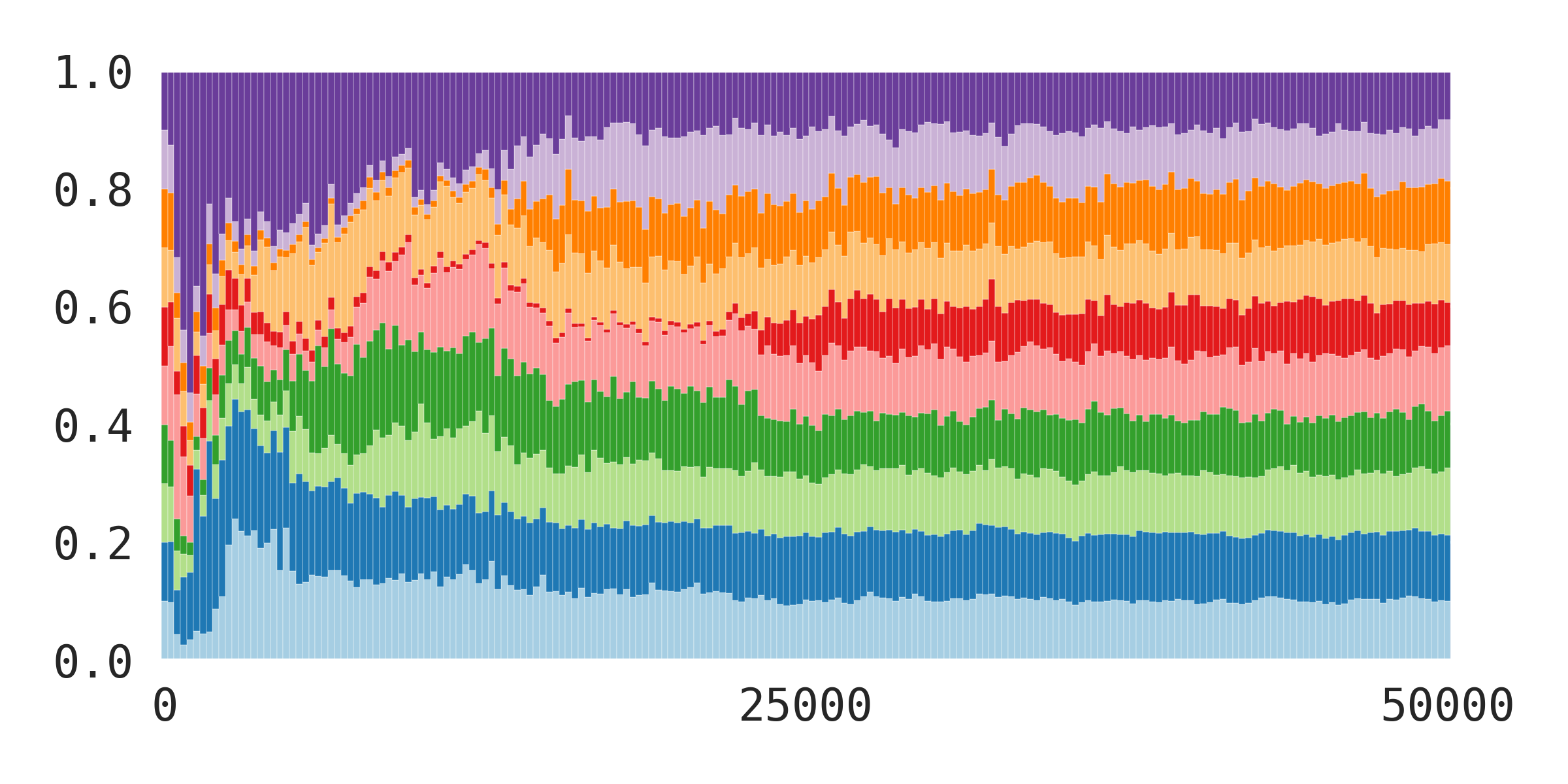}
		\caption{A2C histogram}
	\end{subfigure}
	\begin{subfigure}{.32\linewidth}
		\includegraphics[width=\linewidth]{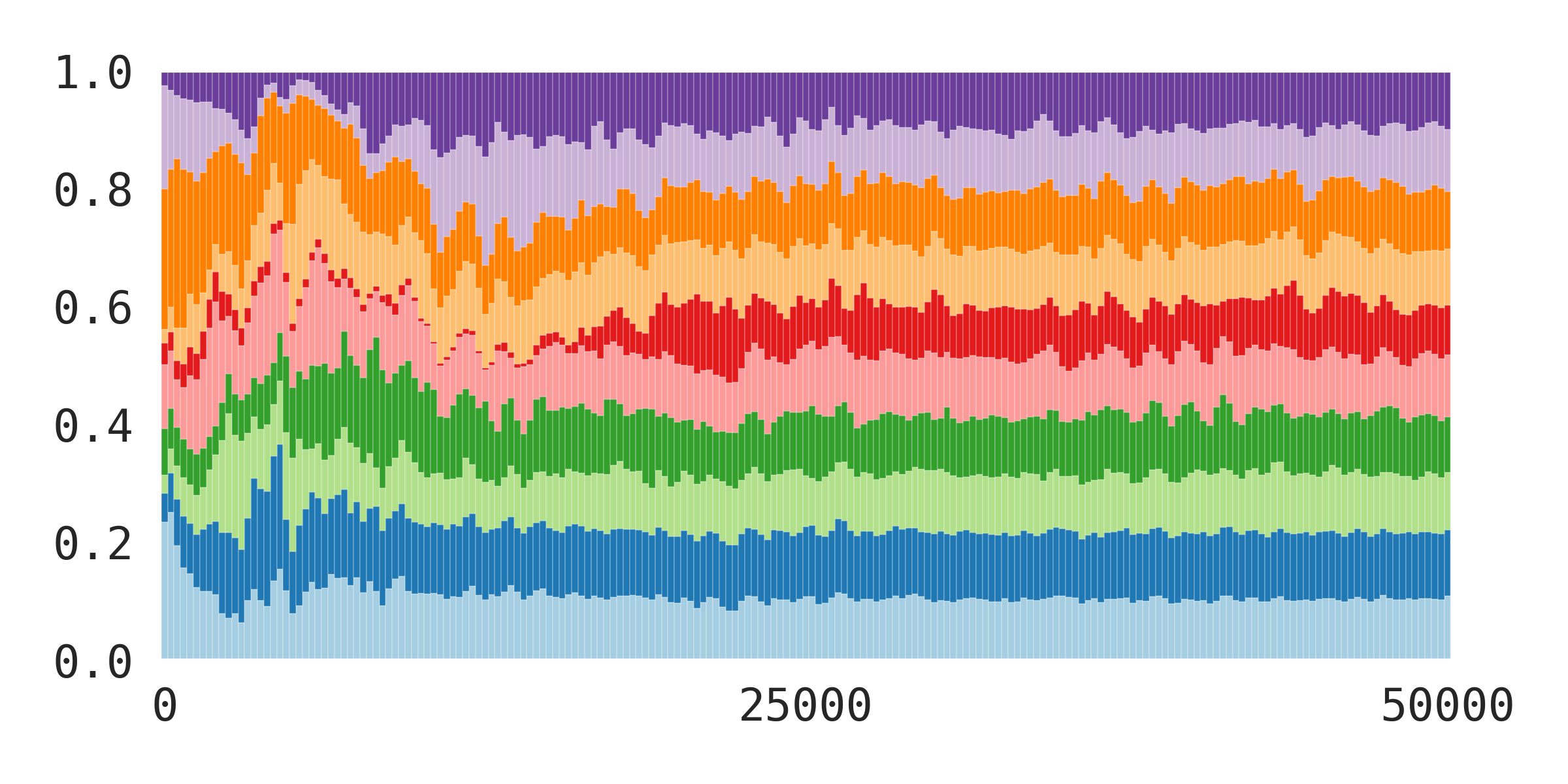}
		\caption{DQN histogram}
	\end{subfigure}
	\begin{subfigure}{.32\linewidth}
		\includegraphics[width=\linewidth]{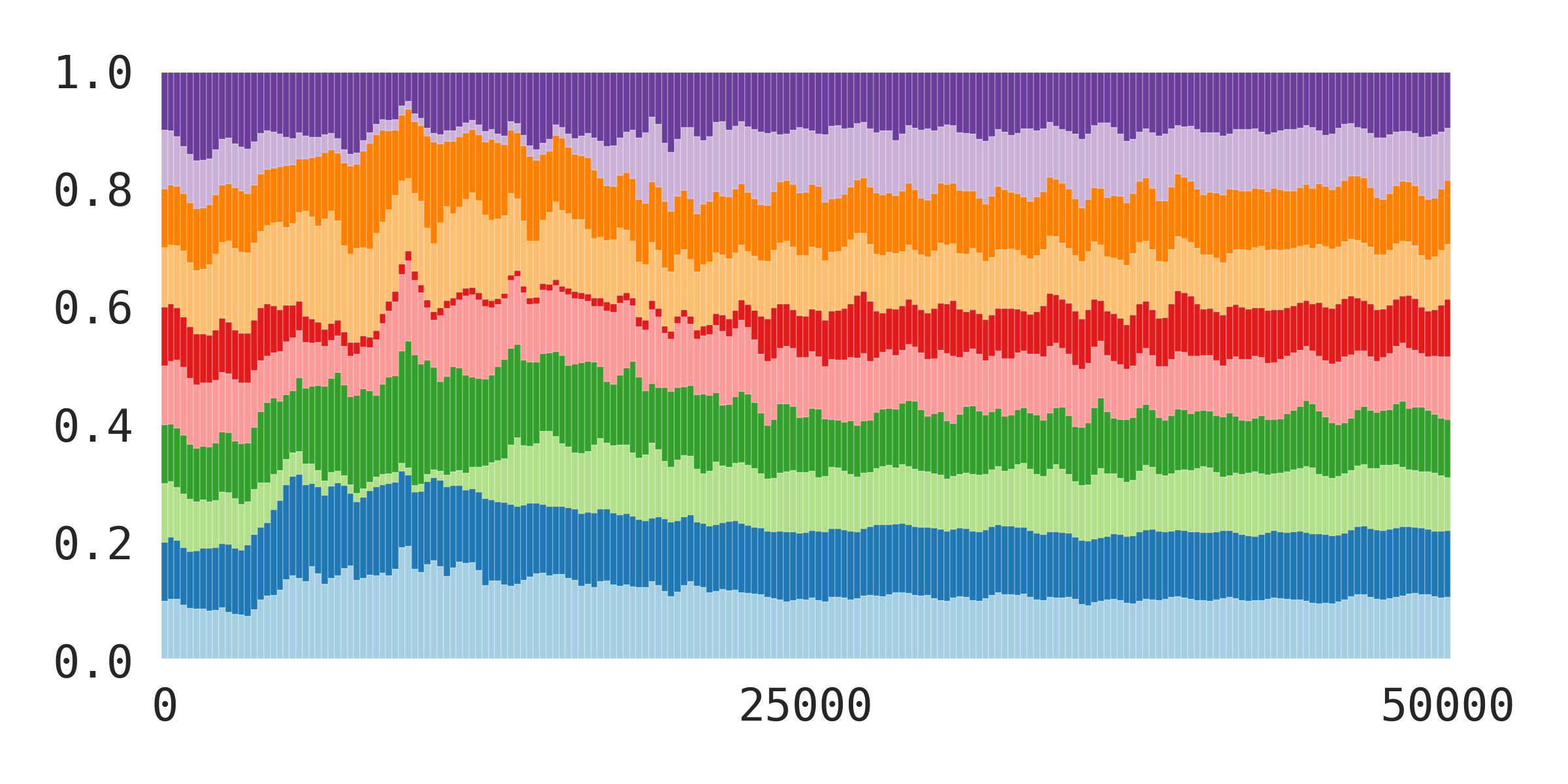}
		\caption{PPO histogram}
	\end{subfigure}
	\caption{Unsorted action selection histograms from the Image Classification Experiment on MNIST dataset.}
	\label{fig:mnist_dist_raw}
\end{figure}

\begin{figure}[h]
	\centering
	\begin{subfigure}{.32\linewidth}
		\includegraphics[width=\linewidth]{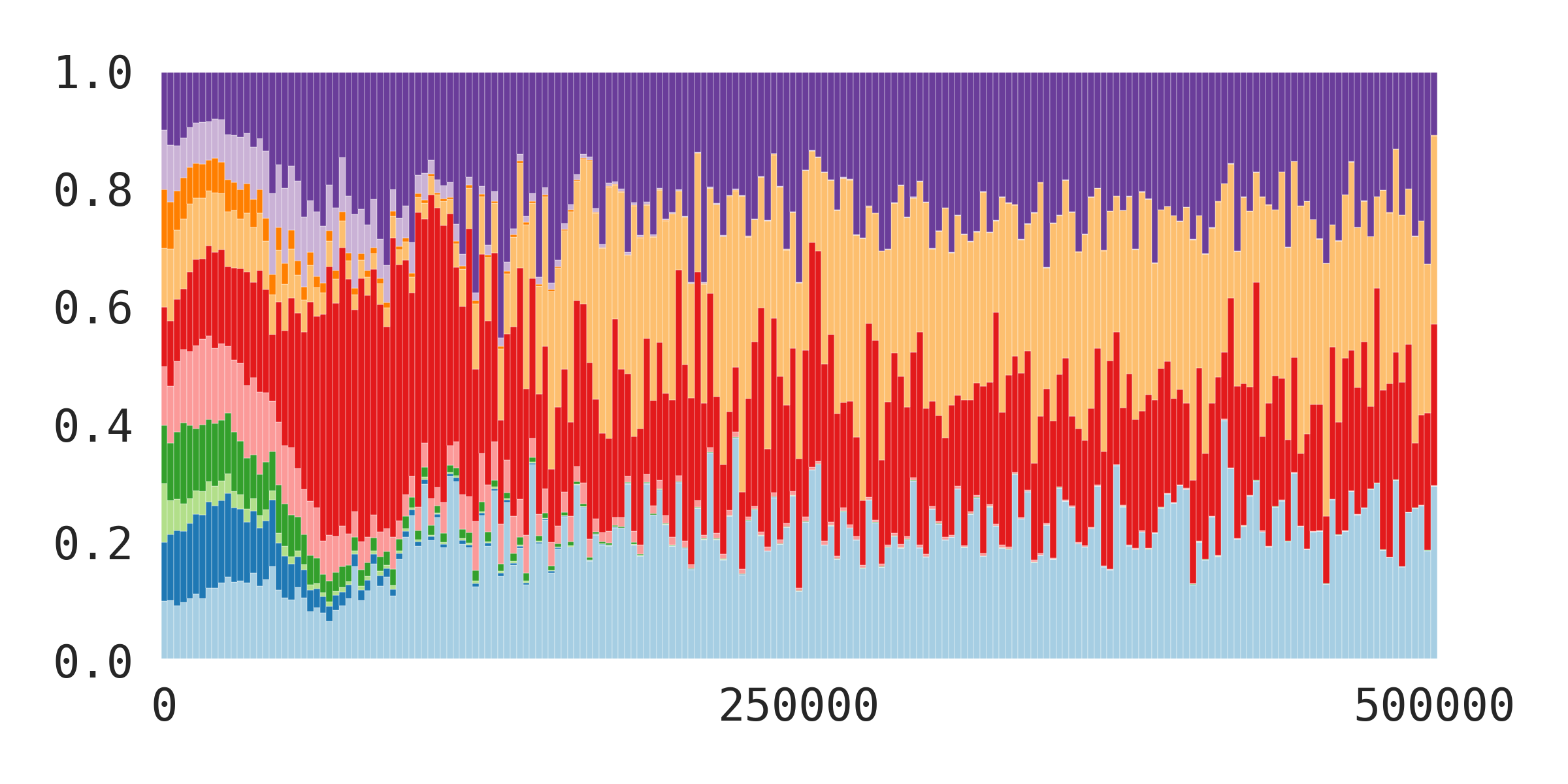}
		\caption{A2C histogram}
	\end{subfigure}
	\begin{subfigure}{.32\linewidth}
		\includegraphics[width=\linewidth]{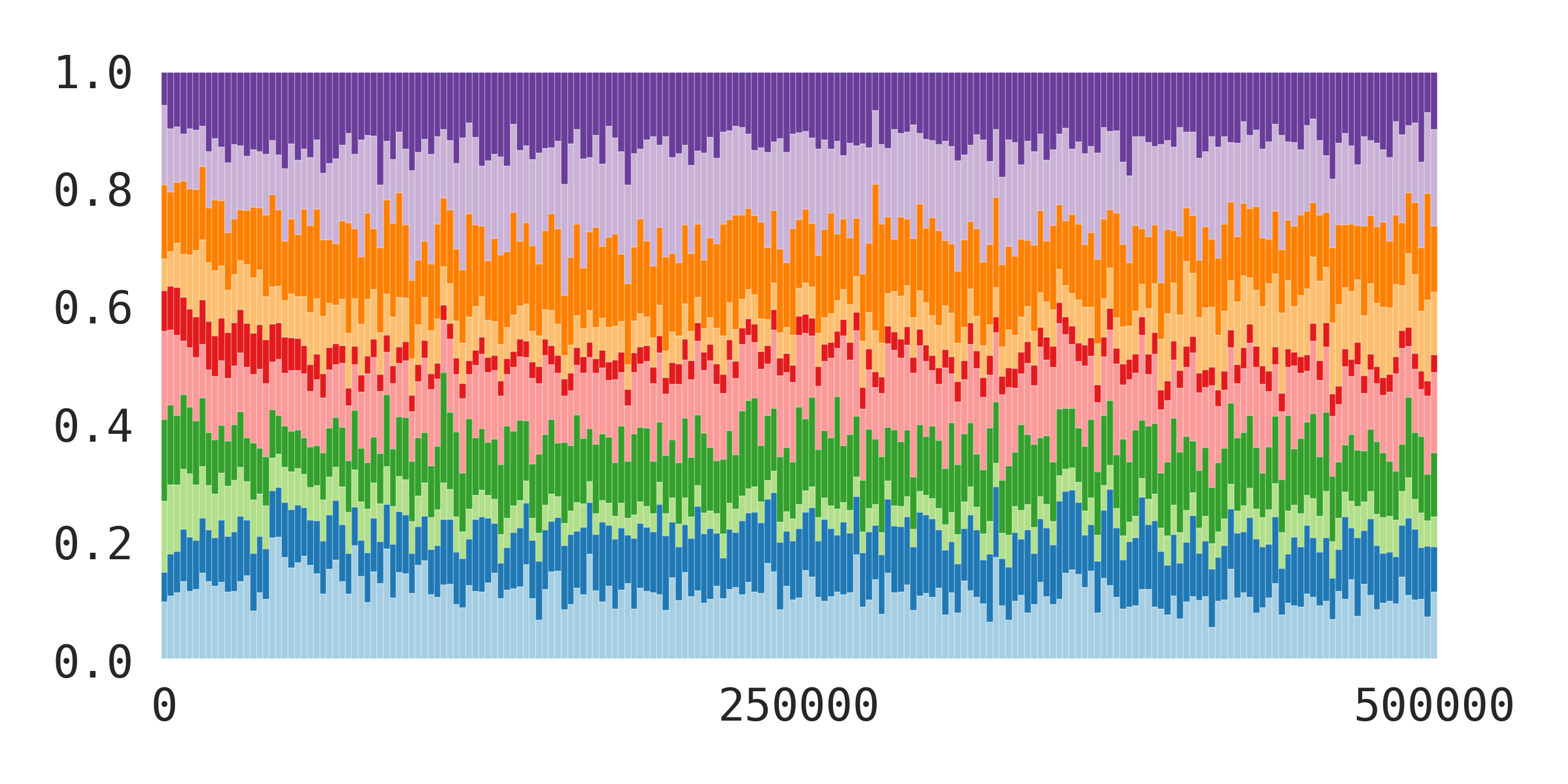}
		\caption{DQN histogram}
	\end{subfigure}
	\begin{subfigure}{.32\linewidth}
		\includegraphics[width=\linewidth]{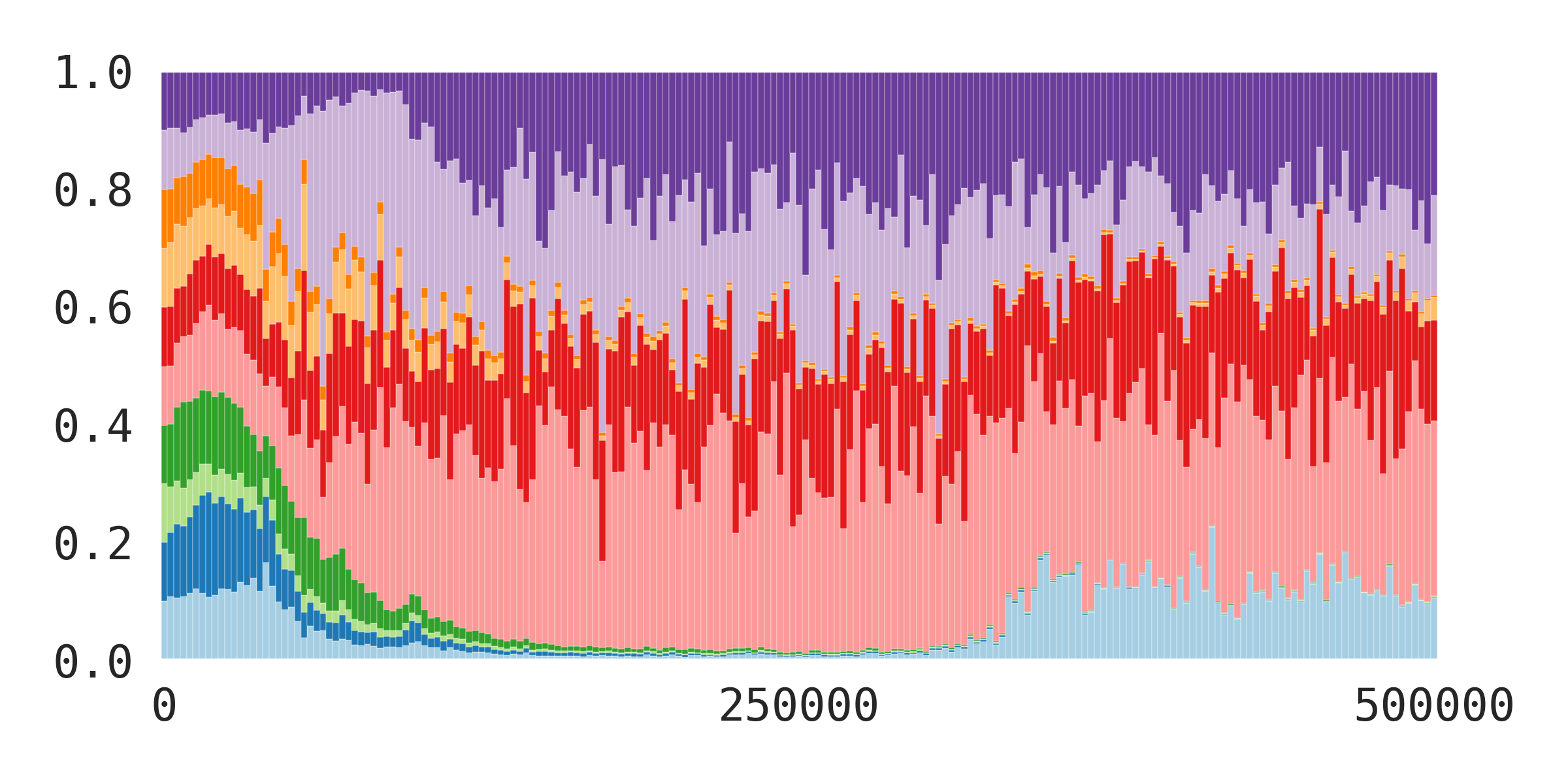}
		\caption{PPO histogram}
	\end{subfigure}
	\caption{Unsorted action selection histograms from the Image Classification Experiment on CIFAR10 dataset.}
	\label{fig:cifar10_dist_raw}
\end{figure}

\begin{table}[h]
	\centering
	\begin{tabular}{cccc}
		\toprule
		 & \multicolumn{3}{c}{Parameters}
		\\\cmidrule{2-4}
		Agent & Architecture & Batch size & Learning rate
		\\\midrule
		A2C & [256,256] & 32 & 1e-3
		\\
		DQN & [256,256] & 32 & 1e-3
		\\
		PPO & [256,256] & 32 & 1e-3
		\\\bottomrule
	\end{tabular}
	\caption{Hyperparameter selection for the Music Recommendation Experiment.}
	\label{tab:spotify_params}
\end{table}

\begin{table}
	\centering\small
	\begin{filecontents*}{./images/spotify_genres.csv}
Index,Playlist name,Music genre,Number of tracks,Spotify ID
0,Classic Acoustic,acoustic,100,37i9dQZF1DX504r1DvyvxG
1,The New Alt,alternative,100,37i9dQZF1DX82GYcclJ3Ug
2,Ambient Essentials,ambient,100,37i9dQZF1DWUrPBdYfoJvz
3,Blues Classics,blues,70,37i9dQZF1DXd9rSDyQguIk
4,Chill Pop,chill,100,37i9dQZF1DX0MLFaUdXnjA
5,Classical Essentials,classical,100,37i9dQZF1DWWEJlAGA9gs0
6,90s Country,country,100,37i9dQZF1DWVpjAJGB70vU
7,mint,electro,75,37i9dQZF1DX4dyzvuaRJ0n
8,Essential Folk,folk,100,37i9dQZF1DWVmps5U8gHNv
9,Top Christian \& Gospel,gospel,100,37i9dQZF1DXcb6CQIjdqKy
10,Most Necessary,hiphop,100,37i9dQZF1DX2RxBh64BHjQ
11,Peaceful Piano,instrumental,100,37i9dQZF1DX4sWSpwq3LiO
12,State of Jazz,jazz,100,37i9dQZF1DX7YCknf2jT6s
13,K-Pop Generations,kpop,100,37i9dQZF1DX14fiWYoe7Oh
14,Heavy Metal,metal,100,37i9dQZF1DX9qNs32fujYe
15,New Music Friday,pop,100,37i9dQZF1DX4JAvHpjipBk
16,Pure Pop Punk,punk,100,37i9dQZF1DXasneILDRM7B
17,Rock Classics,rock,100,37i9dQZF1DWXRqgorJj26U
18,Fresh Finds R\&B,rnb,100,37i9dQZF1DWUFAJPVM3HTX
19,Feelin' Good,soul,100,37i9dQZF1DX9XIFQuFvzM4
\end{filecontents*}
\csvautobooktabular{./images/spotify_genres.csv}
	\caption{Playlists of various musical genres on Spotify, as of 08/16/2023.}
	\label{tab:spotify_genres}
\end{table}

\begin{table}
	\centering\fontsize{8}{9.6}\selectfont
	\begin{filecontents*}{./images/spotify_features.csv}
Genre,Acoustic,Dance,Energy,Instruments,Liveness,Loudness,Mode,Speech,Tempo,Valence
acoustic,0.561,0.555,0.400,0.049,0.165,0.790,0.840,0.037,0.600,0.526
alternative,0.087,0.549,0.772,0.046,0.196,0.907,0.660,0.062,0.648,0.549
ambient,0.903,0.181,0.114,0.873,0.112,0.576,0.630,0.045,0.453,0.078
blues,0.445,0.567,0.518,0.094,0.191,0.817,0.714,0.057,0.620,0.700
chill,0.536,0.546,0.460,0.008,0.159,0.857,0.800,0.055,0.574,0.385
classical,0.920,0.245,0.094,0.731,0.118,0.606,0.620,0.044,0.510,0.125
country,0.261,0.633,0.660,0.001,0.180,0.863,0.960,0.036,0.647,0.670
electro,0.090,0.650,0.811,0.256,0.213,0.898,0.520,0.080,0.652,0.443
folk,0.635,0.520,0.343,0.040,0.188,0.774,0.860,0.041,0.578,0.515
gospel,0.272,0.512,0.609,0.000,0.169,0.890,0.880,0.053,0.614,0.362
hiphop,0.148,0.744,0.621,0.000,0.186,0.881,0.600,0.267,0.630,0.423
instrumental,0.989,0.419,0.037,0.932,0.106,0.536,0.850,0.043,0.513,0.255
jazz,0.421,0.542,0.542,0.586,0.161,0.821,0.440,0.069,0.587,0.469
kpop,0.102,0.736,0.834,0.002,0.207,0.927,0.540,0.078,0.590,0.709
metal,0.011,0.456,0.892,0.068,0.217,0.905,0.560,0.079,0.616,0.390
pop,0.234,0.645,0.628,0.053,0.160,0.885,0.650,0.109,0.613,0.481
punk,0.014,0.490,0.892,0.006,0.198,0.929,0.720,0.079,0.720,0.578
rnb,0.376,0.609,0.532,0.043,0.170,0.857,0.330,0.121,0.612,0.500
rock,0.136,0.536,0.728,0.050,0.175,0.869,0.750,0.062,0.622,0.584
soul,0.393,0.646,0.558,0.008,0.200,0.841,0.740,0.057,0.580,0.726
\end{filecontents*}
\csvautobooktabular{./images/spotify_features.csv}
	\caption{Audio features of musical genres for the Music Recommendation Experiment, as of 08/16/2023.}
	\label{tab:spotify_features}
\end{table}

\begin{table}
	\centering\fontsize{8}{9.6}\selectfont
	\begin{filecontents*}{./images/spotify_top50.csv}
Index,Track,Artist,Spotify ID
0,Seven (feat. Latto),Jung Kook,7x9aauaA9cu6tyfpHnqDLo
1,Cruel Summer,Taylor Swift,1BxfuPKGuaTgP7aM0Bbdwr
2,What Was I Made For?,Billie Eilish,6wf7Yu7cxBSPrRlWeSeK0Q
3,LALA,Myke Towers,7ABLbnD53cQK00mhcaOUVG
4,Dance The Night,Dua Lipa,1vYXt7VSjH9JIM5oRRo7vA
5,vampire,Olivia Rodrigo,3k79jB4aGmMDUQzEwa46Rz
6,Barbie World (with Aqua),Nicki Minaj,741UUVE2kuITl0c6zuqqbO
7,Columbia,Quevedo,6XbtvPmIpyCbjuT0e8cQtp
8,bad idea right?,Olivia Rodrigo,2i8f4VnnBjy0yDqH2C452a
9,fukumean,Gunna,4rXLjWdF2ZZpXCVTfWcshS
10,Sprinter,Dave,2FDTHlrBguDzQkp7PVj16Q
11,Paint The Town Red,Doja Cat,2IGMVunIBsBLtEQyoI1Mu7
12,Super Shy,NewJeans,5sdQOyqq2IDhvmx2lHOpwd
13,Love Me Again,V,2N0SPREDYqILVEFSsWF5N5
14,MELTDOWN (feat. Drake),Travis Scott,67nepsnrcZkowTxMWigSbb
15,Daylight,David Kushner,1odExI7RdWc4BT515LTAwj
16,As It Was,Harry Styles,4Dvkj6JhhA12EX05fT7y2e
17,WHERE SHE GOES,Bad Bunny,7ro0hRteUMfnOioTFI5TG1
18,Flowers,Miley Cyrus,4DHcnVTT87F0zZhRPYmZ3B
19,Kill Bill,SZA,1Qrg8KqiBpW07V7PNxwwwL
20,La Bebe - Remix,Yng Lvcas,2UW7JaomAMuX9pZrjVpHAU
21,FE!N (feat. Playboi Carti),Travis Scott,42VsgItocQwOQC3XWZ8JNA
22,Classy 101,Feid,6XSqqQIy7Lm7SnwxS4NrGx
23,MI EX TENIA RAZON,KAROL G,54zcJnb3tp9c5OVKREZ1Is
24,Rainy Days,V,5ydjxBSUIDn26MFzU3asP4
25,(It Goes Like) Nanana,Peggy Gou,23RoR84KodL5HWvUTneQ1w
26,Style,Taylor Swift,4lIxdJw6W3Fg4vUIYCB0S5
27,LADY GAGA,Peso Pluma,7mXuWTczZNxG5EDcjFEuJR
28,Like Crazy,Jimin,3Ua0m0YmEjrMi9XErKcNiR
29,Ella Baila Sola,Eslabon Armado,3qQbCzHBycnDpGskqOWY0E
30,august,Taylor Swift,3hUxzQpSfdDqwM3ZTFQY0K
31,I Wanna Be Yours,Arctic Monkeys,5XeFesFbtLpXzIVDNQP22n
32,I KNOW ?,Travis Scott,6wsqVwoiVH2kde4k4KKAFU
33,Blank Space,Taylor Swift,1p80LdxRV74UKvL8gnD7ky
34,K-POP,Travis Scott,5L3ecxQnQ9qTBmnLQiwf0C
35,Popular (with Playboi Carti \& Madonna),The Weeknd,6WzRpISELf3YglGAh7TXcG
36,Anti-Hero,Taylor Swift,0V3wPSX9ygBnCm8psDIegu
37,Cupid - Twin Ver.,FIFTY FIFTY,7FbrGaHYVDmfr7KoLIZnQ7
38,Calm Down (with Selena Gomez),Rema,1s7oOCT8vauUh01PbJD6ps
39,TELEKINESIS (feat. SZA \& Future),Travis Scott,1i9lZvlaDdWDPyXEE95aiq
40,Primera Cita,Carin Leon,25OeKzcqakFPaJlXHPE5lm
41,un x100to,Grupo Frontera,4QctbZRsuWYPT0pbKg7u75
42,Last Night,Morgan Wallen,7K3BhSpAxZBznislvUMVtn
43,TULUM,Peso Pluma,7bPp2NmpmyhLJ7zWazAXMu
44,Creepin' (with The Weeknd \& 21 Savage),Metro Boomin,2dHHgzDwk4BJdRwy9uXhTO
45,TQG,KAROL G,0DWdj2oZMBFSzRsi2Cvfzf
46,Die For You,The Weeknd,2LBqCSwhJGcFQeTHMVGwy3
47,BESO,ROSALIA,609E1JCInJncactoMmkDon
48,Rich Men North of Richmond,Oliver Anthony Music,78Du4CMFgnhdlG33gblkiP
49,I'm Good (Blue),David Guetta,4uUG5RXrOk84mYEfFvj3cK
\end{filecontents*}
\csvautobooktabular{./images/spotify_top50.csv}
	\caption{Tracks from `Top 50 - Global' playlist on Spotify, as of 08/16/2023.}
	\label{tab:spotify_actions}
\end{table}

\begin{figure}[h]
	\centering
	\begin{subfigure}{.32\linewidth}
		\includegraphics[width=\linewidth]{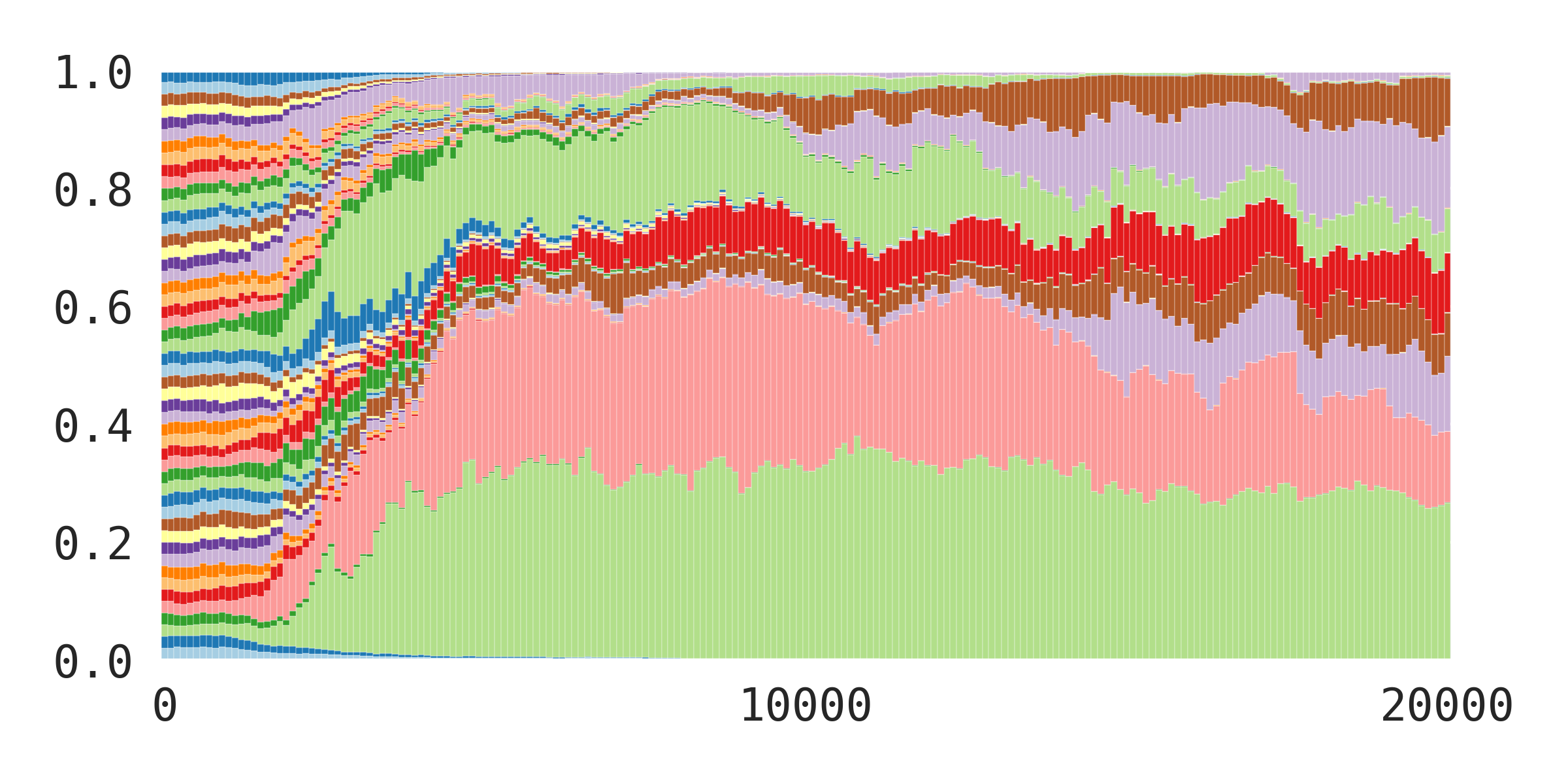}
		\caption{A2C histogram}
	\end{subfigure}
	\begin{subfigure}{.32\linewidth}
		\includegraphics[width=\linewidth]{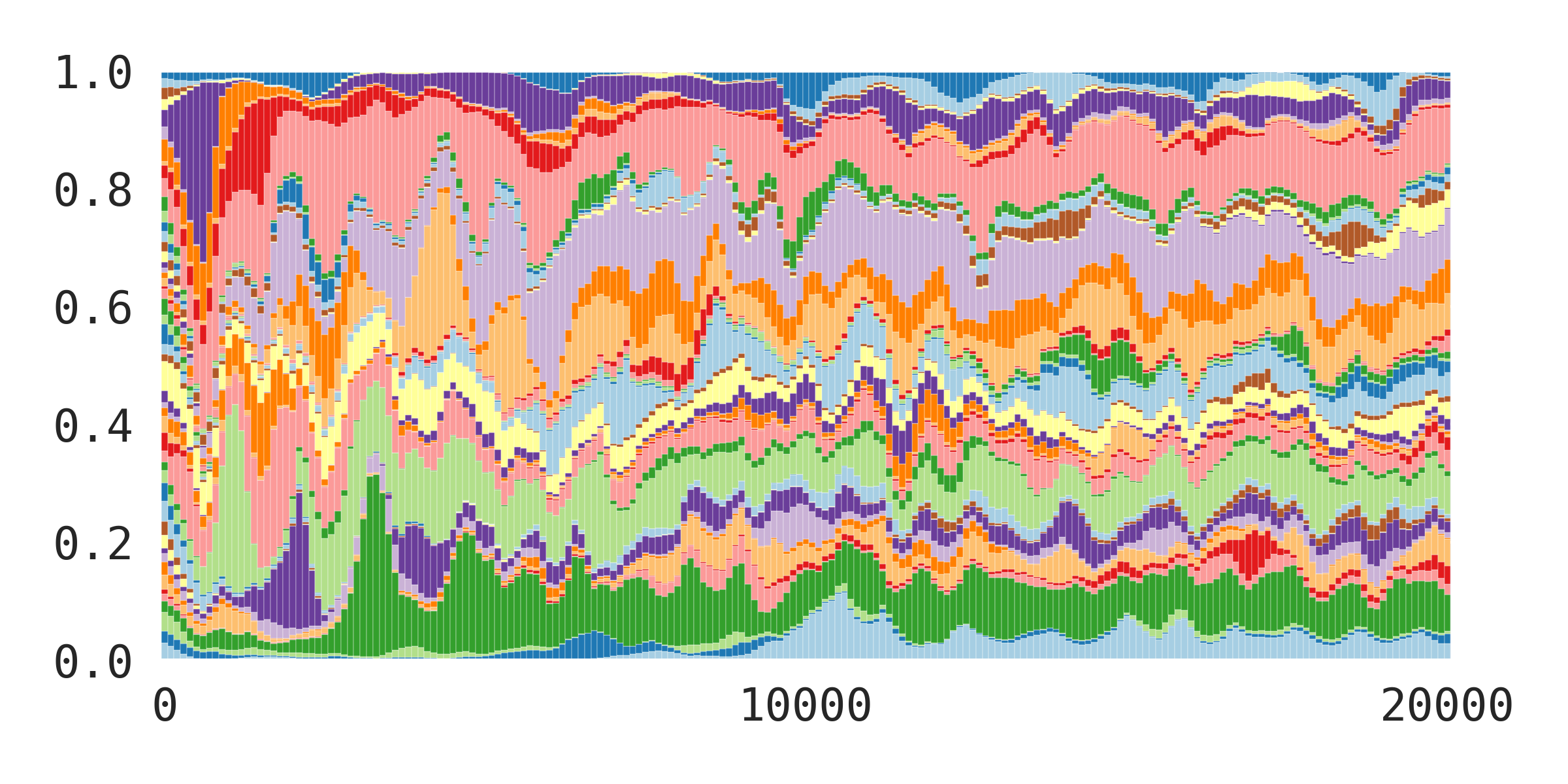}
		\caption{DQN histogram}
	\end{subfigure}
	\begin{subfigure}{.32\linewidth}
		\includegraphics[width=\linewidth]{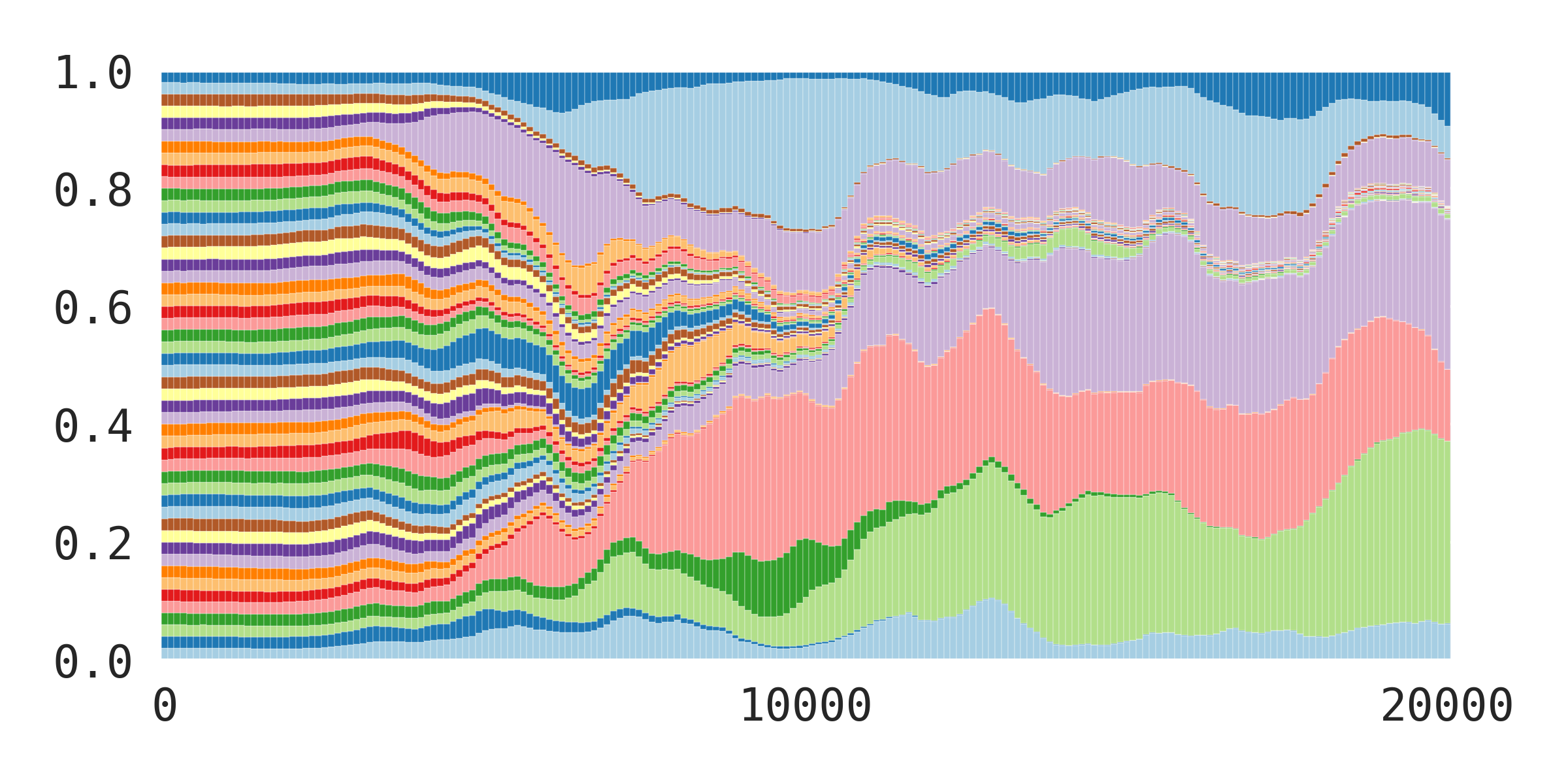}
		\caption{PPO histogram}
	\end{subfigure}
	\caption{Unsorted action selection histograms from the Music Recommendation Experiment.}
	\label{fig:spotify_dist_raw}
\end{figure}

\begin{table}[h]
	\centering
	\begin{tabular}{cccc}
		\toprule
		 & \multicolumn{3}{c}{Parameters}
		\\\cmidrule{2-4}
		Agent & Architecture & Batch size & Learning rate
		\\\midrule
		A2C & [512,512] & 32 & 1e-4
		\\
		DQN & [512,512] & 32 & 1e-4
		\\
		PPO & [512,512] & 32 & 1e-4
		\\\bottomrule
	\end{tabular}
	\caption{Hyperparameter selection for the Online Advertisement Experiment.}
	\label{tab:recogym_params}
\end{table}

\begin{figure}[h]
	\centering
	\begin{subfigure}{.32\linewidth}
		\includegraphics[width=\linewidth]{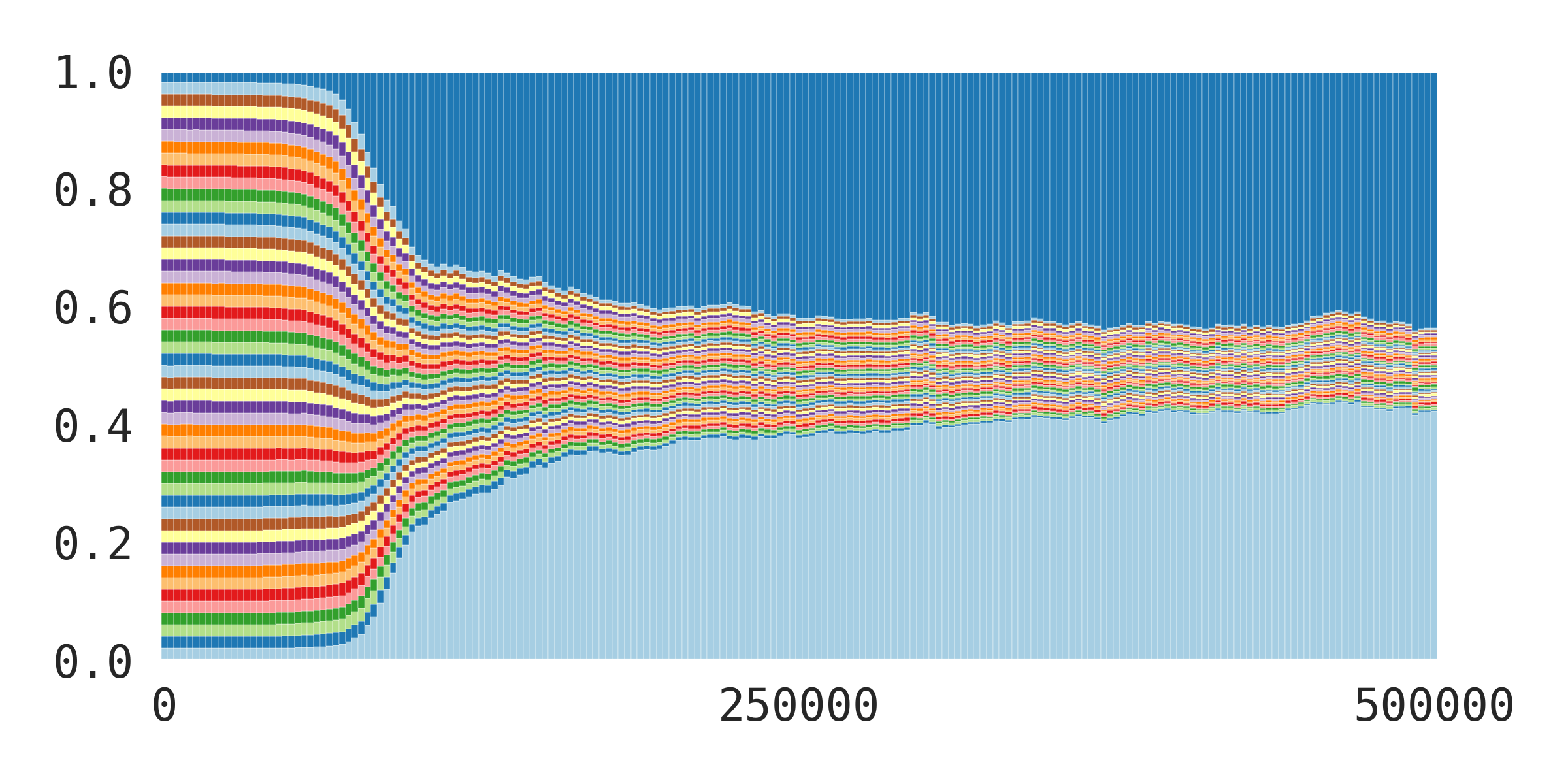}
		\caption{A2C histogram}
	\end{subfigure}
	\begin{subfigure}{.32\linewidth}
		\includegraphics[width=\linewidth]{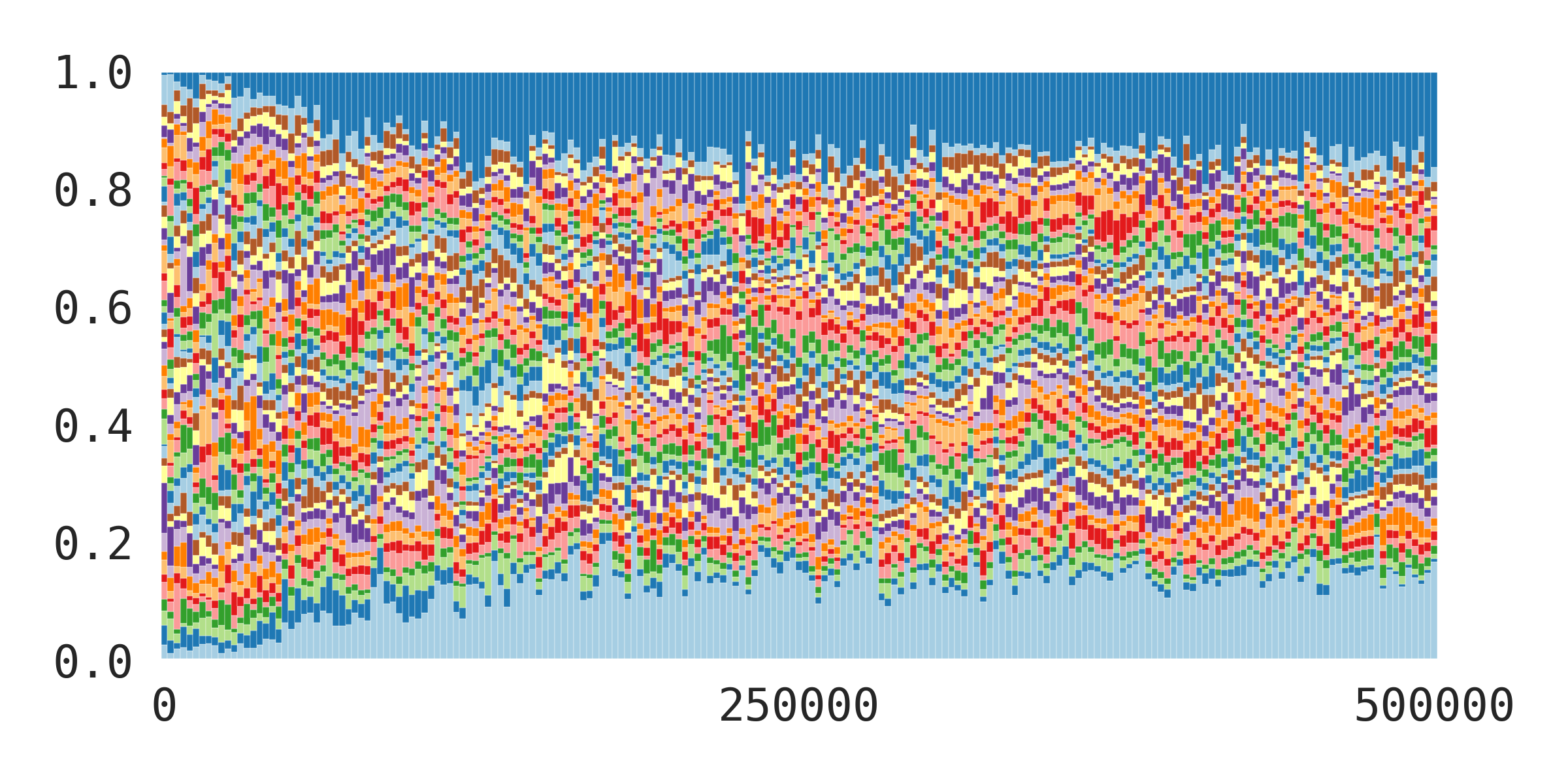}
		\caption{DQN histogram}
	\end{subfigure}
	\begin{subfigure}{.32\linewidth}
		\includegraphics[width=\linewidth]{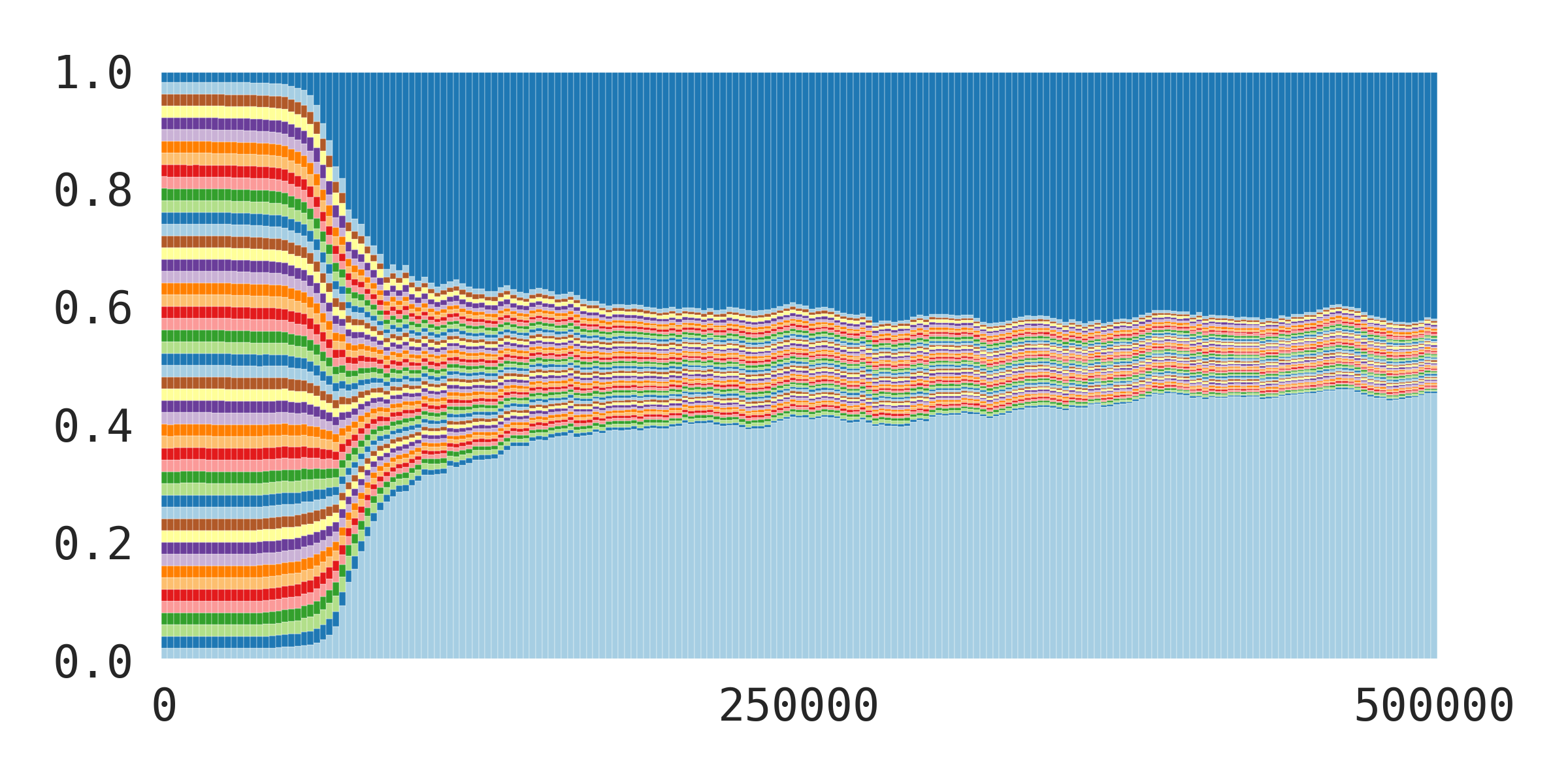}
		\caption{PPO histogram}
	\end{subfigure}
	\caption{Unsorted action selection histograms from the Online Advertisement Experiment.}
	\label{fig:recogym_dist_raw}
\end{figure}

\begin{table}[h]
	\centering
	\begin{tabular}{cccc}
		\toprule
		 & \multicolumn{3}{c}{Parameters}
		\\\cmidrule{2-4}
		Agent & Architecture & Batch size & Learning rate
		\\\midrule
		A2C & [512,512] & 32 & 1e-3
		\\
		DQN & [512,512] & 32 & 1e-3
		\\
		PPO & [512,512] & 32 & 1e-3
		\\\bottomrule
	\end{tabular}
	\caption{Hyperparameter selection for the Behavioral Preference Experiment.}
	\label{tab:personalization_params}
\end{table}

\begin{figure}[h]
	\centering
	\begin{subfigure}{.32\linewidth}
		\includegraphics[width=\linewidth]{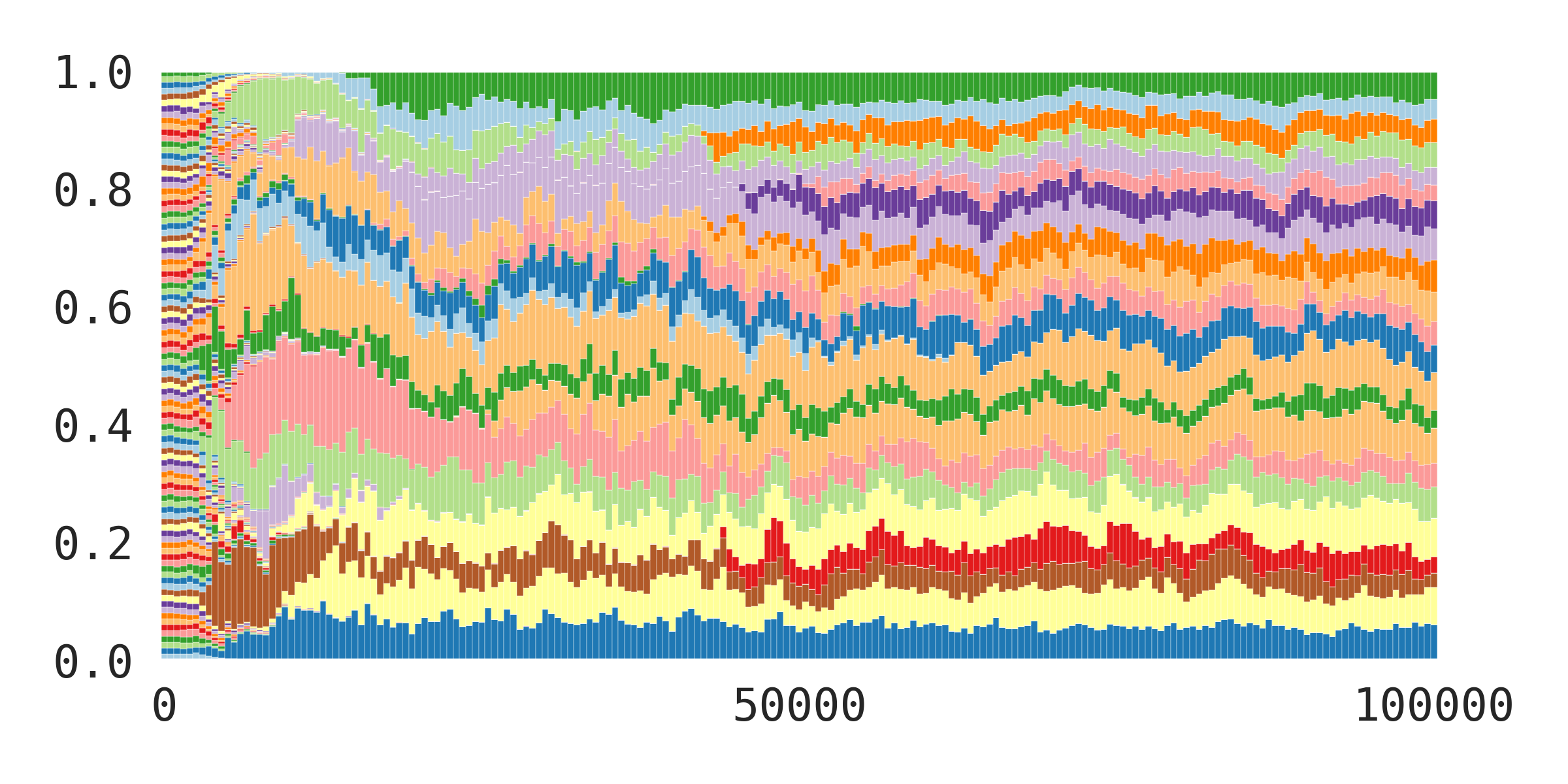}
		\caption{A2C histogram}
	\end{subfigure}
	\begin{subfigure}{.32\linewidth}
		\includegraphics[width=\linewidth]{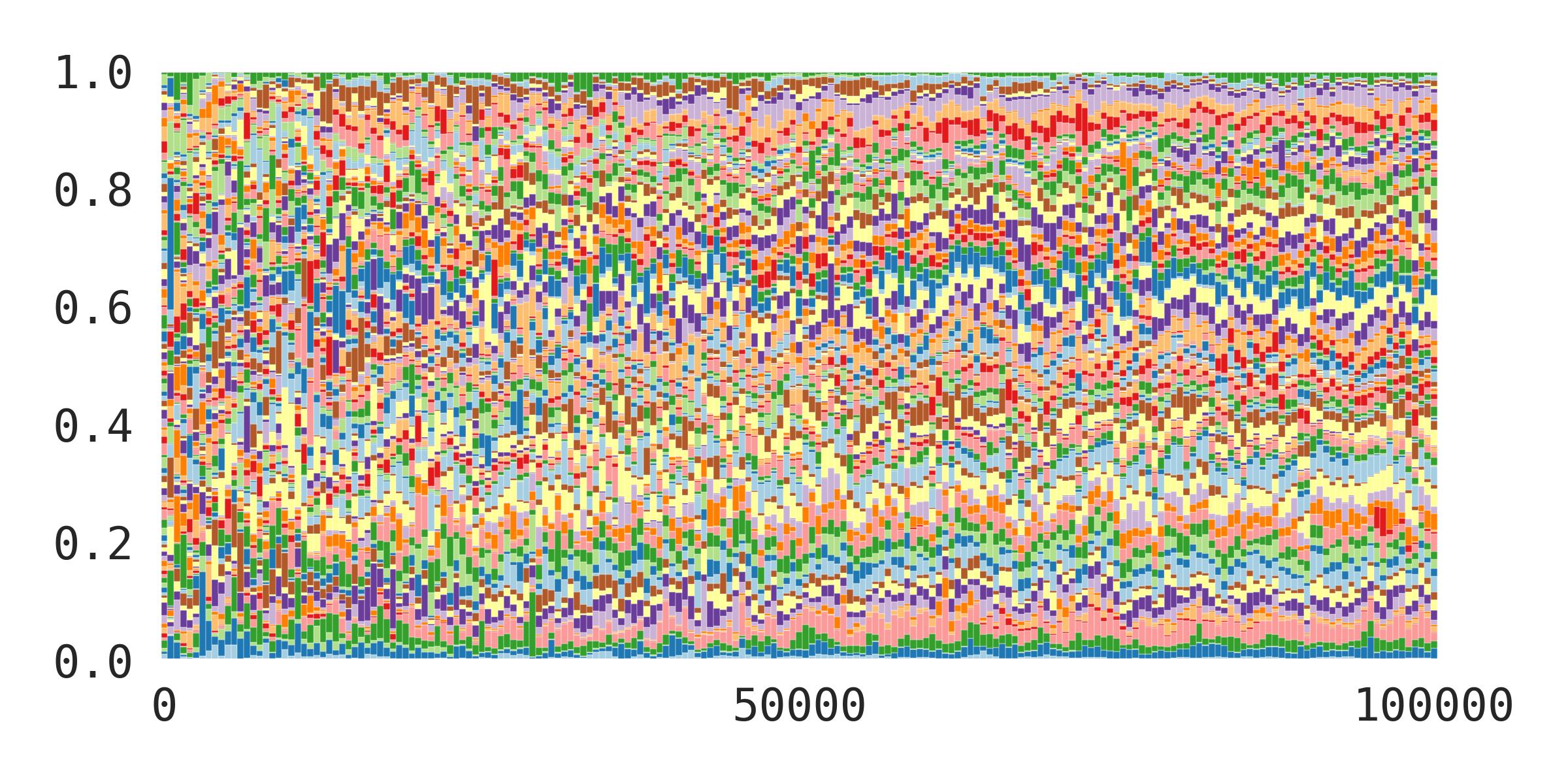}
		\caption{DQN histogram}
	\end{subfigure}
	\begin{subfigure}{.32\linewidth}
		\includegraphics[width=\linewidth]{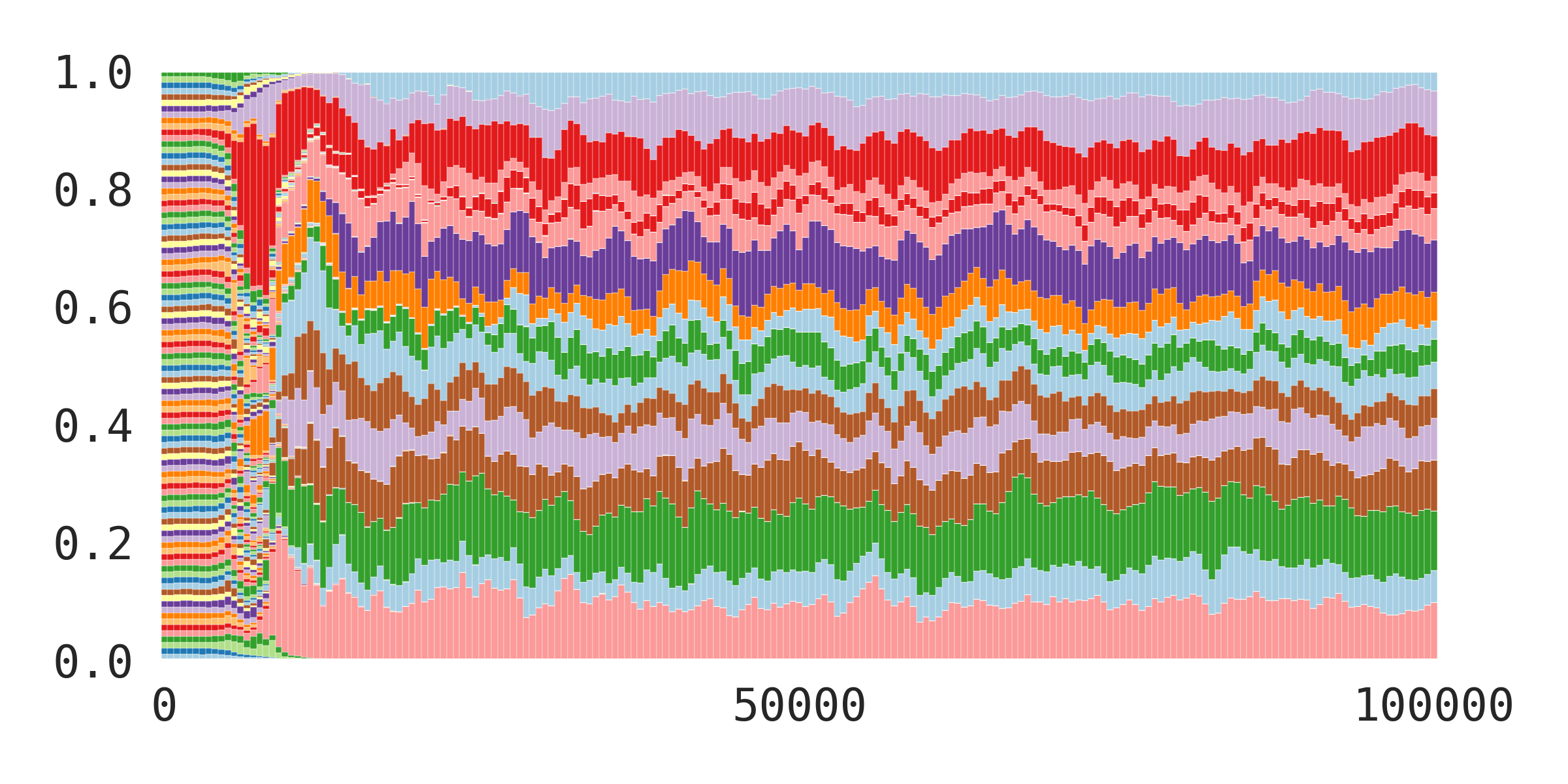}
		\caption{PPO histogram}
	\end{subfigure}
	\caption{Unsorted action selection histograms from the Behavioral Preference Experiment.}
	\label{fig:personalization_dist_raw}
\end{figure}

\section{Theoretical Analysis}\label{sec:theory_appendix}
In this section we provide the proofs for the theoretical results stated in Section~\ref{sec:theory}.
We begin with an auxiliary lemma.

\begin{lemma}\label{thm:grad_pi}
Let the policy $\pi = \pi(\Theta)$ be obtained as softmax-normalization of the outputs $\mathcal{Z} = \{z_1, \ldots, z_K\}$, i.e.
\[
	\pi(\,\cdot\,;\Theta) = \softmax(\mathcal{Z}(\,\cdot\,;\Theta)).
\]
Then
\[
	\nabla \pi(a|s) = \pi(a|s) \big[\mathbbm{1}(a=k) - \pi(k|s)\big]_{k=1}^K \times \nabla\mathcal{Z}(s).
\]
\end{lemma}
\begin{proof}[Proof of Lemma~\ref{thm:grad_pi}]
For any state $s \in \mathcal{S}$ and action $a \in \mathcal{A} = \{1, 2, \ldots, K\}$ we have
\[
	\pi(a|s) = \softmax(\mathcal{Z}(s))\big|_a
	= \frac{\exp(z_a(s))}{\sum_{k=1}^K \exp(z_k(s))}.
\]
Then for any $k \in \{1, 2, \ldots, K\}$ we derive
\[
	\frac{\partial \pi(a|s)}{\partial z_k}
	= \frac{\partial}{\partial z_k} \frac{\exp(z_a(s))}{\sum_{k=1}^K \exp(z_k(s))}
	= \left\{\begin{array}{ll}
		\pi(a|s) (1 - \pi(k|s)) & \text{ if } k = a
		\\
		-\pi(a|s)\pi(k|s) & \text{ if } k \neq a
		\end{array}\right.
\]
and for any weight $\theta$ we write:
\begin{align*}
	\frac{\partial \pi(a|s)}{\partial \theta}
	= \sum_{k=1}^K \frac{\partial \pi(a|s)}{\partial z_k} \frac{\partial z_k(s)}{\partial \theta}
	&= \pi(a|s) \sum_{k=1}^K \big( \mathbbm{1}(a=k) - \pi(k|s) \big) \frac{\partial z_k(s)}{\partial \theta}
	\\
	&= \pi(a|s) \ub{\big[\mathbbm{1}(a=k) - \pi(k|s)\big]_{k=1}^K}{1 \times K}
	\times \ub{\frac{\partial\mathcal{Z}(s)}{\partial \theta}}{K \times 1}.
\end{align*}
Combining over all trainable weights $\Theta$, we obtain
\[
	\nabla \pi(a|s)
	= \pi(a|s) \ub{\big[\mathbbm{1}(a=k) - \pi(k|s)\big]_{k=1}^K}{1 \times K}
	\times \ub{\nabla\mathcal{Z}(s)}{K \times M}.
\]
\end{proof}

Next we prove our theoretical result regarding the linear network agents.

\begin{proof}[Proof of Theorem~\ref{thm:outputs_linear}]
Let $\mathcal{S} \subset \mathbb{R}^d$ and $\mathcal{A} \subset \mathbb{R}^K$, then the linear network $\mathcal{Z} : \mathcal{S} \to \mathbb{R}^K$ is given as
\[
	\mathcal{Z}(s) = \ub{W}{K \times d} \times \ub{s}{d} \in \mathbb{R}^K,
\]
where $W \in \mathbb{R}^{K \times d}$ is the matrix of trainable weights, that is updated via the gradient descent optimization as
\[
	W^\prime = W - \frac{\lambda}{N} \sum_{n=1}^N \nabla\mathcal{L}(s_n,a_n,r_n)
\]
on the batch of transitions $\{(s_n,a_n,r_n)\}_{n=1}^N$.
Then the network outputs $\mathcal{Z}$ change as follows
\[
	\mathcal{Z}^\prime(x)
	= W^\prime \times x
	= \mathcal{Z}(x) - \frac{\lambda}{N} \sum_{n=1}^N \nabla\mathcal{L}(s_n,a_n,r_n) \times x.
\]
The exact formulation of the loss gradient $\nabla\mathcal{L}$ differs depending on the type of agent we consider: policy gradient or q-learning.

For policy gradient agent the loss gradient $\nabla\mathcal{L}_{pg}$ is given as (see e.g.~\cite[Chapter~13]{sutton2018reinforcement})
\[
	\nabla\mathcal{L}_{pg}(s,a,r)
	= - r \nabla\log\pi(a|s)
	= -r \ub{\big[ \mathbbm{1}(a=k) - \pi_{pg}(k|s) \big]_{k=1}^K}{1 \times K}
	\times \ub{\nabla\mathcal{Z}(s)}{K \times K \times d},
\]
where the last equality is justified by Lemma~\ref{thm:grad_pi}.
For any $1 \le k \le K$ we write
\[
	\nabla z_k(s) = \nabla \langle W_k, s \rangle
	= \left[\begin{array}{c}
		0 \\ \vdots \\ s \\ \vdots \\ 0
	\end{array}\right] \in \mathbb{R}^{K \times d},
\]
where $W_k$ is the $k$-th row of the weight matrix $W \in \mathbb{R}^{K \times d}$ and the last matrix consists of the vector $s \in \mathcal{S}$ on the $k$-th row and $0$ everywhere else.
Then we rewrite the loss gradient as
\[
	\nabla\mathcal{L}_{pg}(s,a,r)
	= -r \Big[ \big(\mathbbm{1}(a=k) - \pi_{pg}(k|s) \big) \, s \Big]_{k=1}^K
	\in \mathbb{R}^{K \times d}
\]
and the outputs update as
\[
	\mathcal{Z}^\prime(x)
	= \mathcal{Z}(x)
	+ \frac{\lambda}{N} \sum_{n=1}^N r_n \Big[ \big(\mathbbm{1}(a_n=k) - \pi_{pg}(k|s_n) \big) \, \langle s_n, x \rangle \Big]_{k=1}^K.
\]

For q-learning agent the loss gradient $\nabla\mathcal{L}_{ql}$ is given as
\[
	\nabla\mathcal{L}_{ql}(s,a,r)
	= \nabla (z_a(s) - r)^2
	= 2 (z_a(s) - r) \nabla z_a(s)
	= 2 (z_a(s) - r)
	\left[\begin{array}{c}
		0 \\ \vdots \\ s \\ \vdots \\ 0
	\end{array}\right] \in \mathbb{R}^{K \times d}
\]
and thus the outputs update rule is
\[
	\mathcal{Z}^\prime(x)
	= \mathcal{Z}(x)
	- \frac{2\lambda}{N} \sum_{n=1}^N \big( z_{a_n}(s_n) - r_n \big)
	\left[\begin{array}{c}
		0 \\ \vdots \\ \langle s_n, x \rangle \\ \vdots \\ 0
	\end{array}\right]
\]
where the only non-zero coordinate of the column vector is at the row $a_n \in \mathcal{A} = \{1, 2, \ldots, K\}$.
\end{proof}

Finally, we prove the general result for the network update in the contextual bandit environment.
We note that the loss functions used throughout the proofs are taken from the Stable-Baselines3 implementation so that our theoretical analysis matches our numerical experiments.
To the best of our knowledge, such formulations of the loss functions are the most commonly used in practice.

\begin{proof}[Proof of Theorem~\ref{thm:outputs}]
The update rule for the vector of trainable weights $\Theta \in \mathbb{R}^M$ is given by
\[
	\Theta^\prime = \Theta - \frac{\lambda}{N} \sum_{n=1}^N \nabla\mathcal{L}(s_n,a_n,r_n;\Theta),
\]
where the loss gradient $\nabla\mathcal{L}$ depends on the type of learning algorithm of the agent.
In order to quantify the change of the agent's outputs $\mathcal{Z} = \{z_1, \ldots, z_K\}$ with the update of the trainable weights $\Theta \leftarrow \Theta^\prime$ we use Taylor expansion to obtain
\begin{align}
	\nonumber
	\mathcal{Z}(\,\cdot\,;\Theta^\prime)
	&= \ub{\mathcal{Z}(\,\cdot\,;\Theta)}{K \times 1}
	+ \ub{\nabla\mathcal{Z}(\,\cdot\,;\Theta)}{K \times M}
	\times \ub{(\Theta^\prime - \Theta)^T}{M \times 1}
	+ \mathcal{O}(\|\Theta^\prime - \Theta\|^2)
	\\
	\label{eq:z_update}
	&= \mathcal{Z}(\,\cdot\,;\Theta)
	- \nabla\mathcal{Z}(\,\cdot\,;\Theta)
	\times \frac{\lambda}{N} \sum_{n=1}^N \nabla\mathcal{L}(s_n,a_n,r_n;\Theta)
	+ \mathcal{O}(\|\Theta^\prime - \Theta\|^2).
\end{align}
Next, we derive the loss gradient $\nabla\mathcal{L}$ for each of the agents we are using.

For A2C the loss gradient $\nabla\mathcal{L}_{a2c}$, see e.g.~\cite{zahavy2020self}, consists of the policy gradient $\nabla\mathcal{L}_\pi$ and the value gradient $\nabla\mathcal{L}_v$ (often referred to as the actor and critic loss, respectively):
\[
	\nabla\mathcal{L}_{a2c} = \nabla\mathcal{L}_\pi + \frac{1}{2} \nabla\mathcal{L}_v.
\]
The policy gradient $\nabla\mathcal{L}_\pi$, see e.g.~\cite{sutton1999policy}, is given as
\[
	\nabla\mathcal{L}_\pi(s,a,r) = -(r - v(s)) \nabla\log \pi(a|s)
\]
and the value gradient $\nabla\mathcal{L}_v$ as
\[
	\nabla\mathcal{L}_v(s,a,r) = \nabla (r - v(s))^2 = -2 (r - v(s)) \nabla v(s).
\]
By Lemma~\ref{thm:grad_pi} we rewrite the policy gradient $\nabla\mathcal{L}_\pi$ as
\[
	\nabla\mathcal{L}_\pi(s,a,r) = -(r - v(s)) \big[\mathbbm{1}(a=k) - \pi(k|s)\big]_{k=1}^K \nabla\mathcal{Z}(s)
\]
and derive
\begin{equation}\label{eq:a2c_loss}
	\nabla\mathcal{L}_{a2c}(s,a,r) = -\big( r - v(s) \big)
	\cdot \Big( \ub{\nabla v(s)}{1 \times M}
	+ \ub{\big[\mathbbm{1}(a=k) - \pi(k|s)\big]_{k=1}^K}{1 \times K} \times \ub{\nabla\mathcal{Z}(s)}{K \times M} \Big).
\end{equation}

For DQN the gradient loss $\nabla\mathcal{L}_{dqn}$ is given as
\begin{equation}\label{eq:dqn_loss}
	\nabla\mathcal{L}_{dqn}(s,a,r)
	= \nabla (r - z_a(s))^2
	= -2 (r - z_a(s)) \ub{\nabla z_a(s)}{1 \times M}.
\end{equation}

For PPO the loss gradient $\nabla\mathcal{L}_{ppo}$ consists of the clip loss gradient $\nabla\mathcal{L}_{clip}$ and the value gradient $\nabla\mathcal{L}_v$ as
\[
	\nabla\mathcal{L}_{ppo} = \nabla\mathcal{L}_{clip} + \frac{1}{2} \nabla\mathcal{L}_v.
\]
The clip loss gradient $\nabla\mathcal{L}_{clip}$~\cite{schulman2017proximal} that we use in our implementation is given as
\[
	\nabla\mathcal{L}_{clip}(s,a,r)
	= -(r - v(s)) \nabla \clip\left\{ \frac{\pi(a|s)}{\widetilde{\pi}(a|s)}, 1 - \epsilon, 1 + \epsilon \right\}
	= -(r - v(s)) \left\{\begin{array}{ll}
		\frac{\nabla \pi(a|s)}{\widetilde{\pi}(a|s)}
		& \text{ if } \left| 1 - \frac{\pi(a|s)}{\widetilde{\pi}(a|s)} \right| < \epsilon
		\\
		0 & \text{ otherwise}
		\end{array}\right.
\]
where $\widetilde{\pi}$ is the previous snapshot of the policy, and $\epsilon$ is the clip ratio (in our experiments we use the default value $\epsilon = 0.2$).
By Lemma~\ref{thm:grad_pi} we rewrite the clip loss gradient $\nabla\mathcal{L}_{clip}$ as
\[
	\nabla\mathcal{L}_{clip}(s,a,r)
	= -(r - v(s)) \big[\mathbbm{1}(a=k) - \pi(k|s)\big]_{k=1}^K \times \nabla\mathcal{Z}(s)
	\cdot \left\{\begin{array}{ll}
		\frac{\pi(a|s)}{\widetilde{\pi}(a|s)}
		& \text{ if } \left| 1 - \frac{\pi(a|s)}{\widetilde{\pi}(a|s)} \right| < \epsilon
		\\
		0 & \text{ otherwise}
	\end{array}\right.
\]
and derive
\begin{multline}\label{eq:ppo_loss}
	\nabla\mathcal{L}_{ppo}(s,a,r)
	= - \big( r - v(s) \big)
	\cdot \Bigg( \ub{\nabla v(s)}{1 \times M}
	\\
	+ \ub{\big[\mathbbm{1}(a=k) - \pi(k|s)\big]_{k=1}^K}{1 \times K}
	\times \ub{\nabla\mathcal{Z}(s)}{K \times M}
	\cdot \left\{\begin{array}{ll}
		\frac{\pi(a|s)}{\widetilde{\pi}(a|s)}
		& \text{ if } \left| 1 - \frac{\pi(a|s)}{\widetilde{\pi}(a|s)} \right| < \epsilon
		\\
		0 & \text{ otherwise}
	\end{array}\right.\Bigg).
\end{multline}

Finally, we substitute~\eqref{eq:a2c_loss}, \eqref{eq:dqn_loss}, and~\eqref{eq:ppo_loss} into~\eqref{eq:z_update} to complete the proof.
\end{proof}

\end{document}